%% file: wip/PREPRINT_TEMP_WILL_BE_DELETED_JUST_TO_MODIFY_HEADER.tex
\documentclass{article} 
\usepackage{iclr2027_conference,times}

\input{math_commands.tex}

\usepackage{mathtools}

\title{Length-Independent State Tracking \\ Under a Parallel Scan}

\author{Julien Brandoit, Arthur Fyon, Thomas Braipson, Tom Clara, Florent De Geeter, \\
\textbf{Pierre Sacr\'e, Damien Ernst \& Guillaume Drion} \\
Montefiore Institute, University of Li\`ege \\
Li\`ege, Belgium \\
\texttt{\{jbrandoit,gdrion\}@uliege.be}
}

\iclrfinalcopy

\makeatletter
\def\@oddhead{%
  \vbox{%
    \hrule height 0.5pt 
    \vskip 4pt
    \hbox to \textwidth{\small\hfill Preprint. Under review.\hfill \thepage}%
    \vskip 4pt
    \hrule height 0.5pt 
  }%
}
\def\@evenhead{%
  \vbox{%
    \hrule height 0.5pt 
    \vskip 4pt
    \hbox to \textwidth{\small\thepage\hfill Preprint. Under review.\hfill}%
    \vskip 4pt
    \hrule height 0.5pt 
  }%
}
\makeatother

\usepackage[T1]{fontenc} 
\usepackage{amsmath,amssymb,amsthm}
\usepackage{booktabs}
\usepackage{longtable} 
\usepackage{colortbl} 
\usepackage{float}
\usepackage{array}
\newcolumntype{L}[1]{>{\raggedright\arraybackslash}p{#1}}
\usepackage{graphicx}
\usepackage{xcolor}
\usepackage[hypertexnames=false]{hyperref} 
\usepackage{thm-restate} 
\usepackage{url}
\usepackage{mathrsfs}
\usepackage{xspace}
\usepackage{enumitem}
\setlist{topsep=-\parskip+2pt,partopsep=0pt,parsep=0pt,itemsep=0pt,leftmargin=*}

\newlist{airy}{enumerate}{1}
\setlist[airy]{label=(\roman*),topsep=0pt,partopsep=0pt,parsep=0pt,itemsep=5pt,leftmargin=*}

\newlist{conditions}{enumerate}{1}
\setlist[conditions]{leftmargin=*,align=left,labelsep=0.5em,topsep=-\parskip+2pt,partopsep=0pt,parsep=0pt,itemsep=0pt}

\newtheorem{theorem}{Theorem}[section]
\newtheorem{proposition}[theorem]{Proposition}
\newtheorem{lemma}[theorem]{Lemma}
\newtheorem{corollary}[theorem]{Corollary}
\newtheorem{definition}[theorem]{Definition}
\newtheorem{example}[theorem]{Example}
\newtheorem{remark}[theorem]{Remark}

\renewcommand{\Tr}{\mathsf{T}} 

\newcommand{\Sig}{\Sigma}
\newcommand{\op}{\mathrm{op}}

\newcommand{\diam}{\mathrm{diam}}
\newcommand{\dist}{\mathrm{dist}}
\newcommand{\Ball}{\bar B_\eta}
\newcommand{\enc}{\iota}
\newcommand{\dec}{\pi}
\newcommand{\rnd}{\mathrm{rd}}
\newcommand{\Mtar}{M_{\delta}}   
\newcommand{\Mrnn}{M_{\Phi}}       
\usepackage{xparse}

\NewDocumentCommand{\todo}{g}{%
  \textcolor{red}{\textbf{[TODO\IfValueT{#1}{: #1}]}}%
}

\NewDocumentCommand{\cellname}{o}{%
  \IfValueTF{#1}
    {\emph{Neural Finite-State Machine} (NFSM)}
    {NFSM\xspace}%
}
\NewDocumentCommand{\cellSname}{o}{%
  \IfValueTF{#1}
    {\emph{Neural Finite-State Machines} (NFSMs)}
    {NFSMs\xspace}%
}

\newcommand{\condstyle}[1]{#1}
\makeatletter
\newcommand{\condlabel}[1]{\def\@currentlabel{(#1)}\label{cond:#1}}
\makeatother
\newcommand{\condref}[1]{\condstyle{\ref{cond:#1}}}
\newcommand{\Tone}{\condref{T1}\xspace}
\newcommand{\Ttwo}{\condref{T2}\xspace}
\newcommand{\Rone}{\condref{R1}\xspace}
\newcommand{\Rtwo}{\condref{R2}\xspace}
\newcommand{\Rthree}{\condref{R3}\xspace}
\newcommand{\oh}{\mathbf{1}}

\usepackage{tikz}
\usetikzlibrary{arrows.meta,calc,positioning,shapes.geometric,shapes.misc}

\usepackage{pgfplots}
\usepackage{pgfplotstable}
\usepgfplotslibrary{fillbetween}
\pgfplotsset{compat=1.18}
\definecolor{cellA}{RGB}{206,228,247}
\definecolor{cellB}{RGB}{206,236,214}
\definecolor{cellC}{RGB}{252,229,199}
\definecolor{cellD}{RGB}{245,211,229}
\definecolor{cellE}{RGB}{228,225,243}
\definecolor{codecol}{RGB}{160,35,35}
\definecolor{deccol}{RGB}{20,90,150}
\definecolor{failcol}{RGB}{225,30,30}

\definecolor{cfp64}{RGB}{31,119,180}
\definecolor{cfp32}{RGB}{255,127,14}
\definecolor{cbf16}{RGB}{44,160,44}
\definecolor{cfp16}{RGB}{214,39,40}
\definecolor{cfpa}{RGB}{148,103,189}
\definecolor{cfpb}{RGB}{140,86,75}

\begin{document}

\maketitle

\begin{abstract}
Learning robust and scalable finite-state tracking is fundamental to sequence processing. While linear recurrent neural networks (RNNs), linear attention, and state space models enable scalable parallel training through affine recurrences, their theoretical expressivity guarantees assume idealized arithmetic and do not extend to finite precision, where the parallel scan that makes them fast is itself a source of perturbation. We formalize finite-state tracking at finite precision and characterize \emph{length independence}: tracking that stays correct at every sequence length, at a precision cost that does not grow with the length. We show that length-independent state tracking requires two competing dynamics within a single map: \emph{contraction} to suppress numerical perturbations and \emph{separation} to keep distinct states apart. We prove that affine recurrences, which offer a single rate at each step to serve both roles, realize at most \emph{definite} automata at finite precision. Instead of treating scan compatibility as a restriction on the update map, we reinterpret it as a computational budget and introduce the \emph{Neural Finite-State Machine} (NFSM): a nonaffine, scan-compatible RNN built for length-independent finite-state tracking. On synthetic benchmarks spanning abelian and nonabelian groups, noninvertible monoids, and textual state-tracking tasks, affine baselines fail on every nondefinite task, most of them within a few hundred steps. A single NFSM layer instead learns the exact transition tables of every algebraic task, which certifies correctness beyond the tested lengths, and a stack of NFSMs keeps perfect accuracy on the textual tasks at every tested length.
\end{abstract}

\section{Introduction}
\label{sec:intro}
\emph{State tracking} is the maintenance of an internal memory variable that is sequentially updated as inputs arrive, capturing the cumulative effect of the past inputs. It enables models to resolve pronouns across texts \citep{Kim2023}, trace variable states during code execution \citep{Zaremba2015a}, and maintain beliefs about partially observable environments \citep{Kaelbling1998}. Synthetic tasks such as flip-flop retrieval, permutation composition, and modular counting isolate this capability directly \citep{Liu2022,Liu2023,Deletang2022}. Learning state tracking from data is a central challenge in sequence processing.

Recurrent neural networks (RNNs) naturally possess the structure required for state tracking \citep{Casey1996,Omlin1996}, but traditional training requires unrolling through time, leading to $O(L)$ sequential depth for sequence length $L$. To address this bottleneck, recent architectures leverage affine recurrences, which can be parallelized via scan algorithms to achieve $O(\log L)$ depth \citep{Blelloch,Martin2018}. This approach, which underlies modern State Space Models (SSMs) \citep{Smith2022,Gu2024}, gated linear RNNs \citep{Martin2018,Qin2023,Feng2024}, and linear attention \citep{Katharopoulos2020,Yang2024}, introduces a fundamental tradeoff between computational efficiency and the expressivity of the state dynamics. An extensive body of work has analyzed this tradeoff in exact or log-precision arithmetic \citep{Merrill2024,Sarrof2024,Grazzi2025,Siems2025a}, yielding architectures that theoretically attain full finite-state tracking capabilities \citep{Terzic2025}.

However, these expressivity guarantees rely on idealized arithmetic and often do not hold for trained models. Practical models operate at finite precision, learn parameters that only approximate optimal target solutions, and are trained on finite context windows. In practice, state tracking succeeds up to the training horizon but collapses toward chance not far beyond it \citep{Shakerinava2025,Karuvally2025,Grazzi2025,Terzic2025}. This collapse shows that theoretical expressivity is necessary but not sufficient: it does not survive the gap between idealized constructions and finite-precision, gradient-trained models. 
Most of this gap is a learning problem, but not all of it (Appendix~\ref{app:learning}). We address the part that no amount of training removes, finite precision, and focus on \emph{finite-state} tracking, where the state takes values in a finite set and evolves as in a finite automaton. We ask three questions: (i)~What conditions must an RNN satisfy to achieve state tracking at finite precision over arbitrary sequence 
lengths? (ii)~Can affine recurrences satisfy these conditions? (iii)~If not, can we design a scan-compatible RNN that guarantees robust state tracking for arbitrarily long sequences at finite precision?

We call a realization \emph{length-independent} when it tracks the state correctly at every sequence length at a fixed precision. 
Unlike empirical extrapolation, length independence requires an RNN to execute exact target transitions at every step \emph{at a precision cost that does not grow with $L$}. Correctness for any sequence length then follows by composition rather than empirical evaluation. We model finite-precision arithmetic as a perturbation of the hidden state bounded by $\eta>0$ at every step. 
The bound $\eta$ is set by the arithmetic resolution and the conditioning of the recurrence, and it includes the perturbation injected by the scan itself. Our analysis answers questions (i)--(iii) in turn.
\begin{airy}
\item Length-independent state tracking at finite precision reduces to three conditions on a single step, which together demand two competing dynamics within one map: 
\emph{contraction} to suppress numerical errors, and \emph{separation} to maintain distinguishable states (Section~\ref{sec:realization}).
\item Affine recurrences offer a single rate, which over repeated steps cannot serve both roles; consequently, at any precision $\eta > 0$, they realize at most \emph{definite} automata, whose state depends only on a bounded suffix of the input (Section~\ref{sec:tworates}).
\item Scan compatibility is a budget on the representation of the recurrence update maps, not a restriction to affine maps (Section~\ref{sec:budget}). Within this budget, the \cellname[full], a nonaffine recurrence, realizes every finite automaton with length independence and runs under a parallel scan (Section~\ref{sec:architecture}).
\end{airy}

To validate our theoretical results, we evaluate \cellname (Section~\ref{sec:experiments}) on synthetic state-tracking tasks spanning abelian and nonabelian groups, noninvertible monoids, and textual \emph{Tracking Shuffled Objects} tasks \citep{Srivastava2023}. 
Affine baselines fail beyond a finite length on every nondefinite task, whereas a single \cellname layer keeps perfect sequence accuracy at every tested length on the algebraic tasks, and a stack of two layers keeps perfect question accuracy on the textual tasks. On the algebraic tasks, the transition tables extracted from the trained models match the target exactly, which certifies correctness beyond the tested lengths instead of extrapolating it.

Section~\ref{sec:discussion} draws out the implications. Proofs and additional mathematical derivations are available in Appendix~\ref{app:proofs}, and Appendix~\ref{app:related} discusses related work. All source code is available at \url{https://julienbrandoit.github.io/length-independent-state-tracking/}, and readers are invited to follow along with our interactive illustration as they progress through the paper.

\section{Formalization}
\label{sec:formalization}

This section formalizes state tracking and connects finite-state automata to RNNs. We describe the transition structure of each system, show how an RNN can realize the dynamics of an automaton, and account for finite-precision perturbations.

\textbf{State tracking and automata.} We formalize state tracking using a semiautomaton (automaton for short), defined as a 3-tuple $(Q,\Sig,\delta)$, where $Q$ is a set of states, $\Sig$ is an input alphabet, and $\delta:Q\times\Sig\to Q$ is the transition function specifying how inputs change the state. In this work, we focus on state tracking with finite $Q$ of size $n$. A single input $x\in\Sig$, called a symbol, acts on the state space through $\delta_x\coloneq\delta(\cdot,x):Q\to Q$. We call a finite sequence of symbols a word, and denote the set of all words by $\Sigma^*$. For a word $w=x_1\cdots x_{|w|}\in\Sig^*$, single-symbol transitions chain to form the map $\delta_w:Q\to Q$, defined by $\delta_w\coloneq\delta_{x_{|w|}}\circ\cdots\circ\delta_{x_1}$, with $\delta_{\varepsilon}\coloneq\mathrm{id}_Q$ for the empty word $\varepsilon$. Thus, $\delta_w(q)$ is the state reached from $q$ after reading the word $w$. The set of all such maps, $\Mtar=\{\delta_w \mid w\in\Sig^*\}$, is the \emph{transition monoid} of the automaton, a submonoid of the \emph{full transformation monoid} $T_n\coloneq\{f:Q\to Q\}$, which contains all $n^n$ possible state mappings. Since every $n$-state task induces a submonoid of $T_n$, an architecture that can implement every element of $T_n$ as a transition covers all $n$-state tasks; the \cellname of Section~\ref{sec:architecture} is such an architecture (Proposition~\ref{prop:maximal}).

\textbf{Recurrent neural network.} An RNN carries a \emph{hidden state} $h_t$ in a space $\mathcal{H} \coloneq \R^d$, where $d$ is the hidden state dimension, and updates it through an input-driven update function $\Phi:\mathcal{H}\times\Sig\to\mathcal{H}$:
\begin{equation}
    h_t = \Phi(h_{t-1},x_t), \qquad t=1,\dots,L,
  \label{eq:rnn}
\end{equation}
from an initial hidden state $h_0$; the hidden state $h_t$ is the memory the RNN keeps of past inputs. Each input symbol $x$ applies a transformation $\Phi_x\coloneq\Phi(\cdot,x):\mathcal{H}\to\mathcal{H}$ to the hidden state space, and, as for $\delta_w$ above, processing a word $w=x_1\cdots x_{|w|}$ chains these into a composed map $\Phi_w\coloneq\Phi_{x_{|w|}}\circ\cdots\circ\Phi_{x_1}:\mathcal{H}\to\mathcal{H}$, with $\Phi_\varepsilon\coloneq\mathrm{id}_\mathcal{H}$. These maps form the \emph{RNN transition monoid}, $\Mrnn=\{\Phi_w\mid w\in\Sig^*\}$.
Throughout, \emph{state} refers to the automaton and \emph{hidden state} to the RNN.

The RNN architectures that have survived the scalability demands of modern deployment, namely SSMs, gated linear RNNs, and linear attention, are affine recurrences
\begin{equation}
h_t = A_{x_t}h_{t-1}+b_{x_t}, \qquad t=1,\dots,L,
  \label{eq:affinernn}
\end{equation}
where $A_{x}\in\R^{d\times d}$ and $b_{x}\in\R^{d}$ are input-driven, with arbitrary dependence on the symbol; we call such a recurrence \emph{affine in the hidden state}.
Their appeal is that affine maps compose in closed form, which lets a parallel scan evaluate them.

\textbf{Automaton realization.} The automaton evolves in a discrete state space and the RNN in a continuous one, so realizing an automaton by an RNN requires an \emph{encoding} $\enc:Q\to\mathcal{H}$ of automaton states as representations and a partial \emph{readout} $\dec:\mathcal{H}\rightharpoonup Q$ decoding them back, such that encoding a state, running the word through the RNN and reading out gives the same state as running the word through the automaton:
\begin{equation}
\dec\bigl(\Phi_w(\enc(q))\bigr) = \delta_w(q) \qquad\text{for every } q\in Q, w\in\Sig^*.
\label{eq:realize}
\end{equation}
Equivalently, $\Mrnn$ must contain, for every $\delta_w\in\Mtar$, a map $\Phi_w$ reproducing its action on $Q$ through $(\enc,\dec)$, so the class of automata an RNN can express is fixed by the achievable algebraic structure of $\Mrnn$. 
We say that an RNN \emph{realizes} the automaton $(Q,\Sig,\delta)$ if such a triple $(\enc,\Phi,\dec)$ exists. In this paper, we further ask which state-tracking properties survive computation at finite precision.

\textbf{Automaton realization at finite precision.} At finite precision, each RNN step lands near its intended image rather than exactly on it: $h_t = \tilde\Phi_{x_t}(h_{t-1})$, where $\tilde\Phi_{x_t}(h_{t-1})$ is the approximation of $\Phi_{x_t}(h_{t-1})$ that accounts for numerical perturbations of $x_t$ and of $\Phi$ itself. We write this as $\tilde\Phi_{x_t}(h_{t-1})=\Phi_{x_t}(h_{t-1})+e_t$, where $e_t$ is a per-step error (Appendix~\ref{app:finite_precision}, Lemma~\ref{lem:machine}).

\begin{definition}[$\eta$-trajectory and $\eta$-reachable set]
\label{def:noise}
Let $\eta>0$ denote the maximum per-step perturbation level, and let $\Ball\coloneq\{e \mid \|e\|\le\eta\}$ be the corresponding perturbation ball for a fixed norm $\|\cdot\|$ on $\mathcal{H}$. We call $(h_t)_{t=0}^{|w|}$ an $\eta$-trajectory of a word $w=x_1\cdots x_{|w|}$ from $h_0$ if $h_t-\Phi_{x_t}(h_{t-1})\in\Ball$ for every $t$, and define the $\eta$-reachable set of the recurrence as $U^\star\coloneq\{h_{|w|} \mid q \in Q, w \in \Sig^*, (h_t)_{t=0}^{|w|} \text{ is an } \eta\text{-trajectory of } w \text{ from } \enc(q)\}$.
\end{definition}

An $\eta$-trajectory formalizes a realistic, perturbed execution path of an RNN processing an input sequence, where no single update deviates by more than $\eta$ from its exact representation. The $\eta$-reachable set $U^\star$ is the set of all hidden states reachable from an encoded state, over words of any length, under bounded perturbations. For $U^\star$ to admit a representation using a finite number of bits regardless of input length, it must be bounded (Proposition~\ref{prop:precision}). Automaton realization at finite precision thus requires correct decoding along every $\eta$-trajectory and a bounded $\eta$-reachable set.

\begin{definition}[Length-independent realization at finite precision]
\label{def:realization}
A triple $(\enc,\Phi,\dec)$ realizes the target automaton $(Q,\Sig,\delta)$ with length independence at finite precision, modeled by a per-step perturbation level $\eta>0$, if, for the $\eta$-reachable set $U^\star$:
\begin{conditions}[label={\normalfont\condstyle{(T\arabic*)}},ref={(T\arabic*)}]        
    \item\condlabel{T1} \emph{Correct at every length.} $\dec(h_{|w|})=\delta_w(q)$ for every $q\in Q$, $w\in\Sig^*$, and every $\eta$-trajectory $(h_t)_{t=0}^{|w|}$ of $w$ from $\enc(q)$.
    \item\condlabel{T2} \emph{Fixed precision at every length.} $U^\star$ is bounded.
\end{conditions}
\end{definition}

\section{Length independence requires three conditions on one step}
\label{sec:realization}
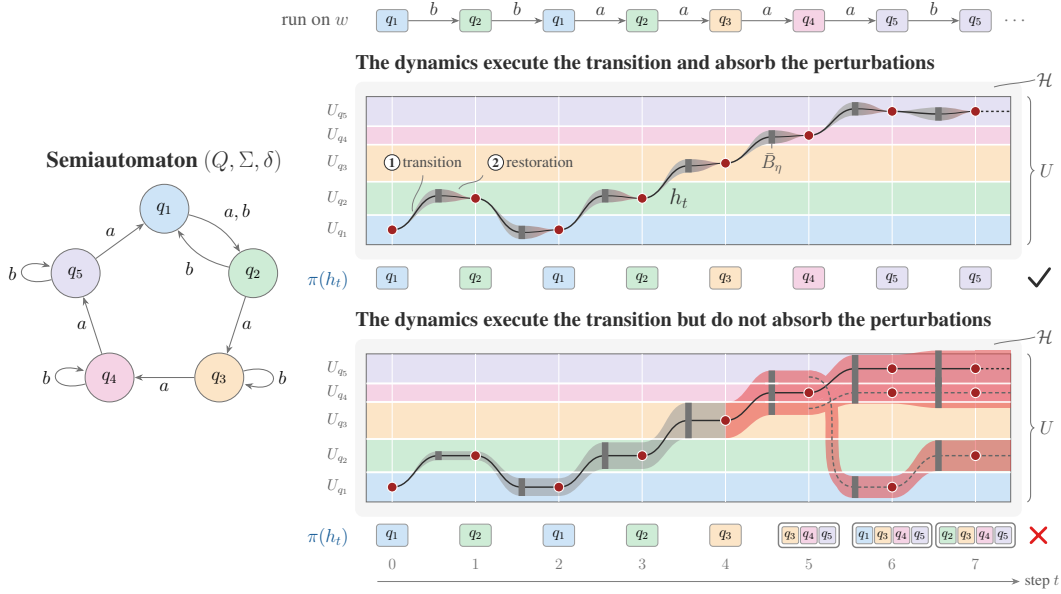
\begin{figure}
    \centering
    \resizebox{1\textwidth}{!}{%
        \input{assets/figure_1}%
    }
\caption{\textbf{Length independence requires executing the transition and restoring the state.} The same automaton (left), realized twice and unrolled over time; the top strip is its run on $w$. The vertical axis sketches $\mathcal{H}$, with the bounded region $U$ split into the cells $U_q=\{u\in U:\dec(u)=q\}$, drawn as bands, and the code points $\enc(q)$ as red dots. \textbf{Top}: \textcircled{\scriptsize 1} the update moves the hidden state into the next cell, and computing it spreads the state over $\Ball$; \textcircled{\scriptsize 2} restoration pulls that ball back onto the code point, so $\dec(h_t)=q_t$ at every length. \textbf{Bottom}: nothing pulls the hidden state back, so the ball grows by $\eta$ per step and at $t=5$ straddles $U_{q_3}$, $U_{q_4}$ and $U_{q_5}$; the single symbol $a$ then drives three runs at once (dashed) and $\dec(h_t)$ is not unique anymore.}
    \label{fig:realization}
\end{figure}

In this section, we first translate the trajectory-level conditions of Definition~\ref{def:realization} into equivalent step-level conditions on $\Phi$. We then show what these step-level conditions impose on the geometry of the RNN dynamics.

First, \Tone with the empty word requires $\dec\circ\enc=\mathrm{id}_Q$ (Lemma~\ref{lem:code}). Second, \Tone requires the perturbed RNN update $\tilde\Phi_x$ to map hidden states representing $q$ to hidden states representing the correct automaton successor $\delta_x(q)$ at every step on the $\eta$-reachable set $U^{\star}$ (Lemma~\ref{lem:transition}). Finally, \Ttwo, the boundedness of $U^{\star}$, keeps the representation within a fixed region of the hidden state space, so that a fixed, finite number of bits suffices at every sequence length (Proposition~\ref{prop:precision}). We formalize these conditions in Theorem~\ref{thm:equiv}, where $\oplus$ denotes the Minkowski sum, $A\oplus B = \{a+b \mid a\in A, b\in B\}$, and maps act setwise, i.e.,~$\Phi_x(\mathcal{S}) = \{\Phi_x(s) \mid s \in \mathcal{S}\}$.

\begin{restatable}[Conditions for length-independent realization at finite precision]{theorem}{thmequiv}
\label{thm:equiv}
A triple $(\enc,\Phi,\dec)$ realizes the target automaton $(Q,\Sig,\delta)$ with length independence at $\eta>0$ \emph{iff} there is a region $U\supseteq\enc(Q)$ on which $\dec$ is defined and that satisfies
\begin{conditions}[label={\normalfont\condstyle{(R\arabic*)}},ref={(R\arabic*)}]        
    \item\condlabel{R1} \emph{Decodable code.} $\dec\circ\enc=\mathrm{id}_Q$.
    \item\condlabel{R2} \emph{Restoration.} $U$ is bounded and $\Phi_x(U)\oplus\Ball\subseteq U$ for every $x\in\Sig$.
    \item\condlabel{R3} \emph{Correct transition.} $\dec(\Phi_x(u)+e)=\delta_x(\dec(u))$ for every $x\in\Sig$, $u\in U$ and $\|e\|\le\eta$.
\end{conditions}
Whenever such a region $U$ exists, the $\eta$-reachable set $U^{\star}$ itself satisfies \Rone--\Rthree, and $U^{\star}\subseteq U$.
\end{restatable}

\Rone to \Rthree reproduce, from a coding argument, von Neumann's division of labor: a chain of unreliable stages is reliable at arbitrary depth only if each carries an \emph{executive organ} performing the computation and a \emph{restoring organ} that ``erase[s] the degradation caused by the executive organs'' \citep{vonNeumann1956}. \Rone is the code, \Rtwo the restoring organ, and \Rthree the executive organ (Appendix~\ref{app:provenance}). What is specific to recurrence is that one map must act as both the restoring and the executive organ (Figure~\ref{fig:realization}).

We call the encoding $\enc(q)$ of the states \emph{code points}, and their set $\enc(Q)$ the \emph{code}. The readout partitions $U$ into disjoint cells $U_q\coloneq\{u\in U \mid \dec(u)=q\}$: \Rone places each code point in its own cell, \Rtwo keeps the partition $U$ bounded, and \Rthree demands that, under an input $x$, an entire cell $U_q$ land inside $U_{\delta_x(q)}$ with a margin of at least $\eta$, so that one-step perturbations cannot push the state out of the target cell.

\begin{restatable}[Cell partition]{theorem}{thmcells}
\label{thm:cells}
A triple $(\enc,\Phi,\dec)$ realizes the target automaton $(Q,\Sig,\delta)$ with length independence at $\eta > 0$ \emph{iff} there exists a bounded region $U \supseteq \enc(Q)$ partitioned into cells $U_q \coloneq \{u \in U \mid \dec(u) = q\}$ such that $\enc(q) \in U_q$ for all $q \in Q$, and
\begin{equation}
\Phi_x(U_q)\oplus\Ball\subseteq U_{\delta_x(q)} \qquad \text{for every } q \in Q \text{ and } x \in \Sig.
  \label{eq:cellincl}
\end{equation}
\end{restatable}

Theorems~\ref{thm:equiv} and~\ref{thm:cells} constrain single steps. Iterated along a word $w\in\Sig^+$ (with $\Sig^+=\Sig^{*}\setminus\{\varepsilon\}$), \eqref{eq:cellincl} imposes two geometric requirements on every composite map $\Phi_w$, uniformly in $|w|$:

\textbf{(1) Local contraction:} The image $\Phi_w(U_q)$ must be $2\eta$ narrower than the cell $U_{\delta_w(q)}$ it lands in, so that a perturbation cannot push it out. When $w$ returns to its starting state, $\delta_w(q)=q$, the cell shrinks strictly into itself: numerical errors are damped rather than compounded.

\textbf{(2) Global separation:} At the same time, the images of cells $U_q$ and $U_{q'}$ with $\delta_w(q)\ne\delta_w(q')$ must stay bounded away from one another, however long $w$ is, preventing $\eta$-perturbations from causing cross-cell misclassifications.

Length-independent state tracking therefore requires \emph{dual-rate dynamics}: local contraction inside cells paired with global separation between them.

\begin{restatable}[Two rates in one map]{theorem}{thmtworates}
\label{thm:tworates}
Under \eqref{eq:cellincl}, every composite map $\Phi_w$ is at once \emph{locally contracting} and \emph{globally separating}: for every $q,q'\in Q$ and every $w\in\Sig^+$,
\begin{enumerate}[label=\textbf{(\roman*)}]
    \item \textbf{Local contraction.} $\diam\big(\Phi_w(U_q)\big)\le\diam\big(U_{\delta_w(q)}\big)-2\eta$; in particular, $\diam\big(\Phi_w(U_q)\big)\le\diam(U_q)-2\eta$ whenever $\delta_w(q)=q$.
    \item \textbf{Global separation.} $\dist\big(\Phi_w(U_q), \Phi_w(U_{q'})\big) \ge 2\eta \quad \text{whenever} \quad \delta_w(q)\ne\delta_w(q')$.
\end{enumerate}
\end{restatable}

Figure~\ref{fig:realization} shows both properties at work, and what fails without restoration (Corollary~\ref{cor:failure}).

\section{Length independence requires a nonaffine recurrence}
\label{sec:tworates}

Length-independent realization imposes a dual constraint on $\Phi$ along every word: the image of each cell must be $2\eta$ narrower than the cell it lands in (Theorem~\ref{thm:tworates}(i)), while the images of cells heading to different states must stay $2\eta$ apart (Theorem~\ref{thm:tworates}(ii)), with margins that do not shrink as the word grows. This section shows that affine recurrences cannot satisfy this dual constraint along the words of a nondefinite target, and identifies what a recurrence requires to satisfy it.

\subsection{Affine recurrences supply a single rate}
\label{sec:affine}

In affine recurrences, the displacement between two trajectories depends only on their initial displacement:
\begin{equation}
\Phi_w(h)-\Phi_w(h') = A_w(h-h'), \quad A_w  \coloneq  A_{x_{|w|}}\cdots A_{x_1}.
  \label{eq:affine}
\end{equation}
An affine recurrence supplies a single rate $A_w$ over the whole hidden state space, hence the same operator determines the dynamics of both the signal and the perturbation at each time step (Definition~\ref{def:singlerate}, Proposition~\ref{prop:singlerate}). Along repeated words $w^k$ ($w$ written $k$ times, so $A_{w^k}=A_w^k$), the operator $A_w^k$ that must damp the perturbation also shrinks the gap between the codes of two states that the target keeps apart, until that gap falls below $2\eta$, which Theorem~\ref{thm:tworates}(ii) forbids. Selective (input-dependent) matrices modulate that single rate \emph{through time}, not across space, so they do not escape this.

Specifically, an affine recurrence loses either \Rtwo or \Rthree depending on the spectral radius $\rho(A_w)$: either $\rho(A_w)\ge1$ for some word, so perturbations accumulate or amplify, or $\rho(A_w)<1$ for every word, so older inputs are forgotten along with the perturbation. 
Forgetting old inputs is harmless only when the target never needs them, that is, when it is a \emph{definite} automaton, whose state depends only on the last $k$ inputs, for some fixed $k$ \citep{Perles1963}.
The first regime precludes length independence for every target, the second for every nondefinite one. Both still allow length generalization up to a horizon $L_{\max}$.

\begin{restatable}[Affine recurrences realize at most definite automata]{theorem}{thmaffine}
\label{thm:affine}
Let $\Phi$ be affine in the hidden state, $\Phi_x(h)=A_xh+b_x$ for $x \in \Sig$.
\begin{enumerate}[label=(\roman*)]
\item If $\rho(A_w) \ge 1$ for some word $w \in \Sig^+$, then \Rtwo \textbf{fails}: no nonempty bounded set $U \subset \mathcal{H}$ satisfies $\Phi_w(U) \oplus \Ball \subseteq U$ for this $w$. Consequently, $\Phi$ realizes no automaton with length independence at any $\eta > 0$.
\item If $\rho(A_w) < 1$ for all words $w \in \Sig^+$, then $\Phi$ realizes \textbf{at most a definite automaton} with length independence at $\eta > 0$. For any nondefinite target, no region satisfying \Rone and \Rtwo satisfies \Rthree.
\end{enumerate}
\end{restatable}

Most scan-compatible architectures fall on one of the two branches of Theorem~\ref{thm:affine} (see Table~\ref{tab:branches} in Appendix~\ref{app:census}), and Section~\ref{sec:experiments} shows both failure modes empirically. Concurrent work reaches the same limits of affine recurrences on error propagation by a different route \citep{Chung2026}.

\subsection{Two rates require multistability}
\label{sec:multistability}

Theorem~\ref{thm:affine} shifts the question from the \emph{spectrum} of a transition matrix to the \emph{geometry} of the update map. Specifically, realization demands a cell partition (Theorem~\ref{thm:cells}) that absorbs local perturbations while mapping each region cleanly into its prescribed successor, a dynamic reminiscent of the restoring organ of Section~\ref{sec:realization}.

The canonical neural system with this error-correcting behavior is the Hopfield network \citep{Hopfield1982}, but its restoration is \emph{static}: iterating $\Phi$ to convergence returns a corrupted state to \emph{the same} pattern it started near. A state tracker must execute a \emph{dynamic restoration}: a single step of $\Phi_x$ must simultaneously remove the perturbation \emph{and} move the hidden state to the correct \emph{successor} cell. Satisfying \eqref{eq:cellincl} thus requires a single update map to maintain multiple distinct cells, one per automaton state. When the target is not definite, some word $w$ acts on $Q$ as an idempotent $\delta_w$ that is not constant (Lemma~\ref{lem:definite}), so it fixes at least two states and, by Theorem~\ref{thm:tworates}, $\Phi_w$ shrinks each of their cells into itself while keeping them $2\eta$ apart. No affine map does this (Proposition~\ref{prop:affinecapacity}): the recurrence must be \emph{multistable}, which requires a nonaffine operation on the previous hidden state $h_{t-1}$ (Lemma~\ref{lem:affinefix}). Its \emph{robust capacity}, the number of bits one map can hold under perturbation (Definition~\ref{def:capacity}), is then $\log_2 n$, $n$ being the number of cells, with one attractor per cell and one cell per automaton state. A definite target needs neither multistability nor nonaffinity (Proposition~\ref{prop:definiteattained}).

This capacity is set by the number of cells rather than by the dimension: extra width allows for wider margins rather than more states, and extra depth leaves the temporal recurrence within each layer affine, so the robust capacity of an affine recurrence beyond a bounded suffix of the input is zero at any width (Proposition~\ref{prop:affinecapacity}) and any depth, whatever the nonlinear maps between layers (Theorem~\ref{thm:depth}).

\section{Length independence is compatible with a parallel scan}
\label{sec:budget}

To meet modern standards of scalability, a recurrence has to be \emph{compatible with the parallel scan}. 
The scan operates on state-update maps $\Phi_x$ and their compositions $\Phi_w$. Because function composition is associative, these maps form a monoid; the core requirement for a scan-compatible recurrence is therefore not associativity, but \emph{bounded representation}. The parallel scan combines these bounded composites pairwise across sequence steps, with costs governed by storage per composite ($b$~bits) and computation per composition ($c$~operations).

\begin{definition}[Scan-representable recurrence]
\label{def:scanrep}
The recurrence $h_t = \Phi(h_{t-1},x_t)$ is \emph{scan-representable with budget $(b,c)$} if there exists an injective representation $\zeta : \Mrnn \to \{0,1\}^b$ of the maps of $\Mrnn$, the identity $\Phi_\varepsilon$ included, and an algorithm computing $\zeta(\Phi_v \circ \Phi_u)$ from $(\zeta(\Phi_v),\zeta(\Phi_u))$ in $c$ operations, for all $\Phi_u, \Phi_v \in \Mrnn$.
\end{definition} 

Given per-symbol representations $\zeta(\Phi_{x_1}),\ldots,\zeta(\Phi_{x_L})$, and assuming that building $\zeta(\Phi_x)$ from a symbol and evaluating a composite at a hidden state also cost $O(c)$, a parallel scan computes $h_1,\ldots,h_L$ in depth $O(\log L)$, work $O(Lc)$, and memory $O(Lb)$ \citep{Blelloch}. The depth, the number of successive rounds of compositions, sets the running time with enough processors, whereas a sequential recurrence needs $L$ rounds.

\begin{restatable}[The budget forces a finite reachable set]{proposition}{propcollapse}
\label{prop:collapse}
If $\Phi$ is scan-representable with budget $(b,c)$, then $\Mrnn$ has at most $2^b$ elements and, for every initial hidden state $h_0$, so does the set $\{\Phi_w(h_0) \mid w\in\Sig^*\}$ of reachable hidden states, $h_0=\Phi_\varepsilon(h_0)$ included. On these $K\le2^b$ hidden states, $\Mrnn$ acts as a submonoid of $T_K$, the full transformation monoid on $K$ elements.
\end{restatable}

\textbf{The scan is itself a perturbation source.} 
Affine maps form a monoid only in exact arithmetic: floating-point operations are not associative, so the composite a scan returns differs from the sequential one, and an affine recurrence, lacking restoration, lets this error accumulate (Appendix~\ref{app:finite_precision}). The perturbation floor $\eta$ is thus as much a product of parallelizing an affine recurrence as an assumption imposed on it.

Multistable recurrences can resolve this mismatch by replacing continuous drift with discrete basins. Consider $K$ attractors $a_1,\dots,a_K\in V$, each in its own basin $a_k\in V_k$, whose basins $V_1,\dots,V_K$ partition $V$. Under \textbf{exact restoration}, any hidden state in a basin contracts in one step onto an attractor: $\Phi$ is \emph{exactly restoring} on $V$ if every $\Phi_x$ maps each basin $V_k$ to a single attractor $a_{\sigma_x(k)}$. The attractors and their basins are those of the restoring organ of Section~\ref{sec:realization}, one and the same map for every symbol; the executive organ then carries $a_k$ into the basin of $a_{\sigma_x(k)}$ and the restoring organ finishes the step, so $a_k$ is a fixed point of $\Phi_x$ itself exactly when $\sigma_x(k)=k$ (Remark \ref{rem:whichmap}). 
Consequently, the hidden states reachable from an attractor are attractors, and each symbol acts as a transition table $\sigma_x$ on the $K$ attractor indices, which meets the budget of Definition~\ref{def:scanrep}: a table costs $b = K\lceil\log_2 K\rceil$ bits, and composing two of them is a lookup per index, $c = O(K)$ operations.

This table-composition view explains architectures that achieve scan-compatible multistable dynamics. Memory Recurrent Units \citep{DeGeeter2026,Brandoit2026}, which assume that the hidden state converges onto an attractor between two steps, implicitly build such tables (Appendix~\ref{app:mru}).

Scan compatibility and length independence can coexist. Mapping attractor indices onto automaton states, with $K=n$, $a_q=\enc(q)$ and cells $U_q \coloneq a_q\oplus\Ball$ that fit inside the basins $V_q$, turns exact restoration into the cell inclusions of \eqref{eq:cellincl} as soon as every table is the target transition, $\sigma_x=\delta_x$: the recurrence then realizes the target with length independence (Theorem~\ref{thm:cells}) and fits the budget at $b=n\lceil\log_2 n\rceil$ and $c=O(n)$. Since composites of index maps are index maps, the merge is exact whatever the shape of the tree and the scanned trajectory is bit-identical to the sequential one (Proposition~\ref{prop:maximal}). Section~\ref{sec:architecture} builds such a recurrence for every finite automaton.

\section{The neural finite-state machine}
\label{sec:architecture}
The \cellname[full] is a multistable recurrence that learns the transition tables of Section~\ref{sec:budget} directly. We first describe one head as an RNN, then show how it realizes a finite automaton, and finally assemble heads into a block.

\textbf{One head.} A head has a state $h_t\in\R^d$ and, for each input $x$, a logit matrix $\theta_x\in\R^{d\times d}$ given by a learned map of the input. Writing $\rnd(h) \coloneq \oh_{\arg\max_j\oh_j^{\Tr}h}$ for the basis vector of the largest coordinate of $h$, where $\oh_j^{\Tr}h$ selects the $j$-th coordinate of $h$, the head updates its state as
\begin{equation}
\Phi_x(h)=\rnd\big(\theta_x\,\rnd(h)\big).
\label{eq:construction}
\end{equation}
The outer $\rnd$ is the restoring organ of Section~\ref{sec:realization}: it maps the output of the product with $\theta_x$ back to the code, absorbing every perturbation that leaves its largest entry in place. The inner $\rnd$ does the same for the stored state; on a trajectory it acts as the identity, since $h_t$ is already a basis vector. The product with $\theta_x$ is the executive organ: since $\theta_x\oh_k$ is column $k$ of $\theta_x$, composing it with the outer $\rnd$ sends $\oh_k$ to $\oh_{\sigma_x(k)}$, with $\sigma_x(k)$ the row of the largest entry of that column:
\[
\sigma_x(k)\coloneq\arg\max_{j}\,\oh_j^{\Tr}\theta_x\oh_k,
\qquad
\rnd(\theta_x\oh_k)=\oh_{\sigma_x(k)}.
\]
The head is thus multistable, with $d$ attractors $a_k\coloneq\oh_k$ and basins $V_k\coloneq\{h\in\R^d \mid \arg\max_j\oh_j^{\Tr}h=k\}$, the cones of states whose largest coordinate is the $k$-th. It is exactly restoring in the sense of Section~\ref{sec:budget}: in one step, each basin $V_k$ collapses onto the single attractor $a_{\sigma_x(k)}$. The head outputs a learned vector attached to its current attractor. The map $\rnd$ jumps wherever the largest coordinate changes, so the recurrence is not affine, as multistability requires (Lemma~\ref{lem:affinefix}).

\textbf{Realization.} As at the end of Section~\ref{sec:budget}, we now map attractor indices onto automaton states. Assume $d\ge n$, number the states $Q=\{1,\dots,n\}$, and take
\[
\enc(q)\coloneq\oh_q,
\qquad
\dec(h)\coloneq\arg\max_{q\in Q}\oh_q^{\Tr}h,
\]
that is, $\dec$ reads the index of the basin $V_k$ that contains $h$. \Rone holds, since $\dec(\enc(q))=\arg\max_{q'\in Q}\oh_{q'}^{\Tr}\oh_q=q$, so $\dec\circ\enc=\mathrm{id}_Q$. For \Rtwo, take as cells the $\eta$-balls around the code points, $U_q\coloneq\enc(q)\oplus\Ball$. The basins do not overlap around the attractors, so for $\eta$ small enough, independently of the length of the word and of the target, each cell $U_q$ lies inside its basin $V_q$, and the cells are pairwise disjoint. Their disjoint union $U\coloneq\bigsqcup_{q\in Q}U_q$ is bounded. One step collapses each cell onto an attractor, so $\Phi_x(U)\oplus\Ball\subseteq U$ as soon as $\sigma_x$ maps $Q$ into itself. One map thus carries the two rates of Theorem~\ref{thm:tworates}: every cell shrinks to a point, while distinct attractors stay more than $2\eta$ apart. \Rthree holds if and only if $\sigma_x=\delta_x$ on $Q$, that is, if the largest entry of column $q$ of $\theta_x$ sits in row $\delta_x(q)$ for every $q\in Q$. One head realizes any automaton on at most $d$ states with length independence (Propositions~\ref{prop:maximal} and~\ref{prop:realizes}). Finite precision can only perturb the logits, and a perturbation below half the smallest column margin leaves every table unchanged at every length (Proposition~\ref{prop:margin}).

\textbf{Scan compatibility.} Each table $\sigma_x$ depends on its symbol alone, through the logits $\theta_x$, so a head builds the tables of a sequence in parallel and composes them by a parallel scan of depth $O(\log L)$ whose merges are gathers, $(\sigma_y\circ\sigma_x)(k)=\sigma_y(\sigma_x(k))$. This is the table budget of Section~\ref{sec:budget}, $d\lceil\log_2 d\rceil$ bits and $O(d)$ operations, and every merge is exact (Proposition~\ref{prop:maximal}). The \cellname is one block within a potentially larger model. Training through the $\arg\max$ uses a straight-through Gumbel-softmax estimator \citep{Jang2017} (Appendix~\ref{app:constr-grad}). Appendix~\ref{app:walltime} reports wall-clock measurements of the scan implementation.

\textbf{Heads as memory channels.} A target often tracks several quantities at once, and giving each its own memory channel keeps them apart. The \cellname is therefore a block of $r$ heads that read the same input and update independently: head $i$ has $d_i$ attractor indices and holds one component of the joint state $h_t=\big(h^{(1)}_t,\dots,h^{(r)}_t\big)\in\mathcal{H}=\R^{d_1}\times\cdots\times\R^{d_r}$. The quantities need not be independent: a running product of permutations, for instance, can be carried point by point by several heads (Example~\ref{ex:sn}). Targets whose factors depend on one another are realized by stacking layers, each reading the attractor indices of the layer below (Proposition~\ref{prop:stack}). Heads also keep the cost low: $k$ independent quantities of $v$ values cost $k$ heads, $kv^2$ logits per step, rather than one head of $v^k$ attractor indices, $v^{2k}$ logits. The heads are scanned side by side, so their budgets add, and their outputs are concatenated and projected back to the model width.

\section{Experiments}
\label{sec:experiments}

We compare \cellname with three affine recurrences that cover both branches of Theorem~\ref{thm:affine}: Mamba$^{-}$, Mamba with negative eigenvalues \citep{Grazzi2025}, with $\rho<1$; AUSSM \citep{Karuvally2025}, with $\rho=1$; and PD-SSM \citep{Terzic2025}, with $\rho<1$ and a hardmax that selects the transition matrix but does not act on the hidden state.

The algebraic tasks present a sequence of generators of a target monoid, and the target at each position is the composite of the generators seen so far. They cover abelian groups ($\mathbb{Z}_2$, $\mathbb{Z}_{16}$), nonabelian groups ($S_3$, $S_4$, and the nonsolvable $A_5$ and $M_{11}$, the first Mathieu group), and the noninvertible flip-flop monoid, drawn in two ways: $\mathrm{DFF}_5$ never draws more than four consecutive identities, so that the state is a function of the last five inputs, whereas FF draws identity runs of unbounded length. The textual task $\mathrm{TSO}_N$ \citep{Srivastava2023} gives the same kind of algebra through text: swaps between $N\in\{3,4,5\}$ people holding $N$ items are described in natural language, for example ``\emph{alice has the key. alice swaps with bob.} $\dots$ \emph{who has the key?}'', and the answer must be read off the tracked state.

All models share one backbone (Appendix~\ref{app:training}): an encoder embeds the input tokens into a latent space, the recurrent layers act within this latent space, and a decoder maps it to the predicted states. Only the recurrent layers differ. The baselines use four layers of width $256$. The \cellname uses a single layer on the algebraic tasks, with one head on $\mathbb{Z}_2$, $\mathbb{Z}_{16}$ and the flip-flops, and several heads on the permutation groups, for instance $4$ heads of $11$ attractor indices on $M_{11}$, each tracking where one point is sent. On $\mathrm{TSO}_N$ it stacks two multi-head layers. Models are trained on sequences of a fixed length, with the state supervised at every position, and tested on longer ones. The failing length $L_{\max}$ is the smallest tested length at which accuracy falls below $90\%$ for the baselines and below $100\%$ for \cellname, where accuracy is the fraction of sequences whose predicted states are all correct on the algebraic tasks, and of correctly answered questions on $\mathrm{TSO}_N$. Appendix~\ref{app:experiments} details the baselines, tasks, model sizes, and training.

\begin{table}[htbp]
    \setlength{\belowcaptionskip}{\baselineskip} 
    \caption{\textbf{Failing length on algebraic and textual state-tracking tasks.} Each cell reports the median failing length $L_{\max}$ over the seeds that learned the task, and below it how many of the five seeds did (Appendix~\ref{app:seeds}). $\geq L$: no failure up to $L$, the longest sequence that fits on one GPU. $^{\dagger}$: the extracted tables match the target exactly, which certifies correctness at every length (Appendix~\ref{app:extract-algebraic}). $\times$: no seed learned the task. Brackets below the task names give the number of reachable states.}
    \label{tab:results}
    \centering
    \resizebox{1\textwidth}{!}{%
        \input{assets/figure_2.tex}%
    }
\end{table}

\textbf{\cellname tracks every task; affine baselines do not.} With a single layer, \cellname keeps perfect accuracy up to the longest tested length on every algebraic task and for all five seeds, including the nonsolvable groups and the noninvertible flip-flops, and a stack of two layers does the same on $\mathrm{TSO}_N$ (Table~\ref{tab:results}). The baselines, with four layers and larger states, fail within a few hundred steps on most tasks or do not learn them at all.

\textbf{Affine recurrences fail as Theorem~\ref{thm:affine} predicts.} The flip-flop pair isolates the definite boundary: $\mathrm{DFF}_5$ draws only definite words, which a contracting recurrence can hold, whereas FF does not. The two contracting baselines, Mamba$^{-}$ and PD-SSM, hold $\mathrm{DFF}_5$ up to the longest tested length but fail on FF, while AUSSM, which does not contract, fails on both. The learned trajectories show the two branches (Appendix~\ref{app:pca}). Where the baselines learn, the horizon they reach can vary widely across seeds (Appendix~\ref{app:seeds}). A horizon measured on one trained model therefore does not transfer to the next, which is why we ask for length independence rather than a long horizon.

\textbf{Extracted tables certify length independence.} Each \cellname step commits to a table entry, so the tables a trained model implements can be read off its logits (Appendix~\ref{app:table-extraction}). For every algebraic task and every seed, they match the target exactly. By Propositions~\ref{prop:realizes} and~\ref{prop:heads}, this implies correct tracking at every length, provided the logit margin of Proposition~\ref{prop:margin} exceeds the error of the logit computation. On $\mathrm{TSO}_N$, Appendix~\ref{app:tso-fsm} shows how the heads of the two layers track the registers of the task.

\section{Discussion}
\label{sec:discussion}
Scan compatibility, expressivity, and robustness to finite precision are often viewed as competing demands; we have shown that a single mechanism meets all three. A bounded representation budget forces the recurrence onto a finite set of reachable hidden states, and a well-separated set of them lets one map both erase the step-wise error and execute the next transition. An affine recurrence has a single rate, so at any finite precision it either lets perturbations grow or realizes at most a definite automaton, at any width and depth. For models built on contracting recurrences, this is an \textit{undeclared context window}: older inputs reach the output only through a term that decays geometrically, so beyond a length set by learned gates and the arithmetic floor they are indistinguishable from noise, and that length is exceeded without any signal. The \cellname removes this barrier and leaves learning, the main practical obstacle, checkable rather than solved: the automaton a trained model implements can be read off its weights, so long-horizon correctness is verified rather than extrapolated (Appendix~\ref{app:learning}). Behaviors that need a state set growing with the input, such as unbounded counting, remain open (Appendix~\ref{app:discussion}).

\newpage
\subsection*{AI use statement}

In this work, we used generative AI tools in a supporting role for for formulating mathematical claims, assisting in the writing of proofs, proposing or refining hypotheses, designing or providing feedback on research methodology or experiments, implementing methods, and assisting with translation. 

We did not use generative AI tools for interpreting results or providing critical ingredients for proving mathematical claims. Generating synthetic datasets, developing theoretical models or conceptual frameworks, cleaning or reformatting datasets, and supporting qualitative or thematic data analysis are not applicable to this work.

Additionally, we partially used generative AI tools for identifying relevant literature, creating or editing software code, and editing the research paper to improve readability.

We have reviewed all AI-assisted work. Specifically, all mathematical claims, proofs, and methodological designs were manually checked for rigor and technical accuracy by the authors; AI-generated and edited code was verified and unit-tested; translated text was checked for semantic accuracy; and all referenced literature and information retrieved via AI were independently verified against original source publications. We take responsibility for the final content of this work, including text, claims or artifacts produced with the aid of generative AI.

\subsection*{Ethics statement}

The research presented in this work focuses on general and foundational aspects of artificial intelligence. It does not involve human subjects, sensitive personal data, or direct deployment in high-risk domains. Consequently, we do not believe that our work raises any direct or specific ethical concerns, privacy violations, or immediate risks of misuse.

However, as a contribution to the broader field of artificial intelligence, we are fully conscious of the far-reaching impact that AI tools and methodologies exert on society. AI developments carry significant implications across social, economic, legal, and environmental dimensions, including algorithmic fairness, economic shifts, and the resource demands of computational infrastructure. We strongly believe that these broader socio-technical and environmental impacts must be taken into account in future technical developments, policy formulation, and governance decision-making.

\subsection*{Reproducibility statement}

We are committed to ensuring the reproducibility of all theoretical claims and experimental findings presented in this paper. To facilitate full reproducibility, the entire source code of this project is open source. All empirical evaluation datasets used in this work are completely synthetic and can be directly generated, modified, and evaluated using the provided codebase. During the review period, an anonymized repository containing the full implementation, configuration files, and instructions for replicating all results and datasets is accessible at the following link: \url{https://julienbrandoit.github.io/length-independent-state-tracking/}. Furthermore, complete proofs of all theoretical claims and additional experimental details are given in the Appendix.



\bibliography{iclr2027_conference}
\bibliographystyle{iclr2027_conference}

\appendix
\clearpage
\section*{Contents of this appendix}
\begin{itemize}
\item Appendix~\ref{app:notation}: \nameref{app:notation}
\item Appendix~\ref{app:discussion}: \nameref{app:discussion}
\item Appendix~\ref{app:related}: \nameref{app:related}
\item Appendix~\ref{app:finite_precision}: \nameref{app:finite_precision}
\item Appendix~\ref{app:learning}: \nameref{app:learning}
\item Appendix~\ref{app:proofs}: \nameref{app:proofs}
\item Appendix~\ref{app:construction}: \nameref{app:construction}
\item Appendix~\ref{app:experiments}: \nameref{app:experiments}
\end{itemize}

\section{Notation}
\label{app:notation}

\begin{table}[ht]
\caption{Notation.}
\label{tab:notation}
\begin{center}
\footnotesize
\setlength{\tabcolsep}{4pt}
\begin{tabular}{@{}L{1.3cm}L{4.9cm}L{1.3cm}L{4.9cm}@{}}
\toprule
\multicolumn{2}{@{}l}{\emph{Target}} & \multicolumn{2}{l@{}}{\emph{Model}} \\
$\Sig,\Sig^*$ & alphabet; free monoid of words, with $\Sig^+=\Sig^{*}\setminus\{\varepsilon\}$.  & $\mathcal{H}=\R^d$ & hidden state space \\
$Q$, $n$ & automaton state set; $n \coloneq |Q|\ge2$ & $\Phi_x$, $\Phi_w$ & per-symbol map; composite
  $\Phi_{wx} \coloneq \Phi_x\circ\Phi_w$ \\
$\delta_x,\delta_w$ & symbol action $\delta_x \coloneq \delta(\cdot,x):Q\to Q$; extension to
  words, $\delta_w \coloneq \delta_{x_{|w|}}\circ\cdots\circ\delta_{x_1}$ & $A_x,b_x$ & affine data,
  $\Phi_x(h)=A_xh+b_x$ \\
$\Mtar$ & transition monoid, $\Mtar=\{\delta_w\mid w\in\Sig^*\}\le Q^Q$, under
  composition & $A_w$ &
  $A_{x_{|w|}}\cdots A_{x_1}$ for $w=x_1\cdots x_{|w|}$ \\
$Q^Q$ & all maps $Q\to Q$ under composition; $Q^Q\cong T_n$, $\Mtar\le Q^Q$ & $\rho$, $\kappa$ &
  spectral radius; $\kappa \coloneq \max_x\|A_x\|_\op$, with $\|A\|_\op \coloneq \sup_{v\ne0}\|Av\|/\|v\|$ \\
$T_n,S_n$ & full transformation monoid; symmetric group & $\Mrnn$ & RNN transition monoid, $\Mrnn \coloneq \{\Phi_w \mid w\in\Sig^*\}$ \\
\midrule
\multicolumn{2}{@{}l}{\emph{Arithmetic and error}} & \multicolumn{2}{l@{}}{\emph{Code}} \\
$[\,\cdot\,]_\mathcal{Q}$ & representable object standing for a given one
  (Appendix~\ref{app:finite_precision}) & $\enc,U,\dec$ &
  encoding; decodable region; readout \\
$\mathcal{Q}$ & representable points $[\mathcal{H}]_\mathcal{Q}$ &
  $U_q$ & the cell $U\cap\dec^{-1}(q)$ \\
$\eta$, $\Ball$ & per-step perturbation bound; ball $\{e \mid \|e\|\le\eta\}$ & $U^\star$ & $\eta$-reachable set
  of Definition~\ref{def:noise} \\
$\oplus$ & Minkowski sum & $\diam$, $\dist$ & diameter of a set; distance between two \\
$\eta$-traj. & any $(h_t)$ with $\|h_t-\Phi_{x_t}(h_{t-1})\|\le\eta$ & $a_k$, $V_k$, $\sigma_x$ & attractors; their basins; table on attractor indices (Sections~\ref{sec:budget} and~\ref{sec:architecture}) \\
& & $P_x$, $\rnd$ & column one-hot matrix of $\sigma_x$ (Appendix~\ref{app:constr-realizes}); restoring map to the basis vector of the largest coordinate \\
$c_a$, $S$ & constant map of value $a$; transition semigroup $\{\delta_w\mid w\in\Sig^+\}$ &
  $\oh_j$, $\theta_x$ & basis vector $j$, attractor of index $j$; logit matrix of \eqref{eq:construction} \\
$\zeta$, $b$, $c$ & scan representation and its budget (Definition~\ref{def:scanrep}) &
  $r$, $d_i$, $\gamma$ & number of heads and their numbers of attractor indices (Section~\ref{sec:architecture}); logit margin
  (Appendix~\ref{app:constr-margin}) \\
$\tilde\Phi_x$, $e_t$ & executed step; per-step error $\tilde\Phi_{x_t}(h_{t-1})-\Phi_{x_t}(h_{t-1})$ & $K$ & number of reachable states (Proposition~\ref{prop:collapse}) \\
& & $L_{\max}$ & failing length (Section~\ref{sec:experiments}) \\
& & $\chi$ & injective map of $Q$ into the joint attractor indices of the heads (Appendix~\ref{app:constr-heads}) \\
\bottomrule
\end{tabular}
\end{center}
\end{table}
\clearpage
\section{Extended discussion}
\label{app:discussion}
Robust state tracking asks one map to both erase the perturbation the previous step left behind, and move the state to its successor. An affine recurrence has a single rate, so it buys one at the expense of the other and its horizon is finite at any precision. \cellname does both, because a step is a rounding followed by a table lookup.

\textbf{The horizon is not an artifact of precision.} Our argument turns on the existence of a perturbation, not its amplitude. For any $\eta>0$ a contracting affine recurrence has a finite horizon, and $\eta$ only sets where that horizon falls; more precise formats moves it, but does not remove it. Nor is it an artifact of the hardware: the scan that makes an affine recurrence fast reassociates its products and so supplies $\eta$ by itself, whatever the format. \cellname escapes the horizon not by being more precise but by returning the state to a code point at every step, so that perturbations are annihilated instead of accumulated and, under a scan, are not created at all.

\textbf{An undeclared context window.} For language models built on contracting recurrences the consequence is practical. The realized automaton is definite up to precision, so inputs older than $k$ steps reach the output only through a term of size $O(\kappa^k)$: the model has a context window, but an undeclared one. Its length is set by learned gates and the arithmetic floor rather than by specification, it varies with the input as selectivity varies the decay rate, and, unlike overflow in attention, it is exceeded without any signal.

\textbf{Forgetting as computation, not decay.} Krohn--Rhodes decomposes every finite-state behavior into simple groups, which remember reversibly, and flip-flops, which remember until reset \citep{Krohn1965}. Realizing both makes forgetting a chosen property of the transitions rather than an unwanted residual of the dynamics: what is discarded is chosen by the arriving symbol, a reset erases exactly what it collapses, and what no reset targets survives indefinitely. A contracting recurrence offers erasure through contraction, which is graded and unconditional. Persistence is therefore not the price of forgetting; it is what remains once forgetting is expressed as computation.

\textbf{Beyond gradient descent.} Because the executive organ followed by the restoring organ is a table read off the logits by an $\arg\max$, the automaton a trained \cellname implements can be extracted, inspected, and compared to a target. The same object can in principle be inferred in context rather than learned by gradient descent: a table is small enough to be written from a handful of demonstrations, which makes automaton induction at inference time, rather than in the weights, a natural next question.

\textbf{Limitations and outlook.} Our guarantees cover finite-state tracking. Behaviors requiring unbounded memory, such as counting without a modulus or stack-like nesting, fall outside the transformation-monoid setting and need a state set that grows with the input; how to keep restoration meaningful in that regime is open. The number of states is also a design choice, and while heads and their sizes can be over-provisioned, the trade-off between capacity and optimization is not yet characterized. Finally, durable state is exactly what partially observable control requires, so \cellname is a natural candidate for reinforcement learning agents whose belief must survive long episodes, which we leave to future work.

\section{Related work}
\label{app:related}
This appendix situates our results among five lines of work: parallel scans and the affine recurrences built on them, the expressivity of these recurrences, their length generalization and finite-precision limits, scan-compatible recurrences that are not affine, and recurrent networks as automata. Three subsections then go further: the classical roots of the restoration requirement (Appendix~\ref{app:provenance}), a survey placing scan-compatible affine architectures on the branches of Theorem~\ref{thm:affine} (Appendix~\ref{app:census}), and Memory Recurrent Units read as transition tables (Appendix~\ref{app:mru}).

\textbf{Parallel scans and affine recurrences.} A parallel scan computes all prefixes of a sequence under an associative operation in logarithmic depth \citep{Blelloch}. \citet{Martin2018} used it to train linear recurrences over the sequence length. Linearity in the hidden state has since become the design principle of scalable recurrent models: structured and diagonal state space models \citep{Gu2021,Gu2022,Gupta2022,Smith2022}, selective ones \citep{Gu2024,Dao2024,Lahoti2025}, gated linear RNNs \citep{Orvieto2023,Qin2023,Qin2024,De2024,Feng2024}, and linear attention with its gated and delta-rule variants \citep{Katharopoulos2020,Yang2024a,Yang2024,Yang2024b}. They parallelize over the sequence in different ways, through a parallel scan \citep{Smith2022,Gu2024,Orvieto2023,Feng2024}, a convolution for the time-invariant ones \citep{Gu2021,Gu2022,Gupta2022}, or a chunkwise algorithm \citep{Yang2024a,Yang2024,Dao2024}, but all of them rely on the update being affine in the hidden state, so that the composite of two steps has a closed form. Appendix~\ref{app:census} surveys them from this angle. We take the opposite reading of the same constraint: what a scan requires is a bounded representation of the composites (Definition~\ref{def:scanrep}), not an affine update, and Proposition~\ref{prop:collapse} shows what this budget implies.

\textbf{Expressivity of affine recurrences.} A large body of work asks which automata these recurrences can express. At logarithmic precision and under standard complexity assumptions, state space models such as S4 and Mamba cannot track nonsolvable groups such as $S_5$ \citep{Merrill2024}. At finite precision, state space models with nonnegative gates model exactly the star-free regular languages \citep{Sarrof2024}, and a single layer of input-dependent complex diagonal transitions tracks no nonabelian group, while $k$ such layers track exactly the groups with a subnormal series of length $k$ with abelian factors \citep{Shakerinava2025}. A line of architectures then enlarges the reachable transition monoids: negative eigenvalues bring parity, and products of generalized Householder matrices bring every regular language \citep{Grazzi2025}, with the number of factors trading expressivity for cost \citep{Siems2025a}, and an $\mathrm{NC}^1$-complete problem within attention \citep{Yang2025}; adaptive unitary transitions \citep{Karuvally2025} and generalized delta rules \citep{Peng2025} target state tracking directly, and the column one-hot transitions of PD-SSM reach every finite automaton with as many hidden dimensions as states \citep{Terzic2025}. These results hold at logarithmic precision \citep{Merrill2024}, at a precision that grows logarithmically with the length \citep{Karuvally2025}, or at finite precision \citep{Sarrof2024,Grazzi2025,Siems2025a,Shakerinava2025}, but in none of them is the computation perturbed. In the finite-precision model of \citet{Sarrof2024}, numbers carry a fixed number of fractional bits and an unbounded number of integer bits, and arithmetic on them is exact: a gate equal to one then holds an exactly representable state forever, which is how a state space model with nonnegative gates models the flip-flop and, through the Krohn--Rhodes decomposition, every star-free language. Our model instead adds a perturbation of size up to $\eta$ at every step and asks for correctness under every such perturbation (Definition~\ref{def:noise}). A gate equal to one then accumulates the perturbations, which is branch~(i) of Theorem~\ref{thm:affine}, and a nondefinite target such as the flip-flop FF is out of reach of every affine recurrence. The two results are consistent: theirs describe what an affine recurrence expresses when its arithmetic does not drift, ours what it retains when it does, as under a floating-point scan (Section~\ref{sec:budget} and Appendix~\ref{app:finite_precision}). A second difference is the target: they ask for the recognition of a language, we ask for the realization of the automaton at every step. Expressivity and robustness to finite precision are separate axes.

\textbf{Length generalization and finite precision.} Empirically, state tracking learned by affine recurrences holds up to around the training length and collapses not far beyond it \citep{Shakerinava2025,Karuvally2025,Grazzi2025,Terzic2025}, and synthetic benchmarks isolate the failure: automata and group word problems \citep{Liu2022,Deletang2022}, flip-flops with long gaps \citep{Liu2023}, and tracking shuffled objects \citep{Srivastava2023}.  The $\eta$-trajectory of Definition~\ref{def:noise} is the pseudo-orbit used by \citet{Casey1996} to analyze recurrent networks under bounded noise, while \citet{Maass1998} show that arbitrarily small stochastic analog noise reduces discrete-time analog computation to the power of finite automata; and that rounding changes the dynamics of a recursion, rather than only blurring them, has long been known from the limit cycles of recursive digital filters \citep{Parker1971,Kaneko1973}; under a scan, the nonassociativity of floating-point arithmetic makes the scan itself a perturbation source (Section~\ref{sec:budget}). Concurrent work reaches the same limits of affine recurrences through the dynamics of error propagation \citep{Chung2026}. Our contributions are the step-level characterization of Theorem~\ref{thm:equiv}, the two-branch classification of Theorem~\ref{thm:affine}, its extension to any depth (Theorem~\ref{thm:depth}), and a scan-compatible recurrence that escapes it.

\textbf{Scan-compatible recurrences that are not affine.} Two routes lead beyond affine updates without giving up parallel training. The first keeps an arbitrary nonlinear recurrence and solves for its whole trajectory at once: DEER \citep{Lim2024} and ParaRNN \citep{Danieli2026} cast the sequence of updates as a system of equations and solve it by Newton iterations, each of which scans the affine linearization of the recurrence. They make nonlinear recurrences such as GRUs and LSTMs trainable at scale, but they return the sequential trajectory only up to the tolerance of the solver, which adds to $\eta$ instead of removing it, and they place no constraint on the dynamics of the recurrence itself. The second route, which we take, chooses a nonlinear recurrence whose composites have a bounded exact representation. Memory Recurrent Units \citep{DeGeeter2026,Brandoit2026} are the closest prior instance: they assume that the hidden state converges onto an attractor between two steps, and on the attractors each symbol acts as a transition table (Appendix~\ref{app:mru}). The \cellname makes this table explicit and learns it for every element of $T_d$ (Proposition~\ref{prop:maximal}). PD-SSM \citep{Terzic2025} shares its column one-hot structure but applies its hardmax to the input, not to the hidden state: its step stays affine in $h_{t-1}$, and its diagonal factor makes it contract (Appendix~\ref{app:baselines}). Restoration requires the nonlinearity to act on the hidden state (Lemma~\ref{lem:affinefix}).

\textbf{Recurrent networks as automata.} That recurrent networks can implement finite automata, and must organize their hidden states into attractors to do so robustly, is classical \citep{Omlin1996,Omlin1996a,Casey1996}; Appendix~\ref{app:provenance} discusses this lineage and what changes here. The extraction of finite-state machines from trained networks \citep{Casey1996} infers the automaton after training, from the organization of the hidden state space. In the \cellname the automaton is the parameterization itself: each step commits to a table entry, so the extracted tables are exactly the ones the model executes (Appendix~\ref{app:table-extraction}). Finally, the algebraic view of targets as transition monoids, with groups and flip-flops as the building blocks of the Krohn--Rhodes decomposition \citep{Krohn1965}, motivates both our benchmarks and the stacking of layers (Remark~\ref{rem:factor}).

\subsection{Provenance of the restoration requirement and the target}
\label{app:provenance}

This subsection places \Rone to \Rthree and the target of Definition~\ref{def:realization} in the lineage of classical work, and states what changes here.

\textbf{Roots.} The restoration requirements have a strong connection to classical work. Separation is linked to \citet{Shannon,Shannon1949}: signals must be distinguishable to be told apart at all and repeated restoration relates to \citet{Oliver1948} and \citet{vonNeumann1956}: error must be erased at every stage, not once at the end. \Rone to \Rthree inherit both. Likewise, the $\eta$-trajectory of Definition~\ref{def:noise} is the $\eta$-pseudo-orbit of dynamical systems, used in this form for recurrent networks by \citet{Casey1996}; and that a bounded set closed under one perturbed step forces a strictly stable map is standard in set-invariance control \citep{Rakovic2005}.

\textbf{Reliable computation from unreliable parts.} \citet{vonNeumann1956} shows that reliable computation from unreliable components requires executive and restoring organs in series at every stage, described in 1952 for a setting containing no recurrent networks. \citet{Winograd1963} shows that replication without restoration admits no computation capacity: for a redundant automaton the probability of malfunction falls as $P\approx d\exp(-c/R)$ and so vanishes only as the redundancy grows without bound, where for a coded channel $P\approx2^{-n(C-R)}$ vanishes with block length at any fixed rate $R<C$. Their conclusion is that such automata ``do not exhibit noise-free behavior of the type that would permit computation capacities to be defined,'' and that what buys capacity is complexity rather than copies.

\textbf{Recurrent networks as automata.} \citet{Hopfield1982} gives the canonical multistable recurrent system: an energy descent whose minima are the stored patterns, with basins that absorb perturbation, a restoring organ with the identity as index map. \citet{Omlin1996} exhibit the state-tracking version, attractor-based encodings of arbitrary deterministic finite automata, stable at arbitrary word length with finite weights, with the sigmoid-discriminant analysis in \citet{Omlin1996a}. \citet{Casey1996} supplies the converse, that a recurrent network performing a finite-state computation must organize its state space into attractors mirroring the states of the minimal automaton.\footnote{The noise corollary of \citeauthor{Casey1996} was subsequently found technically incorrect \citep{Casey1998}; the attractor-structure theorem is unaffected, and \citet{Maass1998} give the sharper and correct statement.} Studying finite-state machine extraction from trained networks, with no architecture class in view, their Theorem~3.1 shows that a network robustly performing an automaton computation carries mutually disjoint closed sets with nonempty interior, one per state of the minimal DFA, on which its behavior agrees with the automaton; their Corollary~3.1 then recovers finiteness from compactness of the phase space. \citet{Maass1998} show that arbitrarily small analog noise reduces discrete-time analog computation to finite automata, and to the definite languages in particular together with \citet{Maass1999}. The modern Hopfield line \citep{Ramsauer2021} replaces the discrete update with a softmax and recovers attention, which moves the construction out of the recurrent setting and so does not bear on \Rtwo.

\textbf{What is new here.} In Theorem~\ref{thm:cells}, the partition is \emph{characterized} rather than shown necessary, so the conditions can be checked of a candidate architecture and not only extracted from a working one. The target is \emph{realization of the automaton} rather than recognition of a language, which is worst-case over every word and compositional. And the cells carry a \emph{margin} $\eta$ that does not shrink with $t$, rather than being merely disjoint. The third change is central: disjointness is topological, preserved by every injective map, and therefore cannot distinguish one architecture from another, where a margin is metric: Theorem~\ref{thm:tworates}(i) and~(ii) convert it into two demands along every word, which a single rate cannot meet along repeated words (Theorem~\ref{thm:affine}). \Rtwo also states restoration \emph{deterministically}, as a set inclusion with a worst-case guarantee, where every source above states it probabilistically or asymptotically, and that is what makes it composable with a budget in Section~\ref{sec:budget}.


\subsection{A review of scan-compatible affine recurrences}
\label{app:census}

Most recurrent architectures proposed for scalable sequence modeling since S4 are affine in the state, so that a parallel scan composes them with the closed-form merge $(A_2,b_2)\circ(A_1,b_1)=(A_2A_1,A_2b_1+b_2)$, up to rounding (Section~\ref{sec:budget}). Existing comparisons organize them by family, by transition structure or by exact-arithmetic expressivity \citep{Merrill2024,Sarrof2024,Grazzi2025}; Table~\ref{tab:branches} places them instead on the branches of Theorem~\ref{thm:affine}.

\textbf{Method and reading.} For every architecture, we read off the transition matrix $A_t$ of \eqref{eq:affine} from the parameterization stated in its paper and compute its spectrum. A gate or eigenvalue modulus that \emph{attains} $1$ places the architecture on the $\rho=1$ branch, where \Rtwo fails; one that only approaches $1$ places it on $\rho<1$, where \Rthree fails and the definite automata remain. The survey records what each paper specifies, not what a trained model does, and is not exhaustive.

\textbf{Caveats.} A spectral radius is asymptotic: a nonnormal $A$ with $\rho(A)\le1$ can amplify by a polynomial factor before it decays \citep{Trefethen2005}. And when the matrix is chosen by the input, long products are governed by the joint spectral radius of the family \citep{Jungers2009}, whose boundedness is undecidable in general \citep{Blondel2000}; this is why Theorem~\ref{thm:affine} quantifies over words. Both caveats are vacuous for normal transitions (diagonal, scalar, generalized Householder, and factor by factor for DeltaProduct), for which $\|A_w\|\le\prod_i\rho(A_{x_i})$.

\begin{center}
\setlength{\tabcolsep}{3pt}
\setlength{\LTcapwidth}{\textwidth}
\begin{longtable}{@{}>{\scriptsize\raggedright\arraybackslash}p{3.0cm}>{\scriptsize\raggedright\arraybackslash}p{4.4cm}>{\scriptsize\raggedright\arraybackslash}p{3.3cm}>{\scriptsize\raggedright\arraybackslash}p{2.6cm}@{}}
\caption{Surveyed scan-compatible affine architectures, grouped by the branch of Theorem~\ref{thm:affine}. For each, the recurrence as its paper writes it, the transition matrix $A_t=A_{x_t}$ of \eqref{eq:affine}, and its spectrum; keys $k_t$ and $w_t$ have unit norm.}\label{tab:branches}\\
\toprule
Architecture & Recurrence & Transition $A_t$ & Spectrum \\
\midrule
\endfirsthead
\multicolumn{4}{@{}l@{}}{\scriptsize \emph{Table~\ref{tab:branches}, continued}} \\
\toprule
Architecture & Recurrence & Transition $A_t$ & Spectrum \\
\midrule
\endhead
\bottomrule
\endfoot
\multicolumn{4}{@{}l@{}}{\scriptsize \textbf{$\rho>1$: loses contraction, fails \Rtwo; realizes nothing.}
\emph{Not targeted in our survey}\textsuperscript{a}} \\
\addlinespace[3pt]
\multicolumn{4}{@{}l@{}}{\scriptsize \textbf{$\rho=1$: loses contraction, fails \Rtwo; realizes nothing,
$L_{\max}=O(1/\eta)$ for isometries}\textsuperscript{b}} \\
\rowcolor{black!6}
linear attention\newline \citep{Katharopoulos2020}
  & $S_t=S_{t-1}+\phi(k_t)v_t^\Tr$
  & $I$
  & $\{1\}$, multiplicity $d$ \\
DeltaNet\newline \citep{Yang2024}
  & $S_t=S_{t-1}(I-\beta_tk_tk_t^\Tr)+\beta_tv_tk_t^\Tr$
  & $I-\beta_tk_tk_t^\Tr$, $\beta_t=\sigma(\cdot)\in(0,1)$
  & $\{1^{(d-1)},\,1-\beta_t\}$, $1-\beta_t\in(0,1)$ \\
\rowcolor{black!6}
{}+ negative eigenvalues\newline \citep{Grazzi2025}
  & $H_t=(I-2\beta_tk_tk_t^\Tr)H_{t-1}+\beta_tk_tv_t^\Tr$
  & $I-2\beta_tk_tk_t^\Tr$, $\beta_t\in(0,1)$
  & $\{1^{(d-1)},\,1-2\beta_t\}$, $1-2\beta_t\in(-1,1)$ \\
DeltaProduct\newline \citep{Siems2025a}
  & $H_t=A_tH_{t-1}+B_t$
  & $\prod_{j=1}^{n_h}(I-\beta_{t,j}k_{t,j}k_{t,j}^\Tr)$, $\beta_{t,j}\in(0,1)$ or $(0,2)$
  & $1$ with multiplicity $\ge d-n_h$ \\
\rowcolor{black!6}
PaTH\newline \citep{Yang2025}
  & $S_t=S_{t-1}H_t+v_tk_t^\Tr$
  & $H_t=I-\beta_tw_tw_t^\Tr$, $\beta_t\in(0,2)$
  & $\{1^{(d-1)},\,1-\beta_t\}$, $1-\beta_t\in(-1,1)$ \\
AUSSM\newline \citep{Karuvally2025}
  & $x_t=e^{\Delta_tA_t}x_{t-1}+\Delta_tBu_t$
  & $e^{\Delta_tA_t}$, $A_t$ skew-symmetric, input-dependent
  & $|\lambda|=1$ \\
\rowcolor{black!6}
LinOSS-IMEX\newline \citep{Rusch2024}
  & $x_n=Mx_{n-1}+F_n$
  & $M=\bigl(\begin{smallmatrix}I&-\Delta tA\\\Delta tI&I-\Delta t^2A\end{smallmatrix}\bigr)$, $A\ge0$ diagonal
  & $|\lambda|=1$ if $\Delta t^2A_{kk}\le4$ \\
LinOSS-IM\textsuperscript{g}\newline \citep{Rusch2024}
  & $x_n=M^{-1}x_{n-1}+M^{-1}F_n$
  & $M^{-1}=\bigl(\begin{smallmatrix}S&-\Delta tAS\\\Delta tS&S\end{smallmatrix}\bigr)$, $S=(I+\Delta t^2A)^{-1}$, $A=\mathrm{ReLU}(\hat A)$ diagonal
  & $|\lambda|^2=S_{kk}\in(0,1]$ \\
\rowcolor{black!6}
S4\textsuperscript{f}\newline \citep{Gu2021}
  & $x_k=\bar Ax_{k-1}+\bar Bu_k$
  & $\bar A=(I-\tfrac\Delta2A)^{-1}(I+\tfrac\Delta2A)$, HiPPO-initialized $A$
  & $|\bar\lambda|<1$ iff $\mathrm{Re}\,\lambda<0$ \\
S5\textsuperscript{f}\newline \citep{Smith2022}
  & $x_k=\bar\Lambda x_{k-1}+\bar Bu_k$
  & $\bar\Lambda=e^{\Lambda\Delta}$, $\Lambda$ complex diagonal
  & $|\bar\lambda|=e^{\Delta\,\mathrm{Re}\lambda}<1$ iff $\mathrm{Re}\,\lambda<0$ \\
\rowcolor{black!6}
Longhorn\textsuperscript{h}\newline \citep{Liu2024}
  & $S_t=A_t\odot S_{t-1}+B_t$
  & $A_t=1-\varepsilon_t\otimes k_t^{\odot2}$, $\varepsilon_t=\beta_t/(1+\beta_tk_t^\Tr k_t)$, $k_t=W_kx_t$ not normalized
  & $\lambda\in(0,1]$ \\
Kimi Delta Attention\newline \citep{Team2025}
  & $S_t=(I-\beta_tk_tk_t^\Tr)\mathrm{Diag}(\alpha_t)S_{t-1}+\beta_tk_tv_t^\Tr$
  & $(I-\beta_tk_tk_t^\Tr)\mathrm{Diag}(\alpha_t)$, $\alpha_{t,j},\beta_t\in[0,1]$
  & $\rho\le\max_j\alpha_{t,j}$ \\
\rowcolor{black!6}
RWKV-4\newline \citep{Peng2023}
  & $a_t=e^{-w}\odot a_{t-1}+e^{k_t}\odot v_t$
  & $\mathrm{diag}(e^{-w})$, $w\in\R^d_{\ge0}$
  & $\lambda\in(0,1]$ \\
mLSTM, exp.\ gate\textsuperscript{i}\newline \citep{Beck2024}
  & $C_t=f_tC_{t-1}+i_tv_tk_t^\Tr$
  & $f_tI$, $f_t=\exp(\tilde f_t)$
  & $\lambda=f_t\in(0,\infty)$ \\
\addlinespace[4pt]
\midrule
\multicolumn{4}{@{}l@{}}{\scriptsize \textbf{$\rho<1$: loses separation, fails \Rthree; realizes the definite
automata, $L_{\max}=O(\log\tfrac1\eta)$}\textsuperscript{j}} \\
\rowcolor{black!6}
S4D\newline \citep{Gu2022}
  & $x_k=\bar Ax_{k-1}+\bar Bu_k$
  & $\exp(\Delta A)$, $A=-\exp(A_{\mathrm{Re}})+\mathrm{i}A_{\mathrm{Im}}$
  & $|\bar\lambda|=e^{-\Delta e^{A_{\mathrm{Re}}}}<1$ \\
DSS (DSS-exp)\newline \citep{Gupta2022}
  & $x_k=\bar Ax_{k-1}+\bar Bu_k$
  & $\exp(\Delta\Lambda)$, $\Lambda=-\exp(\Lambda_{\mathrm{re}})+\mathrm{i}\Lambda_{\mathrm{im}}$
  & $|\bar\lambda|=e^{-\Delta e^{\Lambda_{\mathrm{re}}}}<1$ \\
\rowcolor{black!6}
Mamba\newline \citep{Gu2024}
  & $h_t=\bar A_th_{t-1}+\bar B_tx_t$
  & $\exp(\Delta_tA)$, $A$ real diagonal, negative
  & $\lambda\in(0,1)$ \\
Mamba$^{-}$\newline \citep{Grazzi2025}
  & $h_t=\mathrm{Diag}(2s_t-1)h_{t-1}+\bar B_tx_t$
  & $\mathrm{Diag}(2s_t-1)$, $s_t=\exp(-\Delta_t\odot e^{w})$
  & $\lambda\in(-1,1)$ \\
\rowcolor{black!6}
Mamba-2 / SSD\newline \citep{Dao2024}
  & $h_t=a_th_{t-1}+b_t$
  & $a_tI$, $a_t$ scalar per head
  & $\lambda=a_t\in(0,1)$ \\
Mamba-3\newline \citep{Lahoti2025}
  & $h_t=e^{\Delta_tA_t}R_th_{t-1}+\Delta_tB_tx_t$
  & $e^{\Delta_tA_t}R_t$, $A_t<0$ scalar, $R_t$ block-diagonal $2\times2$ rotations
  & $|\lambda|=e^{\Delta_tA_t}\in(0,1)$ \\
\rowcolor{black!6}
LRU\newline \citep{Orvieto2023}
  & $x_k=\Lambda x_{k-1}+\gamma\odot Bu_k$
  & $\Lambda=\mathrm{diag}(\lambda_j)$, $\lambda_j=e^{-e^{\nu_j}+\mathrm{i}e^{\theta_j}}$
  & $|\lambda_j|=e^{-e^{\nu_j}}<1$ \\
gated linear attention\newline \citep{Yang2024a}
  & $S_t=G_t\odot S_{t-1}+k_t^\Tr v_t$
  & $\mathrm{diag}(\alpha_t)$, $\alpha_t=\sigma(\cdot)^{1/\tau}$
  & $\lambda\in(0,1)$ \\
\rowcolor{black!6}
gated DeltaNet\textsuperscript{c}\newline \citep{Yang2024b}
  & $S_t=S_{t-1}\alpha_t(I-\beta_tk_tk_t^\Tr)+\beta_tv_tk_t^\Tr$
  & $\alpha_t(I-\beta_tk_tk_t^\Tr)$, $\alpha_t,\beta_t\in(0,1)$
  & $\alpha_t\{1^{(d-1)},1-\beta_t\}$, $\rho=\alpha_t$ \\
RetNet\newline \citep{Sun2023}
  & $s_n=As_{n-1}+K_n^\Tr v_n$
  & $A=\Lambda(\gamma e^{\mathrm{i}\theta})\Lambda^{-1}$, $\gamma=1-2^{-5-\mathrm{arange}(0,h)}$
  & $|\lambda|=\gamma<1$ \\
\rowcolor{black!6}
RWKV-7\newline \citep{Peng2025}
  & $S_t=S_{t-1}G_t+v_t^\Tr k_t$
  & $G_t=\mathrm{diag}(w_t)-\hat\kappa_t^\Tr(a_t\odot\hat\kappa_t)$, $w_t\in(e^{-e^{-1/2}},1)$, $a_t\in(0,1)$
  & $\lambda\in(-1,1)$ (their Theorem~1) \\
HGRN / HGRN2\newline \citep{Qin2023,Qin2024}
  & $h_t=\lambda_t\odot e^{\mathrm{i}\theta}\odot h_{t-1}+(1-\lambda_t)\odot c_t$
  & $\mathrm{diag}(\lambda_te^{\mathrm{i}\theta})$, $\lambda_t=\gamma_k+(1-\gamma_k)\mu_t$
  & $|\lambda|\in(\gamma_k,1)$ \\
\rowcolor{black!6}
Hawk / Griffin RG-LRU\newline \citep{De2024}
  & $h_t=a_t\odot h_{t-1}+\sqrt{1-a_t^2}\odot(i_t\odot x_t)$
  & $\mathrm{diag}(a_t)$, $a_t=a^{cr_t}$, $a=\sigma(\Lambda)$
  & $\lambda\in(0,1)$ \\
QRNN\newline \citep{Bradbury2017}
  & $h_t=f_t\odot h_{t-1}+(1-f_t)\odot z_t$
  & $\mathrm{diag}(f_t)$, $f_t=\sigma(\cdot)$
  & $\lambda\in(0,1)$ \\
\rowcolor{black!6}
GILR\textsuperscript{d}\newline \citep{Martin2018}
  & $h_t=g_t\odot h_{t-1}+(1-g_t)\odot i_t$
  & $\mathrm{diag}(g_t)$, $g_t=\sigma(Ux_t+b_g)$
  & $\lambda\in(0,1)$ \\
minGRU / minLSTM\textsuperscript{d}\newline \citep{Feng2024}
  & $h_t=(1-z_t)\odot h_{t-1}+z_t\odot\tilde h_t$
  & $\mathrm{diag}(1-z_t)$ resp.\ $\mathrm{diag}(f'_t)$, gates read $x_t$ only
  & $\lambda\in(0,1)$ \\
\rowcolor{black!6}
MLGRU\textsuperscript{d}\newline \citep{Zhu2025}
  & $h_t=f_t\odot h_{t-1}+(1-f_t)\odot c_t$
  & $\mathrm{diag}(f_t)$, $f_t=\sigma(\cdot)$, ternary weights
  & $\lambda\in(0,1)$ \\
PD-SSM\textsuperscript{e}\newline \citep{Terzic2025}
  & $x_t=P(u_t)D(u_t)x_{t-1}+Bu_t$
  & $P$ column one-hot, $D$ complex diagonal, $|D|=\sigma(\cdot)$
  & $\rho\le\max_j|d_j|<1$ \\
\end{longtable}

{\raggedright\scriptsize \textsuperscript{a} No surveyed architecture targets this branch; unconstrained ones (notes~f and~i) can enter it. \quad \textsuperscript{b} Upper bound of Proposition~\ref{prop:neutralbits}, proven for transitions that are isometries (unitary, orthogonal, identity). For the other rows of the block, only the failure of \Rtwo is proven; a nonnormal transition with $\rho=1$ can amplify perturbations polynomially and shorten the horizon further, e.g.\ to $O(\eta^{-1/2})$ for a $2\times2$ Jordan block. $\{\rho=1\}$ has empty interior. \quad \textsuperscript{c} Moves to the $\rho=1$ branch if $\alpha_t$ attains $1$. \quad \textsuperscript{d} Gates read the input only, which makes them scannable. \quad \textsuperscript{e} $P$ alone has $\rho=1$; the factor $D$ makes the product contract. \quad \textsuperscript{f} The paper does not constrain $\mathrm{Re}\,\lambda$: the initialization is stable, but training can reach $\rho\ge1$. \quad \textsuperscript{g} $|\lambda|=1$ where $A_{kk}=0$, which $\mathrm{ReLU}$ allows. \quad \textsuperscript{h} $\lambda=1$ where $k_{t,j}=0$; the exact update $I-\varepsilon_tk_tk_t^\Tr$ it approximates is on the $\rho=1$ branch. \quad \textsuperscript{i} The exponential forget gate is unconstrained, like S4 and S5 (note~f): it attains $1$ and can exceed it, which enters the $\rho>1$ branch. With the sigmoid forget gate, $f_t\in(0,1)$ and mLSTM belongs to the $\rho<1$ branch. \quad \textsuperscript{j} Upper bound of Corollary~\ref{cor:precision}, under the stronger hypothesis $\max_x\|A_x\|_\op<1$; the definite automata are attained by Proposition~\ref{prop:definiteattained}.
\par\normalsize}
\end{center}

\subsection{Memory Recurrent Units as transition tables}
\label{app:mru}
Memory Recurrent Units \citep{DeGeeter2026} start from a multistable recurrence and assume that it converges between two time steps, so that the state only ever sits on an attractor. Read on those attractors, each input symbol acts as a transition table, and the parallel scan composes tables exactly.

\textbf{The BMRU.} With state dimension $d$ and all operations coordinatewise, the Bistable Memory Recurrent Unit (BMRU) of \citet{DeGeeter2026} computes from the input
\begin{equation}
  \hat h_t=W_xx_t+b_x,\qquad \beta_t=|W_\beta x_t+b_\beta|,\qquad z_t=\mathrm{H}\bigl(|\hat h_t|-\beta_t\bigr),
\end{equation}
with $\mathrm{H}$ the Heaviside step, and updates
\begin{equation}
  h_t=z_t\odot\mathrm{sign}(\hat h_t)\odot\alpha+(1-z_t)\odot h_{t-1},
  \label{eq:bmru}
\end{equation}
with a learnable $\alpha\in\R^d_{>0}$.

\textbf{As a transition table.} On the attractors $\{-\alpha_i,+\alpha_i\}$ of coordinate $i$, a symbol acts in one of three ways: it sets the coordinate to $+\alpha_i$ or to $-\alpha_i$ when $z_i=1$, and keeps it when $z_i=0$. These are the three elements of the flip-flop monoid FF of Section~\ref{sec:experiments}, so the transition monoid of a BMRU layer on its attractors is a submonoid of $\mathrm{FF}^d$, one flip-flop per coordinate. Each table costs $2$ bits per coordinate, and composition is exact: a later set overrides, a later keep returns the earlier element. The BMRU is therefore scan-representable in the sense of Definition~\ref{def:scanrep} with $b=2d$ and $c=O(d)$, and the scanned trajectory is the sequential one.

The update equation is affine in the state, with $A_x=\mathrm{diag}(1-z)$, and a keep symbol gives $A_x=I$: read as an update on $\R^d$, it lies on the $\rho=1$ branch of Theorem~\ref{thm:affine} and does not restore. Its length independence comes from the state being held on the attractors, that is, stored as an index, which is the restoration the approximation assumed and which the implementation must then enforce. The tables act coordinate by coordinate and never read the state: no symbol exchanges $+\alpha_i$ and $-\alpha_i$, so a single layer realizes no nontrivial group, not even $\mathbb{Z}_2$.

\textbf{The CMRU.} \citet{Brandoit2026} add a term $\omega\,h_{t-1}$, with $\omega\in[0,1]$, to the set branch of \eqref{eq:bmru}. For $0<\omega<1$ the reachable states form an infinite set, and at $\omega=1$ coordinate $i$ moves on the unbounded lattice $\alpha_i\mathbb{Z}$, so \Ttwo fails; confining each coordinate to $\{-k\alpha_i,\dots,+k\alpha_i\}$ with a barrier at each end makes every symbol a table on these $n=2k+1$ states again, scan-representable with $b=d\,n\lceil\log_2n\rceil$ and $c=O(dn)$.

\section{From finite precision to the perturbation model}
\label{app:finite_precision}

Definition~\ref{def:noise} replaces finite-precision arithmetic by an arbitrary perturbation of size at most $\eta$ at every step. This appendix explains why this is a correct model: $\eta$ is set by the number format, every run of a real machine is an $\eta$-trajectory, and asking for correctness under every such perturbation is what makes a guarantee robust to the hardware and to the implementation.

\textbf{Where $\eta$ comes from.} A machine stores states, parameters, and intermediate results on a fixed number of bits. It can therefore represent only a finite set $\mathcal{Q}\subset\mathcal{H}$ of points; write $[\,\cdot\,]_\mathcal{Q}$ for the rounding that replaces an object by a representable one, e.g.\ the nearest point of $\mathcal{Q}$ under round-to-nearest. Since a finite set cannot contain every exact result, rounding moves points in general. It need not move the points that a particular model visits: an affine recurrence whose coefficients and codes are $0$ and $1$, such as a permutation matrix acting on one-hot vectors, computes every product and sum exactly, in any order, and its rounding error is zero. Such exact structures are the exception for learned models. Gates produced by sigmoids and exponentials of trained weights are generically not exactly representable, the moduli of PD-SSM lie in $(0,1)$ and those of Mamba$^{-}$ in $(-1,1)$, and the rotation of Appendix~\ref{app:learning} is not representable in any format. We model this generic case by a per-step bound $\eta>0$.

\textbf{What a step injects.} The machine does not run \eqref{eq:rnn} but its rounded version
\begin{equation}
  h_t = \tilde\Phi_{x_t}(h_{t-1})
   \coloneq \Bigl[\,[\Phi]_\mathcal{Q}\bigl(h_{t-1},[x_t]_\mathcal{Q}\bigr)\Bigr]_\mathcal{Q},
  \label{eq:rounded}
\end{equation}
where $[x_t]_\mathcal{Q}$ is the rounded input embedding of the symbol $x_t$, $[\Phi]_\mathcal{Q}$ is the update with rounded parameters, evaluated in finite precision, and the outer bracket rounds the result. The per-step perturbation
\[
  e_t \coloneq \tilde\Phi_{x_t}(h_{t-1})-\Phi_{x_t}(h_{t-1})
\]
is the error made by this step alone. It has three sources: the rounding of the parameters and of the input, the rounding of the intermediate operations, and the rounding of the output. It does not include the error already carried by $h_{t-1}$, since both maps are evaluated at the same $h_{t-1}$. Each step therefore adds a fresh, bounded error. Whatever the order of the operations, the machine stores $h_t\in\mathcal{Q}$, and we write $e_t \coloneq h_t-\Phi_{x_t}(h_{t-1})$ for the error of step $t$.

\begin{lemma}[Every machine run is an $\eta$-trajectory]
\label{lem:machine}
Assume that $\Sig$ is finite, that $\enc(Q)\subseteq\mathcal{Q}$, and that no overflow occurs. Then $\eta_{\mathcal{Q}} \coloneq \max\|e_t\|$, over all steps, words, and initial states, exists and does not depend on $t$, on $|w|$, or on the input, and every run of the machine is an $\eta$-trajectory for every $\eta\ge\eta_{\mathcal{Q}}$. Hence, for every $\eta>0$ with $\eta\ge\eta_{\mathcal{Q}}$, a guarantee over all $\eta$-trajectories, \Tone in particular, holds for every run.
\end{lemma}
\begin{proof}
Since $h_{t-1},h_t\in\mathcal{Q}$ and $x_t\in\Sig$, the error $e_t$ lies in the finite set $\{q-\Phi_x(q') \mid q,q'\in\mathcal{Q},\,x\in\Sig\}$, so the maximum exists. A sequence whose steps deviate by at most $\eta_{\mathcal{Q}}$ deviates by at most any $\eta\ge\eta_{\mathcal{Q}}$.
\end{proof}

The bound $\eta_{\mathcal{Q}}$ can be zero, as in the exact example above; the results of this paper then apply at every $\eta>0$, and in particular a guarantee still holds for the actual runs. In floating point, rounding errors are relative, so $\eta_{\mathcal{Q}}$ scales with the size of the states visited.

How small $\eta$ is depends on the format and on the evaluation: for the sequential evaluation \eqref{eq:rounded}, $e_t$ is the rounding error of a single step, and a different order of the operations, such as a parallel scan, changes $e_t$ but not the conclusion.

\textbf{Why every perturbation, and not the actual rounding.} The rounding errors of a run are not a property of the model. They depend on the hardware, through the number format, and on the implementation, through the order of the operations: floating-point addition and multiplication are not associative, so two runs of the same model on the same input, one sequential and one with a parallel scan, are not bit-identical. A particular rounding pattern can even help: rounding turns the recurrence into a map on the finite set $\mathcal{Q}$, which can hold a few states at the scale of the rounding error, as the dead bands and limit cycles of recursive digital filters do, in fixed point and in floating point \citep{Parker1971,Kaneko1973}. Such behavior is not robust, since it can disappear when the number format or the order of the operations changes. We therefore ask for correctness under every perturbation of size at most $\eta$, which is what quantifying over all $\eta$-trajectories does (Definition~\ref{def:realization}). This makes the guarantee a worst-case one, and the impossibility results should be read accordingly: Theorem~\ref{thm:affine} states that no affine recurrence is \emph{guaranteed} correct at every length under perturbations of size $\eta$, not that every run of every affine model fails. A model whose rounding happens to be harmless, or zero, can still be correct on its actual runs, but that correctness is not robust: it can be lost when the number format, the order of the operations or the parameters change (Appendix~\ref{app:learning}).

\textbf{Sequential versus parallel evaluation.} The order of the operations alone is enough to change the trajectory. We evaluate each affine baseline of Section~\ref{sec:tworates} twice on the same word: once with the sequential recurrence $h_t=A_{x_t}h_{t-1}+b_{x_t}$, which rounds the products and sums from left to right, and once with the parallel scan used for training, which combines the pairs $(A_{x_t},b_{x_t})$ along a tree of depth $\log_2|w|$ and therefore rounds the same operations in a different order. The two runs share the parameters, the input and the coefficients mapping $(A_{x_t},b_{x_t})$, so they differ only in this order. By Lemma~\ref{lem:machine}, both are $\eta$-trajectories of the same word from the same $h_0$, and the gap $\|h_t^{\mathrm{par}}-h_t^{\mathrm{seq}}\|_\infty$ shows what the recurrence does with the errors injected along the way (Figure~\ref{fig:seqpar}).

\begin{figure}[t]
  \centering
  \includegraphics[width=\linewidth]{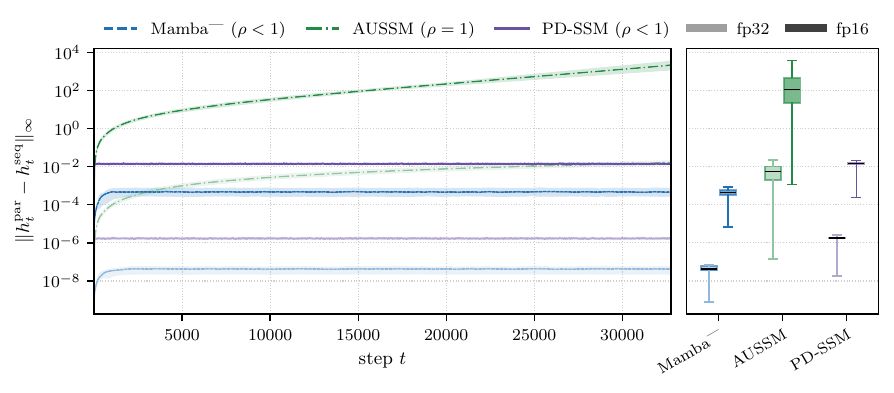}
  \caption{\textbf{Sequential and parallel evaluations of the same model may follow different trajectories.} Gap $\|h_t^{\mathrm{par}}-h_t^{\mathrm{seq}}\|_\infty$ between the parallel scan and the sequential recurrence, with the same coefficients $(A_{x_t},b_{x_t})$, computed in fp32 and rounded to the format (light: fp32; dark: fp16; complex states stored as two real arrays). Untrained models of width $32$, $64$ words of length $2^{15}$ over five letters, $5$ seeds. \textbf{Left:} gap at each step (mean over words and seeds; band: range of the per-seed means). \textbf{Right:} distribution over all steps and seeds (whiskers: minimum and maximum). Absolute values are not comparable across models; the trend in $t$ is.}
  \label{fig:seqpar}
\end{figure}

Three observations follow. First, the gap is nonzero from the first steps, in every model and format. Second, its evolution follows the spectral radius, as Theorem~\ref{thm:affine} predicts. With $\rho<1$ (Mamba$^{-}$, PD-SSM), earlier errors are forgotten as fast as new ones arrive, and the gap saturates. With $\rho=1$ (AUSSM), every error is carried forward, and the gap grows until, in fp16, it reaches the size of the state itself. Third, the number format sets the level of the curves and the recurrence sets their shape: a coarser format shifts every curve up. A guarantee that relied on the rounding pattern of one evaluation would not transfer to the other; a guarantee over all $\eta$-trajectories covers both.

\section{Learning versus finite precision}
\label{app:learning}

In practice, the main obstacle to length generalization is learning: a model must find, by gradient descent on short sequences, parameters that implement the target transitions. Our results do not address this obstacle directly. They address a different one, which remains even when learning succeeds.

\textbf{Three requirements.} Tracking at every length needs a target that is \emph{expressible} in exact arithmetic (Appendix~\ref{app:related}), a realization that is \emph{robust} to a perturbation at every step (Theorems~\ref{thm:equiv} to~\ref{thm:affine}), and parameters that are \emph{found} by training. 

\textbf{A perfectly trained model still drifts.} Figure~\ref{fig:drift} isolates the robustness failure by removing learning altogether. The target is $\mathbb{Z}_7$ under the constant word, and the recurrence is the $\rho=1$ affine realization that a perfectly trained model would hold: the rotation of $\R^2$ by $2\pi/7$, whose seven code points are the vertices of a regular heptagon. We round each entry of this rotation once to each number format, and then run the recurrence in fp64, so that no rounding occurs after the weights are set. Rounding the cosine and the sine separately leaves a rotation by a slightly wrong angle, scaled by a radius slightly different from one. The phase error therefore grows linearly with $t$, and the model decodes the wrong state as soon as it exceeds $\pi/7$, half the angle between two code points. The radius grows or decays geometrically, so the code points themselves leave the bounded region that restoration would require.

\begin{figure}[ht]
  \centering
  \includegraphics[width=\linewidth]{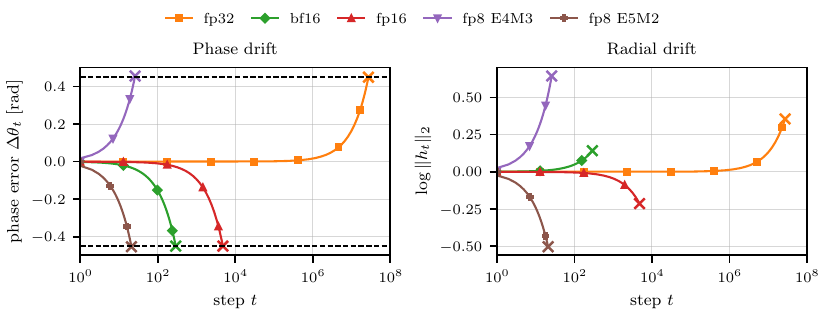}
  \caption{\textbf{A perfectly trained $\rho=1$ recurrence still drifts at finite precision.} The exact rotation by $2\pi/7$, which realizes $\mathbb{Z}_7$ under the constant word, is rounded once to each format and then iterated in fp64 from $h_0=(1,0)$. \textbf{Left:} phase error $\Delta\theta_t$ with respect to the exact run; dashed lines: $\pm\pi/7$, beyond which the readout returns a neighbouring state. \textbf{Right:} log-norm of the hidden state. Crosses mark the first step at which the readout is wrong; every format fails at a finite length, fp32 included.}
  \label{fig:drift}
\end{figure}

First, the drift involves no training error: the weights are the exact solution, up to the resolution of the format. Second, it involves no rounding of the hidden state: the run itself is in fp64, and rounding at every step, or reassociating the products under a scan, only adds to it (Figure~\ref{fig:seqpar}). Third, a finer format moves the failure but does not remove it. A better choice of representable weights can move the failing length, but no representable rotation by $2\pi/7$ exists in any format.

\section{Proofs of theorems and lemmas}
\label{app:proofs}

This appendix collects the proofs deferred from the main text, together with the additional definitions, lemmas, and propositions.

\subsection{Length-independent realization}
\label{app:proofs-sub}

Condition \Ttwo links boundedness to a length-independent precision cost, because a machine stores its state on a fixed finite set of points (Appendix~\ref{app:finite_precision}).

\begin{proposition}[Boundedness is length-independent precision]
\label{prop:precision}
A single finite set of points represents every state of $U^\star$ to within $\eta$ if and only if $U^\star$ is bounded. In that case, $\lceil d\log_2(1+2\diam(U^\star)/\eta)\rceil$ bits suffice, at every step and every length.
\end{proposition}
\begin{proof}
If a finite set $P$ represents $U^\star$ to within $\eta$, then $U^\star\subseteq P\oplus\Ball$, which is bounded. Conversely, let $U^\star$ be bounded, and let $P\subseteq U^\star$ be a maximal set of points pairwise more than $\eta$ apart. By maximality, every point of $U^\star$ lies within $\eta$ of $P$. The balls of radius $\eta/2$ centered at the points of $P$ are disjoint and lie in a ball of radius $\diam(U^\star)+\eta/2$, so comparing volumes gives $|P|\le(1+2\diam(U^\star)/\eta)^d$ in any norm, and indexing $P$ takes $\lceil\log_2|P|\rceil$ bits. Neither $U^\star$ nor $\eta$ depends on $t$ or on $|w|$, so the same set serves at every step and every length.
\end{proof}

For Theorem~\ref{thm:equiv}, necessity is proved with $U=U^\star$, by applying \Tone and \Ttwo to a word and to its extension by one symbol (Lemmas~\ref{lem:code} to~\ref{lem:transition}); sufficiency is an induction on the length of the word.

\begin{lemma}[The encoding must be a code]
\label{lem:code}
Under \Tone, $\enc(Q)\subseteq U^\star$ and $\dec\circ\enc=\mathrm{id}_Q$; in particular $\enc$ is injective.
\end{lemma}

\begin{proof}
Take $w=\varepsilon$. The one-point sequence $h_0=\enc(q)$ is an $\eta$-trajectory of $\varepsilon$ from $\enc(q)$, so $\enc(q)\in U^\star$ and \Tone gives $\dec(\enc(q))=\delta_\varepsilon(q)=q$.
\end{proof}

\begin{lemma}[Restoration on the reachable set]
\label{lem:restorenec}
\emph{(i)} For every $\Phi$ and every encoding, $\Phi_x(U^\star)\oplus\Ball\subseteq U^\star$ for every $x\in\Sig$. \emph{(ii)} Hence, under \Ttwo, $U^\star$ satisfies \Rtwo.
\end{lemma}
\begin{proof}
(i) Let $u\in U^\star$, $x\in\Sig$ and $\|e\|\le\eta$. Choose $q$, $w$ and an $\eta$-trajectory $(h_t)_{t=0}^{|w|}$ of $w$ from $\enc(q)$ with $h_{|w|}=u$, and append $h_{|w|+1} \coloneq \Phi_x(u)+e$. Since $\|h_{|w|+1}-\Phi_x(h_{|w|})\|=\|e\|\le\eta$, the extended sequence is an $\eta$-trajectory of $wx$ from $\enc(q)$, so its endpoint lies in $U^\star$. As $u$, $x$ and $e$ were arbitrary, $\Phi_x(U^\star)\oplus\Ball\subseteq U^\star$ for every $x$. The inclusion holds for any recurrence. (ii) \Rtwo asks for a bounded region that absorbs one perturbation step: (i) gives the inclusion, and \Ttwo gives the boundedness.
\end{proof}

\begin{lemma}[The step must implement the transition]
\label{lem:transition}
Under \Tone, $\dec(\Phi_x(u)+e)=\delta_x(\dec(u))$ for every $x\in\Sig$, every $u\in U^\star$ and every $\|e\|\le\eta$.
\end{lemma}

\begin{proof}
Let $u\in U^\star$, $x\in\Sig$, $\|e\|\le\eta$, and take $q$, $w$ and an $\eta$-trajectory of $w$ from $\enc(q)$ with endpoint $u$, as in Lemma~\ref{lem:restorenec}. \Tone applied to $w$ gives $\dec(u)=\delta_w(q)$; the extended trajectory of $wx$ has endpoint $\Phi_x(u)+e$, so \Tone applied to $wx$ gives $\dec(\Phi_x(u)+e)=\delta_{wx}(q)=\delta_x(\delta_w(q))=\delta_x(\dec(u))$, using $\delta_{wx}(p)=\delta_x(\delta_w(p))$ for every $p$, which is immediate from the definition of $\delta_w$.
\end{proof}

\thmequiv*
\begin{proof}[Proof of Theorem~\ref{thm:equiv}]
\emph{Sufficiency.} We prove the stronger statement that correctness holds from \emph{every} $u\in U$ and not only from code points; \Tone is the case $u=\enc(q)$. Induction on $k \coloneq |w|$. For $k=0$ the trajectory is the single point $h_0=u\in U$ and $\dec(u)=\delta_\varepsilon(\dec(u))$. For the step, let $w\in\Sig^k$ and $x\in\Sig$. Any $\eta$-trajectory $(h_t)_{t=0}^{k+1}$ of $wx$ from $u$ truncates to an $\eta$-trajectory of $w$ from $u$, so the hypothesis gives $h_k\in U$ and $\dec(h_k)=\delta_w(\dec(u))$. Write $h_{k+1}=\Phi_x(h_k)+e$ with $\|e\|\le\eta$. Then $h_{k+1}\in\Phi_x(U)\oplus\Ball\subseteq U$ by \Rtwo, so $\dec(h_{k+1})$ is defined, and $\dec(h_{k+1})=\delta_x(\dec(h_k))$ by \Rthree applied at $h_k\in U$. Since $\dec(h_k)=\delta_w(\dec(u))$ and $\delta_{wx}(p)=\delta_x(\delta_w(p))$ for every $p$, $\dec(h_{k+1})=\delta_{wx}(\dec(u))$. Taking $u=\enc(q)$ and using \Rone gives \Tone; and the same induction keeps every $\eta$-trajectory from $\enc(Q)$ inside $U$, so $U^\star\subseteq U$, which is bounded by \Rtwo and gives \Ttwo. \emph{Necessity, and minimality.} Lemmas~\ref{lem:code}, \ref{lem:restorenec} and~\ref{lem:transition} give \Rone, \Rtwo and \Rthree on $U^\star$, which is bounded by \Ttwo, so $U=U^\star$ carries the triple. Minimality is the containment just established: any $U$ carrying a triple satisfies $U^\star\subseteq U$, since every $\eta$-trajectory from $\enc(Q)$ stays in $U$.
\end{proof}


\thmcells*
\begin{proof}[Proof of Theorem~\ref{thm:cells}]
($\Rightarrow$) Suppose $(\enc,\Phi,\dec)$ realizes $\Mtar$ with length independence at $\eta$, and take the region $U$ of Theorem~\ref{thm:equiv}, bounded by \Rtwo. Put $U_q \coloneq U\cap\dec^{-1}(q)$; since $\dec$ is a function, every $u\in U$ lies in exactly one $U_q$, namely $q=\dec(u)$, so $\{U_q\}_{q\in Q}$ partitions $U$, and $\enc(q)\in U_q$ by \Rone together with $U\supseteq\enc(Q)$. Let $u\in U_q$, $x\in\Sig$ and $\|e\|\le\eta$. Since $u\in U$, \Rtwo gives $\Phi_x(u)+e\in\Phi_x(U)\oplus\Ball\subseteq U$; and $\dec(u)=q$, so \Rthree gives $\dec(\Phi_x(u)+e)=\delta_x(q)$, i.e., $\Phi_x(u)+e\in U\cap\dec^{-1}(\delta_x(q))=U_{\delta_x(q)}$. As $u$ and $e$ were arbitrary, this is \eqref{eq:cellincl}. ($\Leftarrow$) Given such a $U$ and partition, \Rone is immediate: $\enc(q)\in U_q=U\cap\dec^{-1}(q)$ says exactly $\dec(\enc(q))=q$. \Rtwo follows by taking the union over $q$ of $\Phi_x(U_q)\oplus\Ball\subseteq U_{\delta_x(q)}\subseteq U$, with $U$ bounded by hypothesis. For \Rthree, let $u\in U$, $x\in\Sig$, $\|e\|\le\eta$, and put $q \coloneq \dec(u)$, so $u\in U_q$; the same inclusion gives $\Phi_x(u)+e\in U_{\delta_x(q)}=U\cap\dec^{-1}(\delta_x(q))$, so $\dec(\Phi_x(u)+e)=\delta_x(q)= \delta_x(\dec(u))$. Theorem~\ref{thm:equiv} then gives realization with length independence.
\end{proof}

\thmtworates*
\begin{proof}[Proof of Theorem~\ref{thm:tworates}(i)]
We first lift \eqref{eq:cellincl} from letters to words: for $w=w'x$, induction on $|w|$ gives $\Phi_w(U_q)\oplus\Ball=\Phi_x\big(\Phi_{w'}(U_q)\big)\oplus\Ball\subseteq\Phi_x\big(U_{\delta_{w'}(q)}\big)\oplus\Ball\subseteq U_{\delta_w(q)}$, since $\Phi_{w'}(U_q)\subseteq\Phi_{w'}(U_q)\oplus\Ball$. Any nonempty $S$ has $\diam(S\oplus\Ball)=\diam(S)+2\eta$, since $\|(s+e)-(s'+e')\|\le\|s-s'\|+2\eta$ with equality approached by taking $s,s'$ nearly realizing $\diam(S)$ and $e,-e'$ along $s-s'$ of norm $\eta$. Applying this to the lifted inclusion and using monotonicity of $\diam$ under inclusion gives $\diam(\Phi_w(U_q))+2\eta\le\diam(U_{\delta_w(q)})$, and taking $\delta_w(q)=q$ gives the loop case. No quantity here depends on $|w|$, so the margin is the same along every word.
\end{proof}

\begin{proof}[Proof of Theorem~\ref{thm:tworates}(ii)]
Let $\delta_w(q)\ne\delta_w(q')$, and suppose $a\in\Phi_w(U_q)$ and $b\in\Phi_w(U_{q'})$ satisfied $\|a-b\|\le2\eta$. Their midpoint $z \coloneq \tfrac12(a+b)$ has $\|z-a\|=\|z-b\|=\tfrac12\|a-b\|\le\eta$, so $z\in\Phi_w(U_q)\oplus\Ball$ and $z\in\Phi_w(U_{q'})\oplus\Ball$. By \eqref{eq:cellincl} these lie in $U_{\delta_w(q)}$ and $U_{\delta_w(q')}$ respectively, which are distinct cells of a partition and therefore disjoint, so no such $z$ exists. Hence $\|a-b\|>2\eta$ for every $a\in\Phi_w(U_q)$ and every $b\in\Phi_w(U_{q'})$, which is the pointwise form used below. Passing to the infimum gives the statement, with a strict inequality whenever the two images are compact. Lemma~\ref{lem:separate} draws the same conclusion at code points under a weaker hypothesis: correct decoding along the given word only, without a cell partition.
\end{proof}

Every impossibility result below contradicts the following separation.

\begin{lemma}[Separation at code points]
\label{lem:separate}
Let $w\in\Sig^+$, and suppose that \Tone holds for $w$: $\dec(h_{|w|})=\delta_w(p)$ for every $p\in Q$ and every $\eta$-trajectory $(h_t)_{t=0}^{|w|}$ of $w$ from $\enc(p)$. If $\delta_w(q)\ne\delta_w(q')$, then $\|\Phi_w(\enc(q))-\Phi_w(\enc(q'))\|>2\eta$. In particular this holds for every $w\in\Sig^+$ under \Rone to \Rthree, by Theorem~\ref{thm:equiv}.
\end{lemma}
\begin{proof}
Write $w=w'x$ and set $u \coloneq \Phi_{w'}(\enc(q))$, $u' \coloneq \Phi_{w'}(\enc(q'))$. Suppose $\|\Phi_x(u)-\Phi_x(u')\|\le2\eta$ and let $z \coloneq \tfrac12\bigl(\Phi_x(u)+\Phi_x(u')\bigr)$ be the midpoint of the two images, at distance at most $\eta$ from each. The run that follows $w'$ exactly from $\enc(q)$ and then steps to $z$ is an $\eta$-trajectory of $w$ from $\enc(q)$, since only its last step is perturbed, by $z-\Phi_x(u)$; likewise from $\enc(q')$, with last perturbation $z-\Phi_x(u')$. \Tone for $w$ then gives $\dec(z)=\delta_w(q)$ and $\dec(z)=\delta_w(q')$, contradicting the hypothesis.
\end{proof}


The next corollary names the shortest failing length $L^\star$ of a triple $(\enc,\Phi,\dec)$, over all words and all perturbations. The failing length $L_{\max}$ of Section~\ref{sec:experiments} is measured instead on sampled words, with the perturbations that the machine happens to produce; neither is the worst case, so $L_{\max}\ge L^\star$ for the trained model, read with its own readout and with the $\eta$ its machine produces (Lemma~\ref{lem:machine}).

\begin{corollary}[Failure modes]
\label{cor:failure}
If $\Phi$ does not realize $\Mtar$ with length independence at $\eta$, then at least one of the following holds. \emph{\Tone fails:} there is a finite $L^\star$ such that every word of length below $L^\star$ decodes correctly while some word of length $L^\star$ carries a mis-decoding $\eta$-trajectory, so that correctness on $\Sig^{\le L}$ certifies nothing about longer words for any $L<L^\star$. \emph{\Ttwo fails:} $U^\star$ is unbounded, and no fixed number of bits represents every reachable state to within $\eta$ at every length.
\end{corollary}

\begin{proof}
Definition~\ref{def:realization} is the conjunction of \Tone and \Ttwo, so its failure is the failure of at least one of them. We treat the two cases in turn; they are not exclusive.

\emph{\Tone fails.} Then some word $w$ carries an $\eta$-trajectory from some $\enc(q)$ with $\dec(h_{|w|})\neq\delta_w(q)$, and $L^\star$ is the least length of such a word.

\emph{\Ttwo fails.} Then $U^\star$ is unbounded, and we show that no fixed number of bits serves every length. Suppose $m$ bits sufficed at every length, that is, $2^m$ representable points, each state being stored as a point within $\eta$ of it. Since $U^\star$ is unbounded, choose reachable states $u_1,\dots,u_{2^m+1}$ pairwise more than $2\eta$ apart, and let $L$ be the greatest length of a word reaching one of them, so that all of them are reachable at length at most $L$. Each $u_i$ lies within $\eta$ of one of the $2^m$ points, and there are more states than points, so two of them are stored as the same point and are therefore at most $2\eta$ apart, contradicting their choice. Hence, for every $m$ there is a length at which $m$ bits no longer suffice: the precision needed grows without bound as $L$ grows. This is the contrapositive of Proposition~\ref{prop:precision}.
\end{proof}


\subsection{Affine recurrence}
\label{app:affineproofs}

One affine step does not expose the conflict of Theorem~\ref{thm:tworates}: a rate below one shrinks the cells and, if the codes start far enough apart, leaves a gap above $2\eta$ between them. What an affine map cannot do is meet both demands along a repeated word, because its rate does not depend on where the trajectories sit (Definition~\ref{def:singlerate}, Proposition~\ref{prop:singlerate}): along $w^k$, the rate that damps the perturbation also shrinks the gap between two codes. Input-dependent gates modulate that rate \emph{through time} and leave it uniform in space.

\begin{definition}[Single-rate recurrence]
\label{def:singlerate}
$\Phi$ is \emph{single-rate} if the displacement between two trajectories depends only on their initial displacement: for every $w\in\Sig^*$ there is a map $T_w$ with $\Phi_w(u)-\Phi_w(v)=T_w(u-v)$ for all $u,v\in\mathcal{H}$, so that one and the same map governs every pair of points, wherever they sit.
\end{definition}

\begin{proposition}[Single-rate recurrences are exactly the affine ones]
\label{prop:singlerate}
A continuous recurrence is single-rate if and only if it is \emph{affine}, and then $T_w=A_w$ is linear. Its local Lipschitz constant is then the operator norm $\|A_x\|_\op$ at \emph{every} point of $\mathcal{H}$, independently of where the pair of points is taken.
\end{proposition}

\begin{proof}
If $\Phi$ is affine then \eqref{eq:affine} gives $\Phi_w(u)-\Phi_w(v)=A_w(u-v)$, so $T_w=A_w$ meets Definition~\ref{def:singlerate}. Conversely, suppose the definition holds and take $w=x$ a single symbol. Setting $v=0$ gives $\Phi_x(u)=T_x(u)+\Phi_x(0)$, and substituting this back gives $T_x(u)-T_x(v)=T_x(u-v)$ for all $u,v$, so $T_x$ is additive with $T_x(0)=0$. An additive map that is continuous is linear, and $T_x$ is continuous because $\Phi_x$ is; hence each $\Phi_x$ is affine with linear part $A_x \coloneq T_x$, and every composite is affine with $A_w=A_{x_k}\cdots A_{x_1}$. Finally, $\|\Phi_x(u)-\Phi_x(v)\|=\|A_x(u-v)\|$ for all $u,v$, so the ratio is bounded by $\|A_x\|_\op$ uniformly and attained on the corresponding singular direction, independently of where the pair is taken: the same constant governs pairs inside a cell and pairs in different cells, and along a word $w$ the same operator $A_w$ acts on a perturbation and on the difference between two code points.
\end{proof}

\begin{lemma}[Contraction requires a stable matrix]
\label{lem:restore}
If $\Phi$ is affine and satisfies \Rtwo on a nonempty $U$ at some $\eta>0$, then for every $w\in\Sig^+$ and unit $v$, $\sum_{j\ge0}\|(A_w^\Tr)^jv\|\le\diam(U)/2\eta<\infty$, and consequently $\rho(A_w)<1$; in particular $\rho(A_x)<1$ for every $x$.
\end{lemma}
\begin{proof}
We take $\|\cdot\|$ Euclidean here, so that an inner product is available; all norms on $\R^d$ are equivalent, and only the constant in the bound depends on that choice, not the conclusion $\rho(A_w)<1$. Write $A \coloneq A_w$ and $\Psi \coloneq \Phi_w$, so $\Psi(u)=Au+b$ for some $b$. Iterating \Rtwo along $w$ gives
\begin{equation}
  \Psi(U)\oplus\Ball\subseteq U,
  \label{eq:blockR2}
\end{equation}
since $\Phi_{x_2}(\Phi_{x_1}(U)\oplus\Ball)\oplus\Ball \subseteq\Phi_{x_2}(U)\oplus\Ball\subseteq U$, and so on along $w$.

Fix $u_0\in U$, fix $k\ge1$, and let $e_0,\dots,e_{k-1}$ satisfy $\|e_t\|\le\eta$. Define $h_0 \coloneq u_0$ and $h_{t+1} \coloneq \Psi(h_t)+e_t$. By \eqref{eq:blockR2} and induction, $h_t\in U$ for all $t\le k$, and since $\Psi$ is affine,
\begin{equation}
  h_k = \Psi^k(u_0)+\sum_{j=0}^{k-1}A^{\,j}e_{k-1-j}.
  \label{eq:closed}
\end{equation}

Let $v$ be a unit vector. Choose $e_{k-1-j} \coloneq \eta\,(A^\Tr)^jv/\|(A^\Tr)^jv\|$ when $(A^\Tr)^jv\neq0$ and $e_{k-1-j} \coloneq 0$ otherwise, and let $h_k^{+}$ and $h_k^{-}$ be the endpoints \eqref{eq:closed} obtained from $(e_t)$ and from $(-e_t)$. Both lie in $U$, and using $\langle A^{\,j}z,v\rangle=\langle z,(A^\Tr)^jv\rangle$,
\[
  \diam(U)\ge\langle h_k^{+}-h_k^{-},v\rangle
   = 2\sum_{j=0}^{k-1}\bigl\langle e_{k-1-j},(A^\Tr)^jv\bigr\rangle
   = 2\eta\sum_{j=0}^{k-1}\bigl\|(A^\Tr)^jv\bigr\|,
\] the terms with $(A^\Tr)^jv=0$ contributing $0$ to both sides. Letting $k\to\infty$ gives $\sum_{j\ge0}\|(A^\Tr)^jv\|\le\diam(U)/2\eta<\infty$.

A convergent series has vanishing terms, so $\|(A^\Tr)^jv\|\to0$ for every $v$, hence $\|(A^\Tr)^j\|\to0$ in finite dimension and $\rho(A)=\rho(A^\Tr)<1$. Taking $w=x$ gives $\rho(A_x)<1$ for every letter.
\end{proof}

\begin{definition}[Definite automaton]
\label{def:definite}
A semiautomaton $(Q,\Sig,\delta)$ \citep{Perles1963} is \emph{$k$-definite} if $\delta_w$ is a constant map of $Q$ for every word $w$ with $|w|\ge k$, and \emph{definite} if it is $k$-definite for some $k$; we write $c_a\in Q^Q$ for the constant map with value $a$. Equivalently, the state reached after reading a word of length at least $k$ depends on the last $k$ symbols alone: splitting $w=uv$ with $|v|=k$ gives $\delta_w(q)=\delta_v(\delta_u(q))$, the constant value of $\delta_v$, for every $q$ and every prefix $u$. Definite automata are aperiodic, hence recognize star-free behaviors \citep{Schutzenberger1965,McNaughton1971,Eilenberg1974}; the converse fails, so ``definite'' is strictly the stronger of the two conclusions.
\end{definition}

The next lemma turns a statement made one word at a time, that every repeated word ends up acting as a constant map, into a property of the automaton.

\begin{lemma}[Constant powers force definiteness]
\label{lem:definite}
Let $|Q|=n$ and let $S \coloneq \{\delta_w\mid w\in\Sig^+\}$ be the transition semigroup of $(Q,\Sig,\delta)$, carrying the product inherited from concatenation, $\delta_u\cdot\delta_v \coloneq \delta_{uv}=\delta_v\circ\delta_u$. Suppose that for every $w\in\Sig^+$ the map $\delta_{w^k}$ is constant for all sufficiently large $k$. Then \emph{(i)} every idempotent of $S$ is a constant map. \emph{(ii)} Whenever every idempotent of $S$ is constant, in particular under this hypothesis, $\delta_w$ is constant for every word with $|w|>|S|$, so the automaton is $k$-definite with $k=|S|+1\le n^n+1$. Contrapositively, if the automaton is not definite, some word acts as an idempotent of $S$ that is not constant.
\end{lemma}
\begin{proof}
$S$ is a subsemigroup of $Q^Q$ and hence finite, with $|S|\le n^n$.

\emph{(i)} Let $e\in S$ be idempotent and choose $w\in\Sig^+$ with $e=\delta_w$. Then $\delta_{w^k}=e^k=e$ for every $k\ge1$, and $\delta_{w^k}$ is constant for $k$ large by hypothesis; so $e$ is constant.

\emph{(ii)} Let $w=x_1\cdots x_{|w|}$ with $|w|>|S|$ and write $p_i \coloneq \delta_{x_1\cdots x_i}\in S$, $1\le i\le|w|$, for its prefixes. Since $|w|>|S|$ two of them coincide, say $p_i=p_j$ with $i<j$. Put $t \coloneq \delta_{x_{i+1}\cdots x_j}\in S$, so that $p_i=p_j=p_i\cdot t$ and, by induction, $p_i=p_i\cdot t^{r}$ for every $r\ge1$. Every element of a finite semigroup has an idempotent power, so $t^{r}$ is idempotent for some $r\ge1$, hence $t^{r}=c_a$ for some $a\in Q$, since every idempotent of $S$ is constant. Then
\[
  p_i = p_i\cdot t^{r} = t^{r}\circ p_i = c_a\circ p_i = c_a,
\] so $p_i$ is itself constant with value $a$, and $\delta_w=\delta_{x_{i+1}\cdots x_{|w|}}\circ p_i$ is the constant map with value $\delta_{x_{i+1}\cdots x_{|w|}}(a)$. As $w$ was an arbitrary word of length greater than $|S|$, the automaton is $(|S|+1)$-definite.
\end{proof}

\thmaffine*
\begin{proof}[Proof of Theorem~\ref{thm:affine}]
Let $\Phi$ be affine in the hidden state, $\Phi_x(h)=A_xh+b_x$ and $A_w=A_{x_k}\cdots A_{x_1}$ as in \eqref{eq:affine}, and let $\eta>0$.

\emph{(i) $\rho(A_w)\ge1$ for some $w\in\Sig^+$.} Suppose some bounded $U$ satisfied $\Phi_w(U)\oplus\Ball\subseteq U$. The proof of Lemma~\ref{lem:restore} uses no more than that single inclusion: run it with $\Psi \coloneq \Phi_w$ and $A \coloneq A_w$ to get $\sum_{j\ge0}\|(A_w^\Tr)^jv\|\le\diam(U)/2\eta<\infty$ for every unit $v$, whence $\|(A_w^\Tr)^j\|\to0$ and $\rho(A_w)<1$, contradicting the hypothesis. So no bounded $U$ satisfies the inclusion for $\Phi_w$, and a fortiori none satisfies \Rtwo for every symbol, since iterating \Rtwo along $w$ gives $\Phi_w(U)\oplus\Ball\subseteq U$. By Theorem~\ref{thm:equiv} no triple $(\enc,\Phi,\dec)$ realizes any target with length independence at $\eta$, whatever the state dimension, the encoding and the readout.

\emph{(ii) $\rho(A_w)<1$ for every $w\in\Sig^+$.} Suppose a bounded $U$ carries \Rone to \Rthree at $\eta$, and fix $w\in\Sig^+$ and $q,q'\in Q$. Affineness gives
\[
  \bigl\|\Phi_{w^k}(\enc(q))-\Phi_{w^k}(\enc(q'))\bigr\|
   = \bigl\|A_w^{k}\bigl(\enc(q)-\enc(q')\bigr)\bigr\|
  \xrightarrow[\;k\to\infty\;]{}0,
\] since $\rho(A_w)<1$ forces $A_w^k\to0$. Lemma~\ref{lem:separate} requires that quantity to exceed $2\eta$ whenever $\delta_{w^k}(q)\ne\delta_{w^k}(q')$, so there is a threshold $K_{q,q',w}$ with $\delta_{w^k}(q)=\delta_{w^k}(q')$ for every $k\ge K_{q,q',w}$. Taking the maximum over the finitely many pairs $(q,q')$ makes $\delta_{w^k}$ a constant map for all $k$ large, and $w\in\Sig^+$ was arbitrary. Lemma~\ref{lem:definite} then makes the target $k$-definite with $k=|S|+1\le n^n+1$, where $S$ is its transition semigroup: the state after a word of length at least $k$ is fixed by the last $k$ symbols of that word (Definition~\ref{def:definite}).

Contrapositively, if the target is \emph{not} definite then no bounded $U$ carries \Rone to \Rthree at any $\eta>0$: some pair $q\ne q'$ is kept distinguished by $\delta_{w^k}$ for infinitely many $k$, while the images of their code points merge, and \Rthree fails at every $k$ beyond the point where the separation of Lemma~\ref{lem:separate} is lost. The same operator $A_w$ moves the signal and the perturbation (Proposition~\ref{prop:singlerate}), so no choice of $d$, of $\enc$, of $\dec$ or of $U$ separates the two rates.
\end{proof}

\begin{corollary}[No affine recurrence realizes a nontrivial group with length independence]
\label{cor:affinegroup}
Let $\Mtar$ contain a nontrivial group and let $\Phi$ be affine in the hidden state. Then, for any state dimension $d$, any $\eta>0$, any affine parameters $A_x,b_x$, and any encoding $\enc$ and readout $\dec$, no region $U$ satisfies \Rone to \Rthree.
\end{corollary}

\begin{proof}
This is a special case of Theorem~\ref{thm:affine}: branch~(i) excludes every target, and branch~(ii) excludes this one because a monoid containing a nontrivial group is not definite. Indeed, let $G\le\Mtar$ be a nontrivial group with identity $e$, and pick $g\in G$ with $g\ne e$. The identity map $\mathrm{id}_Q$, if it lies in $G$, is an idempotent of $G$ and therefore equals $e$; hence $g\ne\mathrm{id}_Q$, and $g=\delta_w$ for some $w\in\Sig^+$. If the automaton were $k$-definite, then, since $g^{|G|}=e$, the element $g=g^{m|G|+1}=\delta_{w^{m|G|+1}}$ would be a constant map $c_a$ as soon as $(m|G|+1)|w|\ge k$. A constant map is idempotent, so $e=g^{|G|}=c_a^{|G|}=c_a=g$, contradicting $g\ne e$.
\end{proof}

The next proposition shows that every definite automaton (a shift register that remembers the last $k$~symbols) is realized by a contracting affine recurrence, with a precision depending on the order~$k$ of the target rather than the word length.

\begin{proposition}[Every definite automaton has a contracting affine realization]
\label{prop:definiteattained}
Let $(Q,\Sig,\delta)$ be $k$-definite with $k\ge1$, and let $m\coloneq|\Sig|+|Q|$. There are an affine recurrence $\Phi_x(h)=Ah+b_x$ on $\R^{km}$ with $A^k=0$, hence $\rho(A_w)=0$ for every $w\in\Sig^+$, an encoding $\enc$ and a readout $\dec$ that realize the target with length independence at every $\eta<1/(2k)$ in the norm $\|\cdot\|_\infty$. In any other norm, the same holds for $\eta<1/(2k\mu)$, where $\|\cdot\|_\infty\le \mu\|\cdot\|$.
\end{proposition}
\begin{proof}
Index the basis of $\R^m$ by $\Sig\sqcup Q$, write $\oh_y$ for the basis vector of $y\in\Sig\sqcup Q$, and split $h\in\R^{km}$ into $k$ slots $h=(h^{(1)},\dots,h^{(k)})$ of $\R^m$, slot $1$ being the most recent. Let $A$ shift every slot one position down and drop the last one, $(Ah)^{(1)}=0$ and $(Ah)^{(j+1)}=h^{(j)}$, and let $b_x$ hold $\oh_x$ in slot $1$ and $0$ elsewhere. Then $A^k=0$, so $A_w=A^{|w|}$ is nilpotent and $\rho(A_w)=0$ for every $w\in\Sig^+$: the recurrence is on branch~(ii) of Theorem~\ref{thm:affine}. Encode $\enc(q)\coloneq(\oh_q,0,\dots,0)$.

\emph{Exact runs.} From $\enc(q)$, after a word $w$ with $|w|<k$, slots $1$ to $|w|$ hold $\oh_{x_{|w|}},\dots,\oh_{x_1}$, slot $|w|+1$ holds the marker $\oh_q$, and the other slots are $0$; after a word with $|w|\ge k$, the slots hold the last $k$ symbols of $w$, most recent first. Call these configurations \emph{clean}, and let $C$ be the finite set of clean configurations reached from $\enc(Q)$. A clean configuration $c$ determines a state $D(c)$: if $c$ carries the marker $\oh_q$ and the symbols of $w$, then $D(c)\coloneq\delta_w(q)$; if $c$ carries only the last $k$ symbols $v$, then $D(c)$ is the constant value of $\delta_v$. Then $D(\Phi_x(c))=\delta_x(D(c))$ for every $c\in C$ and $x\in\Sig$. If $c$ carries $\oh_q$ and $w$ with $|wx|<k$, this is $\delta_{wx}(q)=\delta_x(\delta_w(q))$. If $|wx|=k$, then $\Phi_x(c)$ carries only $wx$, and the constant value of $\delta_{wx}$ is $\delta_{wx}(q)=\delta_x(\delta_w(q))$. If $c$ carries only $v$, write $vx=yv'$ with $y\in\Sig$; then $\Phi_x(c)$ carries $v'$, and $\delta_{vx}=\delta_{v'}\circ\delta_y$ is constant with the value of $\delta_{v'}$, while $\delta_{vx}=\delta_x\circ\delta_v$ is constant with value $\delta_x(D(c))$.

\emph{Region and margin.} Let $U\coloneq\{c+e \mid c\in C,\ \|e^{(j)}\|_\infty\le j\eta \text{ for } j=1,\dots,k\}$, which is bounded since $C$ is finite and contains $\enc(Q)$. For $c+e\in U$ and $\|e'\|_\infty\le\eta$, $\Phi_x(c+e)+e'=\Phi_x(c)+(Ae+e')$ with $\Phi_x(c)\in C$; slot $1$ of $Ae+e'$ is $e'^{(1)}$, of norm at most $\eta$, and slot $j+1$ is $e^{(j)}+e'^{(j+1)}$, of norm at most $(j+1)\eta$. Hence $\Phi_x(U)\oplus\Ball\subseteq U$, which is \Rtwo. For $\eta<1/(2k)$, every coordinate of $e$ is below $\tfrac12$ in absolute value, while two distinct clean configurations differ by $1$ in some coordinate, so every point of $U$ has a unique clean configuration $c$, and $\dec(c+e)\coloneq D(c)$ is well defined. \Rone is $\dec(\enc(q))=D(\enc(q))=\delta_\varepsilon(q)=q$, and \Rthree is $\dec(\Phi_x(c+e)+e')=D(\Phi_x(c))=\delta_x(D(c))=\delta_x(\dec(c+e))$. Theorem~\ref{thm:equiv} gives realization with length independence. For another norm, $\|e\|\le\eta$ implies $\|e\|_\infty\le\mu\eta$, and the same argument applies with $\mu\eta$ in place of $\eta$.
\end{proof}

For $k=1$ the shift drops everything, $A=0$, and each symbol writes its own code: the realization uses no nonlinearity at all. 

\begin{corollary}[Contracting models need exponentially fine precision]
\label{cor:precision}
Let $\kappa \coloneq \max_x\|A_x\|_\op<1$, which is stronger than the hypothesis $\rho(A_w)<1$ of branch~(ii), let the target be nondefinite, and let $C \coloneq \max_{q,q'\in Q}\|\enc(q)-\enc(q')\|$ be the diameter of the code. At perturbation level $\eta>0$ the model is guaranteed correct on every $w$ with $|w|\le L$ only if $L\le L_{\max}$, where $L_{\max}=\lceil\ln(C/2\eta)/\ln(1/\kappa)\rceil-1$. Equivalently, staying correct on every word of length at most $L$ requires $\eta<C\kappa^L/2$: the perturbation level the model tolerates shrinks exponentially with $L$.
\end{corollary}
\begin{proof}
The operator norm is submultiplicative, so $A_w=A_{x_{|w|}}\cdots A_{x_1}$ of \eqref{eq:affine} obeys $\|A_w\|_\op\le\prod_i\|A_{x_i}\|_\op\le\kappa^{|w|}$, and $\Phi_w$ is $\kappa^{|w|}$-Lipschitz by Proposition~\ref{prop:singlerate}. Hence for $|w|=L$ and any $q,q'$ with $\delta_w(q)\ne\delta_w(q')$, the two code points, at distance at most $C$ apart, induce states separated by at most $C\kappa^L$. Lemma~\ref{lem:separate} forces that separation to exceed $2\eta$, so the bound must itself exceed $2\eta$: correctness at length $L$ requires $C\kappa^L>2\eta$, which is the stated $L_{\max}$. Non-definiteness is what makes the requirement bite, since it supplies, at every length, a word $w$ and a pair $q,q'$ that the target still keeps apart (a word $v$ with $|v|\ge L$ and $\delta_v$ not constant has a suffix of length $L$ whose action is not constant either, since $\delta_v$ factors through it); for a definite target every pair is allowed to merge after a bounded number of steps and there is no horizon to speak of.\end{proof}

For an isometry\footnote{A linear map $A$ is an \emph{isometry} when $\|Av\|=\|v\|$ for every $v$, equivalently $A^\Tr\!A=I$ in the Euclidean norm: it neither shrinks nor stretches any distance, so all its eigenvalues have modulus one. It is the distance-preserving representative of the $\rho=1$ branch of Theorem~\ref{thm:affine}, and covers unitary and orthogonal recurrences as well as $A_x=I$.}, the gap between two trajectories stays at $C$ while the perturbations pile up, and the horizon is the length at which the pile reaches the gap.

\begin{proposition}[Isometric models need precision of order $1/L$]
\label{prop:neutralbits}
Let every $A_x$ be a linear isometry of $\R^d$, let $w$ be a word of length $L$, and let $q\ne q'$ be states with $\delta_w(q)\ne\delta_w(q')$, encoded $C \coloneq \|\enc(q)-\enc(q')\|$ apart. Guaranteed correctness on every word of length at most $L$ requires $\eta<C/2L$: the perturbation level the model tolerates shrinks like $1/L$, so no fixed precision serves all lengths, in agreement with Lemma~\ref{lem:restore}.
\end{proposition}
\begin{proof}
An $\eta$-trajectory deviates from the unperturbed one by $\bar e_L=\sum_{t=1}^LB_te_t$, where $B_t \coloneq A_{x_L}\cdots A_{x_{t+1}}$ and $B_L \coloneq I$ are products of isometries, hence isometries, hence orthogonal and invertible. The same hypothesis keeps the two unperturbed trajectories from $\enc(q)$ and $\enc(q')$ exactly $C$ apart, so correctness fails as soon as the deviations can close that gap: the trajectories then reach a common point and decode alike, contradicting $\delta_w(q)\ne\delta_w(q')$.

Take $v \coloneq \bigl(\Phi_w(\enc(q'))-\Phi_w(\enc(q))\bigr)/C$ and $e_t \coloneq \eta B_t^{-1}v$, of norm $\eta$; then $\bar e_L=\eta Lv$. Running this along the trajectory from $\enc(q)$ and its negative along the one from $\enc(q')$ moves the two endpoints $2\eta L$ towards each other, and scaling the perturbations down if $2\eta L>C$ makes them coincide, so $2\eta L<C$ is necessary.
\end{proof}
\begin{corollary}[The horizon bound is not a design parameter]
\label{cor:untunable}
Write $\Lambda \coloneq \ln(C/2\eta)$, so that the bound of Corollary~\ref{cor:precision} reads $L_{\max}=\Lambda/\ln(1/\kappa)$ up to rounding. Then
\[
  \frac{dL_{\max}}{L_{\max}}
   = \frac{L_{\max}}{\Lambda}\,\frac{d\kappa}{\kappa}
  \xrightarrow[\;\kappa\to1^-\;]{}\infty.
\] Pinning $L_{\max}$ to a relative tolerance $\varepsilon$ therefore requires pinning $\kappa$ to a relative tolerance $\varepsilon\Lambda/L_{\max}$, which tightens as the horizon lengthens: at $C/2\eta=10^{6}$ and $L_{\max}=10^{4}$, a horizon specified to within a factor of two requires $\kappa$ specified to within $7\times10^{-4}$. Neither $\kappa$ nor $C$ is set by the designer: the first is set by learned, input-dependent gates, the second is the spread of the code the encoder learns, and the same holds on the other branch, where Proposition~\ref{prop:neutralbits} bounds the horizon by $C/2\eta$ with $C$ again learned. The horizon that a trained affine cell can at best reach is therefore a property of that particular trained model, measurable after the fact and not specifiable in advance.
\end{corollary}
\begin{proof}
Differentiate $L_{\max}=-\Lambda/\ln\kappa$: $dL_{\max}/d\kappa=\Lambda/(\kappa\ln^2\kappa)$, and divide by $L_{\max}$. The limit follows from $\ln(1/\kappa)\to0$.
\end{proof}

\subsection{Multistability}
\label{app:multistability}

The next lemma reads Theorem~\ref{thm:affine} as a requirement on the architecture. It places the nonlinearity on the previous state, since a gate computed by an arbitrarily nonlinear function of $x_t$ still leaves $\Phi_x$ affine in $h$.

\begin{lemma}[Two rates require the nonlinearity to act on the state]
\label{lem:affinefix}
Let the target automaton be nondefinite, for instance because $\Mtar$ contains a nontrivial group (Corollary~\ref{cor:affinegroup}), and let $\Phi$ satisfy \eqref{eq:cellincl} on a bounded region $U$ as in Theorem~\ref{thm:cells} at some $\eta>0$. Then some $\Phi_x$ is not affine.
\end{lemma}

\begin{proof}
Suppose every $\Phi_x$ is affine. Theorem~\ref{thm:cells} turns \eqref{eq:cellincl} on a bounded $U$ into realization of the target with length independence at $\eta$, so, by Theorem~\ref{thm:equiv}, some region satisfies \Rone to \Rthree. Theorem~\ref{thm:affine} forbids this for a nondefinite target: branch~(i) excludes every target, and branch~(ii) every nondefinite one. Hence some $\Phi_x$ is not affine.
\end{proof}

A restoring step can remove a perturbation outright, as the construction of Section~\ref{sec:architecture} does, or merely shrink it, as a real-analytic recurrence must (Remark~\ref{rem:exactgaps}). The next definition names the two grades.

\begin{definition}[Grades of restoration]
\label{def:regen}
Let $F$ be one of the maps $\Phi_x$, acting on the region $U=\bigsqcup_{q\in Q}U_q$ of Theorem~\ref{thm:cells}. $F$ is \emph{exactly restoring} on $U$ if it sends each cell $U_q$ to a single point, the perturbation being annihilated rather than reduced; for $F=\Phi_x$ we require that point to be $\enc(\delta_x(q))$, as in Section~\ref{sec:budget} once attractor indices are mapped onto automaton states. $F$ is \emph{asymptotically restoring} on a cell $U_q$ if there is $\lambda<1$ with $\|F(u)-F(u')\|\le\lambda\|u-u'\|$ for all $u,u'\in U_q$, perturbations being strictly reduced at every step. Exact restoration is the case $\lambda=0$.
\end{definition}

\begin{remark}[Exact restoration needs gaps or jumps]
\label{rem:exactgaps}
Let $\Phi_x$ be exactly restoring on $U$ in the sense of Section~\ref{sec:budget}, with the attractors mapped onto the code points, sending each cell $U_q$ to the single point $\enc(\delta_x(q))$. If $\Phi_x$ is continuous on $U$, it is constant on each connected component of $U$, since a continuous map from a connected set to a finite set is constant. Two cells whose states $\delta_x$ sends to different states therefore lie in different components: exact restoration by a continuous map requires gaps between the cells, which the \cellname provides by taking $U=\enc(Q)\oplus\Ball$, a union of disjoint balls. On the whole of $\mathcal{H}$, which is connected, a map with finitely many values cannot be continuous unless it is constant, so the jumps of the $\arg\max$ in \eqref{eq:construction} are unavoidable. A real-analytic recurrence, such as a $\tanh$ or sigmoid network, cannot be exactly restoring unless $\delta_x$ is constant: by \eqref{eq:cellincl}, a cell reached by some symbol contains an $\eta$-ball, $\Phi_x$ is constant on that cell, and a real-analytic map that is constant on a ball is constant on $\mathcal{H}$.
\end{remark}

To compare the two grades, and later affine maps and stacks, we count what a single map can keep apart under perturbation: the sets it maps into themselves with a margin of $\eta$.

\begin{definition}[Held sets and robust capacity]
\label{def:capacity}
Let $F:\mathcal{H}\to\mathcal{H}$ and $\eta>0$. $F$ \emph{holds} a set $V\subseteq\mathcal{H}$ if $V$ is nonempty and bounded and $F(V)\oplus\Ball\subseteq V$. The \emph{robust capacity} of $F$ in a region $U$ is $\log_2N$ bits, where $N$ is the largest number of pairwise disjoint subsets of $U$ that $F$ holds.
\end{definition}

Under \eqref{eq:cellincl}, a word $w$ with $\delta_w(q)=q$ makes $\Phi_w$ hold the cell $U_q$, by the lifting in the proof of Theorem~\ref{thm:tworates}. A word whose action fixes $m$ states therefore forces a robust capacity of at least $\log_2m$ bits on $\Phi_w$; for a nonconstant idempotent $\delta_w$ (Lemma~\ref{lem:definite}), $m\ge2$.

\begin{proposition}[The two grades have the same capacity]
\label{prop:samecapacity}
Let $U=\bigsqcup_{q\in Q'}U_q$ be a union of $n$ pairwise disjoint closed cells, and let $F$ hold every cell and be asymptotically restoring on it, with constant $\lambda\in[0,1)$. Then \emph{(i)} each $U_q$ contains a unique fixed point $\alpha_q$ of $F$, which attracts every trajectory of $F$ started in $U_q$, and the open $\eta$-ball around $\alpha_q$ lies in $U_q$; \emph{(ii)} every subset of $U$ that $F$ holds contains one of these balls, so the robust capacity of $F$ in $U$ is exactly $\log_2n$ bits. The count is the same for exact restoration ($\lambda=0$) as for asymptotic restoration ($0<\lambda<1$), and it does not mention $d$.
\end{proposition}
\begin{proof}
\emph{(i)} Each cell is closed and bounded, hence complete, and $F(U_q)\subseteq F(U_q)\oplus\Ball\subseteq U_q$. $F$ restricted to $U_q$ is a $\lambda$-contraction of $U_q$ into itself, so the Banach fixed point theorem gives a unique $\alpha_q\in U_q$ with $F^k(u)\to\alpha_q$ for every $u\in U_q$. For $k\ge1$, $F^k(u)\in F(U_q)$, so $\bar B_\eta(F^k(u))\subseteq U_q$; any $z$ with $\|z-\alpha_q\|<\eta$ satisfies $\|z-F^k(u)\|<\eta$ for $k$ large, hence $z\in U_q$.
\emph{(ii)} Let $V\subseteq U$ be held by $F$, pick $q$ with $V\cap U_q\neq\emptyset$ and $v$ in this intersection. Then $F^k(v)$ stays in $U_q$, since $F$ holds $U_q$, and converges to $\alpha_q$. It also stays in $V$, and for $k\ge1$ it lies in $F(V)$, so $\bar B_\eta(F^k(v))\subseteq V$; as in (i), the open $\eta$-ball around $\alpha_q$ lies in $V$. Two disjoint held subsets of $U$ therefore contain the balls of two distinct fixed points, so there are at most $n$ of them, and the $n$ cells are $n$ such subsets.
\end{proof}

\begin{proposition}[An affine recurrence has zero capacity at any width]
\label{prop:affinecapacity}
Let $\Phi_x$ be affine and satisfy \Rtwo on a nonempty bounded $U$ at some $\eta>0$. Then \emph{(i)} $\Phi_x$ has a unique fixed point $h^\ast_x$, the open $\eta$-ball around it lies in $U$, and every trajectory of $\Phi_x$ in $U$ converges to it; more generally, every set that $\Phi_x$ holds contains this ball, so the robust capacity of $\Phi_x$ (Definition~\ref{def:capacity}) is $\log_21=0$ bits in every region; \emph{(ii)} at most one cell is held, i.e., $\Phi_x(U_q)\oplus\Ball\subseteq U_q$ for at most one $q\in Q$. Neither conclusion mentions the state dimension $d$.
\end{proposition}
\begin{proof}
Lemma~\ref{lem:restore} gives $\rho(A_x)<1$, so $I-A_x$ is invertible, $\Phi_x$ has a unique fixed point $h^\ast_x$, and $\Phi_x^k(u)-h^\ast_x=A_x^k(u-h^\ast_x)\to0$ for every $u$. Fix $u\in U$: \Rtwo places $\bar B_\eta(\Phi_x^k(u))$ in $U$ for every $k\ge1$, and any $z$ with $\|z-h^\ast_x\|<\eta$ satisfies $\|z-\Phi_x^k(u)\|<\eta$ for $k$ large, so $z\in U$; the open $\eta$-ball at $h^\ast_x$ therefore lies in $U$ and its basin is all of $U$, however $U$ is divided into cells. The argument uses only that $\Phi_x$ holds $U$, so it applies to every set that $\Phi_x$ holds, and any two such sets meet. That is (i). For (ii), if two cells were held then the same argument places that ball in each of them, and cells of a partition are disjoint.
\end{proof}

\begin{remark}
\label{rem:zerocapacity}
Proposition~\ref{prop:affinecapacity} does not say that an affine recurrence carries no information: on the contracting branch it stays correct up to a horizon (Corollary~\ref{cor:precision}). The capacity that is zero is the length-independent one.
\end{remark}

 A stack of affine recurrences, with arbitrary nonlinear maps between its layers, is not affine as a whole, so Theorem~\ref{thm:affine} does not apply to it directly. Each layer, however, stays affine in its own hidden state.

\begin{definition}[Stack of affine recurrences]
\label{def:layerwise}
Split the hidden state into layers, $\mathcal{H}=\R^{d_1}\times\cdots\times\R^{d_s}$, each carrying the norm induced by $\|\cdot\|$. A map of $\mathcal{H}$ is \emph{layerwise affine} if the new value of each layer is an affine function of the old value of that same layer, whose coefficients may depend, arbitrarily, on the old values of the layers below it and on nothing else; the coefficients of the bottom layer are constant. A \emph{stack of affine recurrences} is a recurrence whose every $\Phi_x$ is layerwise affine.
\end{definition}

Every stack of affine layers is of this form, whatever the maps between them (MLPs, normalizations, residual connections, input-dependent gates), because what a layer receives at a step is computed from the input symbol and from the layers below it. A causal convolution between layers fits as well, once its buffer is counted as a layer of its own. No continuity is assumed. The composite of two layerwise affine maps is layerwise affine, so every composite $\Phi_w$ of a stack is.

\begin{theorem}[Depth does not create capacity]
\label{thm:depth}
Let $\Psi$ be layerwise affine, of any depth and any widths, and let $\eta>0$. Then $\Psi$ holds (Definition~\ref{def:capacity}) at most one set in any family of pairwise disjoint sets, so its robust capacity is zero in every region. Consequently, a stack of affine recurrences realizes with length independence at most a definite automaton, at any depth and any width.
\end{theorem}

\begin{proof}
The proof is by induction on the number of layers. With one layer, $\Psi$ is affine, and the proof of Proposition~\ref{prop:affinecapacity}(i) uses nothing but the fact that $\Psi$ holds the set: $\Psi$ has a unique fixed point, and every set that $\Psi$ holds contains the open $\eta$-ball around it. Two held sets therefore meet.

With more layers, suppose that $\Psi$ held two disjoint sets. The bottom layer reads nothing above it, so it evolves on its own as a single affine map. A perturbation confined to the bottom layer is a perturbation of the whole hidden state of the same size, so this map holds the bottom-layer projection of each of the two sets, and by the one-layer case its fixed point lies in both projections. Now cut both sets at this fixed point: keep, from each set, the hidden states whose bottom layer sits exactly at the fixed point, and drop the bottom layer. The two slices are nonempty, since the fixed point lies in both projections; bounded, since the sets are; and disjoint, since a point common to both slices would, with the bottom layer added back, be common to the two sets. A hidden state whose bottom layer sits at the fixed point keeps it there, because the bottom layer does not read the layers above, so on such hidden states $\Psi$ acts on the upper layers only. This action holds each slice: a perturbation of size at most $\eta$ on the upper layers, with none on the bottom one, keeps the hidden state inside its held set, and leaves the bottom layer at the fixed point. With the bottom layer frozen, finally, the coefficients of the second layer are constant, so the action on the upper layers is layerwise affine with one layer fewer. It holds two disjoint slices, which the induction hypothesis excludes. Nonlinear maps between layers thus change which point each layer settles on, but give no layer a second one.

For the consequence, suppose that a stack realizes a nondefinite target with length independence at $\eta$, with the cells of Theorem~\ref{thm:cells}. By Lemma~\ref{lem:definite}, some word $w$ acts on $Q$ as an idempotent that is not constant, so it fixes two distinct states. Lifting \eqref{eq:cellincl} to words, as in the proof of Theorem~\ref{thm:tworates}, shows that $\Phi_w$ holds the cells of both states, which are disjoint. Since $\Phi_w$ is layerwise affine, this contradicts the first part.
\end{proof}

\subsection{The scan budget}
\label{app:budgetproofs}

\propcollapse*
\begin{proof}[Proof of Proposition~\ref{prop:collapse}]
The representation $\zeta$ is injective into $\{0,1\}^b$, so $|\Mrnn|\le2^b$. The map $\Phi_w\mapsto\Phi_w(h_0)$ sends $\Mrnn$ onto the set of reachable states, $h_0=\Phi_\varepsilon(h_0)$ included, which therefore has $K\le|\Mrnn|\le2^b$ elements. For the action, $\Mrnn$ is closed under composition, since $\Phi_v\circ\Phi_u=\Phi_{uv}$, and contains the identity. The set of reachable states is invariant under every $\Phi_v$, since $\Phi_v(\Phi_w(h_0))=\Phi_{wv}(h_0)$, so each element of $\Mrnn$ restricts to a map of this $K$-element set into itself. Restriction preserves composition and the identity, so it is a monoid homomorphism from $\Mrnn$ to $T_K$, and $\Mrnn$ acts through its image, a submonoid of $T_K$. This homomorphism need not be injective, since two maps of $\Mrnn$ may agree on every reachable state.

The argument uses only that $\Mrnn$ is closed under composition and admits an injective $b$-bit representation, so it applies verbatim to the executed recurrence of Appendix~\ref{app:finite_precision}, with $\tilde\Phi_x$ in place of $\Phi_x$, as long as the merge is exact; a floating-point scan does not compose the executed maps, since its rounding depends on the shape of the tree.
\end{proof}


\begin{remark}[Which map the attractors belong to]
\label{rem:whichmap}
Exact restoration splits each step in two. Define $R$ on $V$ by $R(h)\coloneq a_k$ for $h\in V_k$, and let $E_x$ be any map of the attractors into $V$ with $E_x(a_k)\in V_{\sigma_x(k)}$. Then $\Phi_x=R\circ E_x\circ R$ for every $x$ and every such $E_x$. The map $R$ does not depend on the symbol, so it is the restoring organ of Section~\ref{sec:realization}, and $E_x$ is the executive organ: it only needs to land in the right basin, and $R$ does the rest. The inner $R$ is what lets an executive defined on the attractors act on a whole basin, which is what \Rthree asks for. For $R$, each $a_k$ is an attractor and $V_k$ is its basin: $R$ is idempotent, its fixed points are $a_1,\dots,a_K$, and every point of $V_k$ reaches $a_k$ in one step and stays there. The map $\Phi_x$ itself fixes $a_k$ exactly when $\sigma_x(k)=k$; if $\sigma_x$ is a cycle, the attractors form one attracting cycle of $\Phi_x$. The case $\sigma_x(k)=k$ is the multistable one of Section~\ref{sec:multistability} and the hypothesis of Proposition~\ref{prop:samecapacity}. In Section~\ref{sec:architecture}, $R=\rnd$ and $E_x(h)=\theta_x h$, which sends $\oh_k$ to column $k$ of $\theta_x$, a point of $V_{\sigma_x(k)}$; the outer $\rnd$ restores this output and the inner one the stored state.
\end{remark}

\section{The \texorpdfstring{\cellname}{NFSM} construction in detail}
\label{app:construction}

A head \emph{as it runs} holds its state as an attractor index and moves it by table lookups (Appendix~\ref{app:constr-runs}); the \emph{one-hot model} of Section~\ref{sec:architecture} carries each index $k$ by the attractor $a_k=\oh_k$, so that the theory, stated on a continuous state space, applies, and the two agree on every run (Appendix~\ref{app:constr-realizes}). Until Appendix~\ref{app:constr-heads}, we describe a single head, drop its index $i$, and write $d$ for its number of attractor indices, or \emph{indices} for short. For a matrix $\theta$, we write $\theta_{jk} \coloneq \oh_j^{\Tr}\theta\,\oh_k$ for the entry in row $j$ and column $k$.

\subsection{A head as it runs}
\label{app:constr-runs}

At step $t$, a learned affine map of the input $x_t$, reshaped into a $d\times d$ matrix, gives the logits $\theta_{x_t}$, with $(\theta_{x_t})_{jk}$ scoring the move from index $k$ to index $j$. The head reads off the table $\sigma_{x_t}$ of Section~\ref{sec:architecture} and applies it to its current index:
\begin{equation}
  \sigma_{x_t}(k) \coloneq \arg\max_{j\in\{1,\dots,d\}}(\theta_{x_t})_{jk},
  \qquad
  k_t = \sigma_{x_t}(k_{t-1}),
  \label{eq:indexstep}
\end{equation}
from a fixed initial index $k_0$. Ties in an $\arg\max$ are broken by a fixed arbitrary ordering of the indices, so every column has one largest entry, and the basins $V_k$ of Section~\ref{sec:architecture} partition $\R^d$; exact ties between two logits are in any case very unlikely, and Proposition~\ref{prop:margin} measures how far a column sits from one. The output of the head at step $t$ is a learned value vector attached to the index $k_t$. The state is an index of $\lceil\log_2d\rceil$ bits, and the only real numbers in a step are the logits and the value vectors.

\subsection{The tables, and their budget}
\label{app:constr-maximal}

\begin{proposition}[Every table, at a linear merge]
\label{prop:maximal}
\emph{(i)} For every $\tau\in T_d$ there are logits with $\sigma_x=\tau$, so the tables of a head range over the full transformation monoid $T_d$, the largest transition monoid on $d$ indices.
\emph{(ii)} A head is scan-representable with $b=d\lceil\log_2d\rceil$ and $c=O(d)$. This is at most a factor $\lceil\log_2d\rceil/\log_2d=1+o(1)$ above $\lceil\log_2|T_d|\rceil=\lceil d\log_2d\rceil$, the fewest bits that can name every element of $T_d$, against $\Theta(d^2)$ numbers and $\Theta(d^3)$ operations for a dense real transition matrix. Composites of tables are tables, so the scan returns exactly the sequential trajectory, whatever the shape of its tree.
\end{proposition}
\begin{proof}
\emph{(i)} Take $(\theta_x)_{jk} \coloneq 1$ if $j=\tau(k)$ and $0$ otherwise.

\emph{(ii)} An element of $T_d$ is the list $(\tau(1),\dots,\tau(d))$ of $d$ indices of $\lceil\log_2d\rceil$ bits each, which gives $b$. The composite $\tau'\circ\tau$ is obtained by the $d$ lookups $k\mapsto\tau'(\tau(k))$, which gives $c$; composition is associative with the identity table as unit, so Definition~\ref{def:scanrep} is met. Since $|T_d|=d^d$, naming every element needs $\lceil d\log_2d\rceil\ge d\log_2d$ bits, whence the factor. Every merge is an exact operation on integers, so the composite returned by the scan equals the sequential one, and so does every index.
\end{proof}

\subsection{The one-hot model}
\label{app:constr-realizes}

The one-hot model sets $\Phi_x(h) \coloneq \rnd(\theta_x\rnd(h))$ on $\R^d$ (Section~\ref{sec:architecture}). Write $P_x$ for the column one-hot matrix of $\sigma_x$, whose column $k$ is $\oh_{\sigma_x(k)}$; then $\Phi_x=P_x\circ\rnd$, and, since $\rnd(h)=\oh_k$ on the basin $V_k$,
\begin{equation}
  \Phi_x(h) = \oh_{\sigma_x(k)}\qquad\text{for every }h\in V_k.
  \label{eq:phiconstruction}
\end{equation}
By induction, the trajectory of the model from $\oh_{k_0}$ is the run \eqref{eq:indexstep} carried by its attractors, step for step. Every statement about such trajectories, including correctness at every length, is therefore a statement about the running head. The perturbations of Definition~\ref{def:noise} act on the model only, and the next proposition shows that a head tolerates them in any norm.

\begin{proposition}[One head realizes any automaton on at most $d$ states]
\label{prop:realizes}
\emph{(i)} For every choice of the logits, every $x\in\Sig$ and every index $k$, $\Phi_x(V_k)=\{a_{\sigma_x(k)}\}$: the head is exactly restoring on $\R^d$ in the sense of Section~\ref{sec:budget}, with attractors $a_k$, basins $V_k$ and tables $\sigma_x$.
\emph{(ii)} There is $\eta_0>0$, depending only on $d$ and on the norm, such that for every $0<\eta<\eta_0$ and every $k$, the ball $a_k\oplus\Ball$ lies in the set $D_k$ of points whose largest coordinate is unique and equal to the $k$-th. In particular, these balls lie in their basins and are pairwise disjoint.
\emph{(iii)} Let $n\le d$, number the states $Q=\{1,\dots,n\}$, and take $\enc(q) \coloneq \oh_q$ and $\dec(h) \coloneq \arg\max_{q\in Q}\oh_q^{\Tr}h$, as in Section~\ref{sec:architecture}. For $0<\eta<\eta_0$, let $U_q \coloneq \enc(q)\oplus\Ball$ and $U \coloneq \bigsqcup_{q\in Q}U_q$. Then \Rone holds, \Rtwo holds whenever $\sigma_x(Q)\subseteq Q$ for every $x\in\Sig$, and \Rthree holds if and only if $\sigma_x(q)=\delta_x(q)$ for every $x\in\Sig$ and $q\in Q$, that is, if and only if the largest entry of column $q$ of $\theta_x$ sits in row $\delta_x(q)$, for every $q$ and $x$.
In that case the head realizes the target with length independence at $\eta$. Such logits exist for every automaton on $n\le d$ states.
\end{proposition}
\begin{proof}
\emph{(i)} This is \eqref{eq:phiconstruction}.

\emph{(ii)} The set $D_k \coloneq \{u \mid \oh_k^{\Tr}u>\oh_j^{\Tr}u\text{ for every }j\ne k\}$ is open, contains $\oh_k$, and lies in $V_k$, and $D_k\cap D_{k'}=\emptyset$ for $k\ne k'$. Since there are finitely many indices, there is $\eta_0>0$ such that $\oh_k\oplus\Ball\subseteq D_k$ for every $k$ and every $\eta<\eta_0$, and the balls are disjoint because the sets $D_k$ are.

\emph{(iii)} \Rone is $\dec(\enc(q))=\arg\max_{q'\in Q}\oh_{q'}^{\Tr}\oh_q=q$. Let $u\in U_q$ and $\|e\|\le\eta$. By (ii), $u\in V_q$, so $\Phi_x(u)=\oh_{\sigma_x(q)}$ by (i). If $\sigma_x(q)\in Q$, then $\Phi_x(u)+e\in U_{\sigma_x(q)}\subseteq U$; hence $\Phi_x(U)\oplus\Ball\subseteq U$ whenever $\sigma_x(Q)\subseteq Q$, and $U$ is bounded: this is \Rtwo. In that case, $\Phi_x(u)+e$ also lies in $D_{\sigma_x(q)}$ by (ii), so $\dec(\Phi_x(u)+e)=\sigma_x(q)$, and \Rthree at $(x,q)$ holds if and only if $\sigma_x(q)=\delta_x(q)$. If instead $\sigma_x(q)=k\notin Q$, every coordinate of $\oh_k$ indexed by $Q$ vanishes, so for any $q'\in Q$ the perturbation $e \coloneq t\,\oh_{q'}$, with $t>0$ small enough that $\|e\|\le\eta$, gives $\dec(\Phi_x(u)+e)=q'$; since $n\ge2$, some $q'\ne\delta_x(q)$ violates \Rthree. Hence \Rthree holds if and only if $\sigma_x=\delta_x$ on $Q$, which also gives $\sigma_x(Q)\subseteq Q$, hence \Rtwo.

If $\sigma_x=\delta_x$ on $Q$ for every $x$, Theorem~\ref{thm:equiv} turns \Rone to \Rthree into realization with length independence. The logits of Proposition~\ref{prop:maximal}(i), with $\tau$ any table of $T_d$ that agrees with $\delta_x$ on $Q$, give $\sigma_x=\delta_x$ on $Q$, and neither $\eta_0$ nor this choice depends on the length of the word, on $\Sig$, or on the target.
\end{proof}

\subsection{The perturbations can affect the logits}
\label{app:constr-margin}

For a running head, the per-step error $e_t=h_t-\Phi_{x_t}(h_{t-1})$ of Appendix~\ref{app:finite_precision}, read through the one-hot model, takes only two kinds of values. If the table of step $t$ is computed correctly, $e_t=0$ exactly, since the state is an index and a step is a gather. Otherwise $e_t$ is the difference of two distinct attractors, a wrong table entry rather than a small deviation. What finite precision, quantized weights, or an imperfectly learned map can perturb is the logit matrix $\theta_{x_t}$, which is computed from the input alone, and that perturbation changes the table only if it reorders a column.

\begin{proposition}[Logit margin]
\label{prop:margin}
Let $\theta\in\R^{d\times d}$ have a unique largest entry in every column, with table $\sigma$, and write
\[
  \gamma(\theta) \coloneq \min_{k}\Bigl(\theta_{\sigma(k)k}-\max_{j\ne\sigma(k)}\theta_{jk}\Bigr)>0
\] for its smallest column margin. If $\max_{j,k}|\Delta_{jk}|<\gamma(\theta)/2$, then $\theta+\Delta$ has the same table $\sigma$. Consequently, if at every step every entry of $\theta_{x_t}$ is perturbed by less than $\gamma(\theta_{x_t})/2$, the executed tables and indices equal the exact ones at every step, for every word.
\end{proposition}
\begin{proof}
Fix a column $k$ and $j\ne\sigma(k)$, and write $\gamma \coloneq \gamma(\theta)$. Then
\[
  (\theta+\Delta)_{\sigma(k)k}>\theta_{\sigma(k)k}-\tfrac{\gamma}2\ge\theta_{jk}+\tfrac{\gamma}2>(\theta+\Delta)_{jk},
\] so the largest entry of column $k$ is still in row $\sigma(k)$. For the second claim, argue by induction on $t$: both runs start from $k_0$, and if they share $k_{t-1}$ they share the table at step $t$, hence $k_t$ by \eqref{eq:indexstep}. Nothing is carried from one step to the next, because the logits read $x_t$ and not the state.
\end{proof}

\subsection{Many heads}
\label{app:constr-heads}

The block holds $r$ heads: head $i$ has its own attractor indices $\{1,\dots,d_i\}$, its own logits $\theta^{(i)}_{x_t}\in\R^{d_i\times d_i}$ computed from the same input, and its own $\arg\max$. The heads update independently, so the joint index, the tuple of the indices of the heads, lies in $\prod_i\{1,\dots,d_i\}$, and the budgets add:
\begin{equation}
  b=\sum_{i=1}^r d_i\lceil\log_2d_i\rceil,
  \qquad
  c=O\Big(\sum_{i=1}^r d_i\Big),
  \qquad
  \text{state}=\sum_{i=1}^r\lceil\log_2d_i\rceil\ \text{bits},
  \qquad
  \text{logits}=\sum_{i=1}^r d_i^2 .
  \label{eq:headbudget}
\end{equation}
Each index of head $i$ owns a value vector, head $i$ emits the vector of the index it occupies, the $r$ vectors are concatenated, and one shared projection maps them back to the model width.

The block step applies the heads side by side, $\Phi_x(h)=\big(\Phi^{(1)}_x(h^{(1)}),\dots,\Phi^{(r)}_x(h^{(r)})\big)$ on $\mathcal{H}=\prod_i\R^{d_i}$. By Proposition~\ref{prop:realizes}(i) applied to each head, it is exactly restoring in the sense of Section~\ref{sec:budget} for every value of the logits: its attractors are the tuples of attractors of the heads, indexed by the joint indices $(k^{(1)},\dots,k^{(r)})$, and the tables of the heads act on the joint indices component by component. Realization maps the automaton states injectively into the joint indices by $\chi=(\chi_1,\dots,\chi_r)$, with $\enc(q)=\big(\oh_{\chi_1(q)},\dots,\oh_{\chi_r(q)}\big)$, and $\dec$ reads in every head the index of the basin that contains $h^{(i)}$, and inverts $\chi$. Heads that update independently can only carry a target whose transitions act on each head separately.

\begin{proposition}[Heads combine as a direct product]
\label{prop:heads}
\emph{(i)} The joint tables of $r$ heads form a submonoid of the direct product $\prod_iT_{d_i}$, within the budget \eqref{eq:headbudget}.
\emph{(ii)} Let $\chi=(\chi_1,\dots,\chi_r):Q\hookrightarrow\prod_i\{1,\dots,d_i\}$ be injective, with $\enc$ and $\dec$ as above. If $\sigma^{(i)}_x\circ\chi_i=\chi_i\circ\delta_x$ for every $x\in\Sig$ and every $i$, the block realizes the target with length independence at every $\eta$ below some $\eta_0>0$. Such tables exist if and only if every symbol acts head by head: $\chi_i(q)=\chi_i(q')$ implies $\chi_i(\delta_x(q))=\chi_i(\delta_x(q'))$, for every $x$, $i$, $q$ and $q'$.
\end{proposition}
\begin{proof}
\emph{(i)} The table of each head depends on the input alone, so the joint step, and every composite of joint steps, acts head by head; the budget is Proposition~\ref{prop:maximal}(ii) summed over heads.

\emph{(ii)} The set of points whose $\arg\max$ in each head is unique and spells $\chi(q)$ is open, contains $\enc(q)$, and does not meet the corresponding set of any $q'\ne q$, because $\chi$ is injective. As in the proof of Proposition~\ref{prop:realizes}, some $\eta_0>0$ keeps every ball of radius $\eta<\eta_0$ around a code point inside its set, and head $i$ maps that ball onto $\oh_{\sigma^{(i)}_x(\chi_i(q))}=\oh_{\chi_i(\delta_x(q))}$, so a step maps the cell of $q$ onto the code point of $\delta_x(q)$. Hence \Rone to \Rthree hold on the union of these balls, which is bounded, and Theorem~\ref{thm:equiv} gives realization with length independence. For the last claim, the condition is necessary, since $\chi_i(\delta_x(q))=\sigma^{(i)}_x(\chi_i(q))$ depends on $q$ through $\chi_i(q)$ only. It is sufficient, since it defines $\sigma^{(i)}_x$ on $\chi_i(Q)$ by $\sigma^{(i)}_x(\chi_i(q)) \coloneq \chi_i(\delta_x(q))$; any extension to the other indices is a table, attainable by Proposition~\ref{prop:maximal}(i).
\end{proof}

\begin{remark}[Direct products within a block, cascades across layers]
\label{rem:factor}
Proposition~\ref{prop:heads} covers heads that run side by side, none reading the index of another. In a cascade, as in the wreath products of the Krohn--Rhodes decomposition \citep{Krohn1965}, the table of one component depends on the current index of another. Inside a single block, a cascade is still a table on the joint indices, and hence still scannable, but only at the budget of that joint space of $\prod_id_i$ indices, not at \eqref{eq:headbudget}. Stacking layers builds a cascade without that cost. The logits of a layer are computed from the output of the layer below, which is the readout of its current index, so the tables of the upper layer depend on the index of the lower one, while each layer is still scanned at its own budget, one layer after the other. One difference with the cascades of Krohn--Rhodes matters. There, a component is updated from the \emph{previous} state of the components below it and the symbol; here, a layer reads the \emph{updated} index of the layer below at the same step. The two coincide when the lower transition is injective, as for groups, since the previous lower index is then determined by the updated one and the symbol. They differ when it is not: after a reset, the updated index no longer tells which index the lower layer came from. A causal convolution of width $2$ before each layer restores this information, since it lets a layer read the output of the layer below at the previous step as well. This is how the stack of \cellname layers operates on TSO (Section~\ref{sec:experiments}).
\end{remark}

\begin{proposition}[Stacked layers realize cascades]
\label{prop:stack}
Consider a stack of $s$ \cellname layers, layer $\ell$ having joint indices $S_\ell=\prod_i\{1,\dots,d^{(\ell)}_i\}$ over its heads, and suppose that the logits of layer $\ell$ at step $t$ are computed from the symbol $x_t$ and from the indices of the layers below it at the same step.
\emph{(i)} The joint index $(k^{(1)}_t,\dots,k^{(s)}_t)$ follows a table on $\prod_\ell S_\ell$ that depends on the symbol alone, so the stack runs on indices, and every composite of its steps is again such a table.
\emph{(ii)} Suppose that the logits of each layer can be chosen as any function of the symbol and of the indices of the layers below. Let $\chi=(\chi_1,\dots,\chi_s):Q\hookrightarrow\prod_\ell S_\ell$ be injective, write $\chi_\ell=(\chi_{\ell,1},\dots,\chi_{\ell,r_\ell})$ for its components over the $r_\ell$ heads of layer $\ell$, and $\chi_{<\ell} \coloneq (\chi_1,\dots,\chi_{\ell-1})$. Suppose that every symbol acts on each head from that head and the updated layers below it: for every $x$, $\ell$ and $i$ there is $\tau^{(\ell,i)}_x$ with
\[
  \chi_{\ell,i}(\delta_x(q)) = \tau^{(\ell,i)}_x\bigl(\chi_{<\ell}(\delta_x(q)),\,\chi_{\ell,i}(q)\bigr)
  \qquad\text{for every }q\in Q .
\]
Then there are logits for which the stack realizes the target with length independence at every $\eta$ below some $\eta_0>0$.
\emph{(iii)} Layer $\ell$ is scanned once the indices of the layers below are known, so the layers are scanned one after the other, each at its own budget \eqref{eq:headbudget}, for a total depth $O(s\log L)$.
\end{proposition}
\begin{proof}
\emph{(i)} Fix $x$. The table of layer $1$ depends on $x$ alone, so $k^{(1)}_t$ is a function of $(x,k^{(1)}_{t-1})$; inductively, the logits of layer $\ell$ are a function of $x$ and of $k^{(1)}_t,\dots,k^{(\ell-1)}_t$, themselves functions of $x$ and of the previous joint index, so $k^{(\ell)}_t$ is a function of $x$ and of the previous joint index. The joint step is therefore an element of $T_{|\prod_\ell S_\ell|}$ indexed by $x$, and composites of tables are tables (Proposition~\ref{prop:maximal}).
\emph{(ii)} For every symbol $x$, layer $\ell$, head $i$ and tuple $\kappa$ of indices of the layers below, choose logits whose table $\sigma^{(\ell,i)}_{x,\kappa}$ agrees with $\tau^{(\ell,i)}_x(\kappa,\cdot)$ wherever the latter is defined and is arbitrary elsewhere; such logits exist by Proposition~\ref{prop:maximal}(i), and the choice is allowed because the logits may depend arbitrarily on $(x,\kappa)$. In the one-hot model on $\prod_\ell\prod_i\R^{d^{(\ell)}_i}$, encode $q$ by the tuple of attractors of $\chi(q)$ and decode a point whose $\arg\max$ in each head spells $\chi(q)$ as $q$. As in the proof of Proposition~\ref{prop:heads}, these sets are open, pairwise disjoint because $\chi$ is injective, and some $\eta_0>0$ keeps every ball of radius $\eta<\eta_0$ around a code point inside its set. Let $u$ lie in the ball around $\enc(q)$. Every head of layer $1$ reads its index $\chi_{1,i}(q)$ and moves it to $\tau^{(1,i)}_x(\chi_{1,i}(q))=\chi_{1,i}(\delta_x(q))$. Inductively, once the layers below $\ell$ have moved to $\chi_{<\ell}(\delta_x(q))$, the heads of layer $\ell$ read this tuple, hence apply the tables $\tau^{(\ell,i)}_x(\chi_{<\ell}(\delta_x(q)),\cdot)$, and move from $\chi_{\ell,i}(q)$ to $\chi_{\ell,i}(\delta_x(q))$. The joint step therefore maps the cell of $q$ onto the code point of $\delta_x(q)$. Hence \Rone to \Rthree hold on the union of these balls, which is bounded, and Theorem~\ref{thm:equiv} gives realization with length independence.
\emph{(iii)} The tables of layer $\ell$ are determined once its inputs are, which are the symbols and the indices of the layers below; scanning the layers in order therefore costs $s$ scans, each of depth $O(\log L)$ and at the budget \eqref{eq:headbudget} of its layer.
\end{proof}

\begin{example}[A running product in $S_n$]
\label{ex:sn}
Let the state be the running product $g_1\cdots g_t$ of permutations of $\{1,\dots,n\}$, so $|Q|=n!$, and one head on $Q$ costs $b=n!\lceil\log_2n!\rceil$ bits and $(n!)^2$ logits. Encode the product instead by the images of the points $1,\dots,n-1$, $\chi_i(g) \coloneq i\cdot g$, with one head of $n$ indices per point: each head is a channel that carries where its point has been sent, and the $n-1$ channels together carry the product. A permutation is fixed by its values on $n-1$ points, so $\chi$ is injective, and reading $g_t$ sends every head from index $j$ to $j\cdot g_t$, the same table for every head. Proposition~\ref{prop:heads}(ii) applies with $r=n-1$ heads that may share one logit matrix: $b=(n-1)\,n\lceil\log_2n\rceil$ bits and $n^2$ logits. For $n=5$ this is $60$ bits and $25$ logits, against $840$ bits and $14400$ logits for one head of $120$ indices.
\end{example}

\subsection{The gradient estimator}
\label{app:constr-grad}

The $\arg\max$ in \eqref{eq:indexstep} has zero derivative wherever it is defined, so the logits receive no gradient as written. We train with the straight-through Gumbel-softmax estimator \citep{Jang2017}, and the forward pass stays on indices throughout. We describe one head.

\emph{Forward pass.} During training, the logits are scaled by a sampling temperature $\nu$, $u_{x_t} \coloneq \theta_{x_t}/\nu$, one independent standard Gumbel variable $(G_t)_{jk}$ is drawn per entry, and the table is $\sigma_t(k) \coloneq \arg\max_j\,(u_{x_t}+G_t)_{jk}$. At test time the Gumbel variables are removed and the table is $\sigma_{x_t}$ of \eqref{eq:indexstep}. A sampled table is still an element of $T_d$, so parts (i) and (ii) of Proposition~\ref{prop:realizes} hold at every training step.

\emph{Backward pass.} Let $\bar g_t(k)$ be the first-order effect on the loss of reading out index $k$ at step $t$, obtained from the value vectors. The effect of being at index $k$ at step $t$, counting all later steps, is the adjoint
\[
  \lambda_t(k) = \bar g_t(k)+\lambda_{t+1}\bigl(\sigma_{t+1}(k)\bigr),
\] which is backpropagation through $h_t=P_th_{t-1}$, the one-hot model on its attractors, with $P_t$ the column one-hot matrix of $\sigma_t$; it is computed with gathers, itself by a reverse parallel scan. At step $t$, only the column $k=k_{t-1}$ that the run actually used receives a gradient: with $\lambda_t$ read as the vector $(\lambda_t(1),\dots,\lambda_t(d))$ and $p \coloneq \mathrm{softmax}\bigl((u_{x_t}+G_t)\oh_k\bigr)$, the surrogate gradient with respect to that column of $u_{x_t}$ is $p\odot(\lambda_t-\langle p,\lambda_t\rangle)$, the softmax Jacobian applied to $\lambda_t$, and dividing by $\nu$ gives the gradient with respect to $\theta_{x_t}$. The one-hot model is thus what the backward pass differentiates, although no one-hot vector is ever formed.

We think that a custom kernel that implements the forward and backward passes of this estimator, with the parallel scan for the adjoint, could drastically reduce the memory and time cost of training. One observation is that the block itself is not compute-bound but rather bandwidth-bound. We left such optimization for future work. Another direction could draw on the dictionary of PD-SSM (Appendix~\ref{app:baselines}) to reduce the memory footprint. Instead of computing all $d^2$ logits of $\theta_x$ from the input, a head would learn $n_\Theta$ logit matrices $\Theta_1,\dots,\Theta_{n_\Theta}$ and compute only the mixture weights $\omega_x\in\R^{n_\Theta}$, with $\theta_x=\sum_{\ell}(\omega_x)_\ell\,\Theta_\ell$. The projection would then have $O(mn_\Theta+n_\Theta d^2)$ parameters instead of $O(md^2)$, and the backward pass would store $n_\Theta$ weights per step rather than $d^2$ logits, recomputing only the column it uses. The table remains the column-wise $\arg\max$ of $\theta_x$, so Proposition~\ref{prop:realizes} still applies. The current implementation is available following the interactive notebook whose URL is in the introduction.

\section{Experiments}
\label{app:experiments}

\subsection{Tasks}
\label{app:tasks}

Every task is a semiautomaton $(Q,\Sig,\delta)$ with a start state $q_0$. A sequence is a random word $x_1\cdots x_L$, and the target at every position $t$ is the state $\delta_{x_1\cdots x_t}(q_0)$. Targets are computed exactly, by a scan over the tables of $\delta$.

\textbf{Algebraic tasks.} Each algebraic task is the word problem of a finite monoid $M$ with generators $\Sig$. The states are the elements of $M$, $q_0$ is the identity, and a letter $x$ maps $q$ to $x\circ q$, so the target is the running composite of the letters. Letters are drawn uniformly unless stated otherwise. Table~\ref{tab:tasks} lists the tasks.

\begin{table}[ht]
\caption{Algebraic tasks. $|Q|$ is the number of reachable states, the order of the monoid.}
\label{tab:tasks}
\centering
\small
\begin{tabular}{llllr}
\toprule
Task & Monoid & Letters & $|\Sig|$ & $|Q|$ \\
\midrule
$\mathbb{Z}_2$ & cyclic group & residues $0,1$ & $2$ & $2$ \\
$\mathbb{Z}_{16}$ & cyclic group & residues $0,\dots,15$ & $16$ & $16$ \\
$S_3$ & symmetric group & adjacent transpositions $(i\;i{+}1)$ & $2$ & $6$ \\
$S_4$ & symmetric group & adjacent transpositions $(i\;i{+}1)$ & $3$ & $24$ \\
$A_5$ & alternating group & adjacent 3-cycles $(i\;i{+}1\;i{+}2)$ & $3$ & $60$ \\
$M_{11}$ & Mathieu group & $a,\,b,\,a^{-1},\,b^{-1}$ & $4$ & $7920$ \\
$\mathrm{DFF}_5$ & flip-flop monoid & identity, reset, set & $3$ & $3$ \\
FF & flip-flop monoid & identity, reset, set & $3$ & $3$ \\
\bottomrule
\end{tabular}
\end{table}

\emph{The Mathieu group.} $M_{11}$ is generated by two permutations of $\{0,\dots,10\}$,
\[
  a=(0\;1\;2\;3\;4\;5\;6\;7\;8\;9\;10),\qquad b=(2\;6\;10\;7)(3\;9\;4\;5),
\]
and has order $7920$. Like $A_5$, it is simple and nonsolvable.

\emph{The flip-flops.} The flip-flop monoid acts on $\{0,1\}$ through three maps: the identity, reset (constant $0$) and set (constant $1$). The target is the identity until the first reset or set, and then the last of them. The two tasks differ only in how letters are drawn. In $\mathrm{DFF}_5$, letters are uniform, except that four identities in a row are always followed by a reset or a set; the target is then a function of the last five letters, and the drawn words are definite. In FF, a letter is a reset or a set with probability $0.1$, and the identity otherwise. Runs of identities have unbounded length, with mean $9$, so the target depends on arbitrarily old letters and the drawn words are not. The length sweep draws FF with an identity probability of $0.99$ instead, so that runs of identities have mean length $99$. The test sequences are then long and their information sparse: the target must be held across long runs of identities, which strengthens the dependence on old letters.

\textbf{Tracking Shuffled Objects.} $\mathrm{TSO}_N$ is adapted from the task of \citet{Srivastava2023}. It involves $N\in\{3,4,5\}$ people (alice, bob, charlie, dana, eve) and as many items (ring, key, coin, stamp, card). Every word and punctuation mark is one token, for a vocabulary of $2N+7$ tokens. A sequence is made of sentences of five tokens, in three parts:
\begin{enumerate}
\item a preamble of $N$ sentences, one per person in order, that assigns the items at random: \emph{alice has the key .};
\item swap sentences, each between a uniformly drawn pair of distinct people: \emph{bob swaps with alice .};
\item a question about a uniformly drawn item, \emph{who has the key ?}, whose answer, the current holder of that item, is predicted at the last token.
\end{enumerate}
When the length is not a multiple of $5$, extra ``.'' tokens are inserted after randomly chosen sentences, at most one per sentence, so that every length is built from the same local patterns.

The underlying automaton reads one token at a time. Its state is a pair $(p,g)$: $p$ is either idle or the first name of the sentence in progress, and $g$ maps every item to its holder, or to ``unset'' before the preamble names it. A name read in the idle state is stored in $p$. A second name swaps the items of the two people and makes $p$ idle again. An item closes a sentence \emph{$p$ has the item}, assigning it to $p$. Every other token leaves the state unchanged, so the question does not modify it. From the start state, with $p$ idle and every item unset, the grammar reaches $44$, $202$ and $1132$ states for $N=3,4,5$. The target at every token is this state, and the answer at the last token is a function of the state and of the asked item.

Answering the question is not state tracking in itself: the answer depends on the asked item, which is not part of the state. The last token is ``?'' and the item is the token before it, so the causal convolution before a higher layer lets its update at ``?'' depend on the item, and the final MLP reads the answer off the states of all layers. The question thus tests whether the tracked state can be used, not only whether it is kept.

\subsection{Baselines}
\label{app:baselines}
All three baselines are affine in the state: $h_t=A_th_{t-1}+b_t$, where $A_t$ and $b_t$ depend only on the input $u_t$ of the layer. A parallel scan evaluates them with the merge $(A_2,b_2)\circ(A_1,b_1)=(A_2A_1,\,A_2b_1+b_2)$, and each baseline keeps $A_t$ in a class closed under products, so that a merge stays cheap. Each takes the place of the \cellname in the same backbone (Appendix~\ref{app:training}). The class of $A_t$ is what sets them apart, and it places them on different branches of Theorem~\ref{thm:affine}.

\textbf{Mamba$^{-}$.} This is the Mamba block \citep{Gu2024}, with the eigenvalues extended to negative values \citep{Grazzi2025}. The inputs $u_t\in\R^d$ and $z_t\in\R^d$ are two linear projections of the layer input, $u_t$ followed by a SiLU. The state has $d$ channels of $d_{\mathrm{state}}$ entries, and each entry evolves on its own:
\[
  h_t[i,j]=a_t[i,j]\,h_{t-1}[i,j]+\Delta_t[i]\,B_t[j]\,u_t[i],\qquad a_t[i,j]=2\exp\bigl(-\Delta_t[i]\,e^{\alpha[i,j]}\bigr)-1,
\]
with step sizes $\Delta_t=\mathrm{softplus}(W_\Delta u_t+b_\Delta)$, vectors $B_t,C_t\in\R^{d_{\mathrm{state}}}$ linear in $u_t$, and learned $\alpha\in\R^{d\times d_{\mathrm{state}}}$. The output is $y_t[i]=\bigl(\sum_jC_t[j]\,h_t[i,j]+D[i]\,u_t[i]\bigr)\,\mathrm{SiLU}(z_t[i])$, followed by a linear projection. Since $\Delta_t>0$, every $a_t[i,j]$ lies in $(-1,1)$, so Mamba$^{-}$ is on the contracting branch, $\rho<1$. We remove the causal convolution of the original block, so that all models receive the same input.

\textbf{AUSSM.} The adaptive unitary SSM \citep{Karuvally2025} uses the same block with a complex state and eigenvalues of unit modulus:
\[
  h_t[i,j]=e^{\mathrm{i}\,\phi_t[i,j]}\,h_{t-1}[i,j]+\Delta_t[i]\,B[j]\,u_t[i],
\]
where the phases $\phi_t\in\R^{d\times d_{\mathrm{state}}}$ are linear in $u_t$, and $B,C\in\mathbb{C}^{d_{\mathrm{state}}}$ are learned but do not depend on the input. The readout takes the real part, $\mathrm{Re}\sum_jC[j]\,h_t[i,j]$, and the rest of the block is as for Mamba$^{-}$. Every eigenvalue has modulus $1$, so AUSSM is on the isometric branch, $\rho=1$.

\textbf{PD-SSM.} PD-SSM \citep{Terzic2025} has a complex state $h_t\in\mathbb{C}^N$ and the transition
\[
  h_t=P_tD_t\,h_{t-1}+b_t,\qquad M_t=\sum_{k=1}^{n_M}\mathrm{softmax}(W_\pi u_t)[k]\,M_k.
\]
Here $P_t$ is the column-wise hardmax of $M_t$, a mixture of $n_M$ learned matrices $M_k\in\R^{N\times N}$, so each column of $P_t$ is one-hot. $D_t$ is diagonal with entries $r_t[j]\,e^{\mathrm{i}\,\psi_t[j]}$, whose moduli $r_t\in(0,1)^N$ and phases $\psi_t\in(0,2\pi)^N$ are the sigmoid outputs of two small MLPs of $u_t$, the second scaled by $2\pi$. The drive $b_t$ is linear in $u_t$, and the output is linear in $(\mathrm{Re}\,h_t,\mathrm{Im}\,h_t)$ with a skip connection from $u_t$. A product of two such matrices $PD$ is again of this form, and a merge costs $O(N)$ operations. The hardmax is nonlinear, but it acts on the input and not on the state: it selects the matrix $A_t=P_tD_t$, and the step remains affine in $h_{t-1}$. It does not restore the state. Each column of $P_tD_t$ has a single nonzero entry, of modulus less than $1$, so $\rho(P_tD_t)<1$ and PD-SSM is on the contracting branch. The dictionary is trained through the softmax surrogate gradient of \citet{Terzic2025}.

\subsection{Backbone, training and evaluation}
\label{app:training}
\textbf{Backbone.} All models share one backbone of width $m=128$ and differ only in the recurrent block. Tokens are embedded, passed through an MLP and a layer norm. Each layer then applies $h\leftarrow h+\mathrm{block}(\mathrm{conv}(\mathrm{LN}(h)))$ followed by $h\leftarrow h+\mathrm{MLP}(\mathrm{LN}(h))$, and a final layer norm, MLP and linear map give the logits. Every MLP is a gated linear unit with SiLU gate, hidden width $2m$ and dropout $0.1$. On TSO, conv is a causal depthwise convolution of width $2$ followed by a SiLU. On the algebraic tasks there is no convolution, so each block reads the current token alone. The baselines have four layers with a recurrent width of $256$: $d=256$ channels with $d_{\mathrm{state}}=16$ for Mamba$^{-}$ and AUSSM, and $N=256$ with $n_M=16$ dictionary matrices for PD-SSM. The \cellname has one layer on the algebraic tasks and two layers on TSO. On the algebraic tasks, the \cellname has a single head with one attractor index per element of the target for $\mathbb{Z}_2$, $\mathbb{Z}_{16}$, $\mathrm{DFF}_5$ and FF, and otherwise several heads that each track a quotient of the target and together separate its elements, that is, a map $\chi$ as in Appendix~\ref{app:constr-heads}: $(3,2)$ attractor indices for $S_3$, $(4,3,2)$ for $S_4$, $(5,12)$ for $A_5$, and $4$ heads of $11$ attractor indices for $M_{11}$, which track the images of four points. On $\mathrm{TSO}_N$, the first layer has $N+1$ heads and the second layer $N$ heads, each of $N+1$ attractor indices.

\textbf{Loss.} The loss is a cross-entropy on the state at every position, and on the answer at the last position for TSO. Two weights rescale the per-position terms. The first weight is a mask after the first error. For each training sequence, let $t^\star$ be the first position where the predicted state is wrong. Positions up to $t^\star$ have weight $1$, and later positions have weight $e^{-\lambda(t-t^\star)}$ with $\lambda=3$. Once the state is wrong, the later predictions start from a wrong state, and their errors are consequences of the error at $t^\star$ rather than errors of their own. Supervising them at full weight would ask the model to reach correct states from a wrong one, which pulls the transitions away from the target.  The mask concentrates the gradient on the first error, where a transition is actually wrong, and it moves along the sequence as training fixes earlier errors. The second weight is class balancing (especially for the DFF/FF and TSO). Each sequence contributes the weighted mean of its terms, and the loss is the mean over the batch. For the \cellname, a third term keeps the sampling exploratory at the first error. The Gumbel noise of Appendix~\ref{app:constr-grad} tries other entries of a column only while the logits of that column are not much more spread than the noise. Once they are, the sampled entry is always the $\arg\max$, and a wrong entry is no longer revisited. At the first error of each sequence, we therefore penalize the spread of the columns used there, beyond the standard deviation $\sigma_G=\pi/\sqrt6$ of the Gumbel noise:
\[
\mathcal{E}^{(b)}=\frac{1}{r}\sum_{i=1}^{r}\mathrm{relu}\bigg(\frac{s\big(u^{(i)}_{x_{t^\star}}\oh_{k^{(i)}_{t^\star-1}}\big)}{\sigma_G}-1\bigg)^{2},
\]
where $s$ is the standard deviation over the entries of a column, $u^{(i)}=\theta^{(i)}/\nu$ are the sampling logits of head $i$, and $k^{(i)}_{t^\star-1}$ is its attractor index before the step, treated as a constant. $\mathcal{E}^{(b)}=0$ for a sequence without error, and with several layers the term is averaged over layers. The full loss is
\[
\mathcal{L}=\frac{1}{B}\sum_{b=1}^{B}\Bigg[\frac{\sum_{t}w^{(b)}_t\,\mathrm{CE}\big(q^{(b)}_t\big)}{\sum_t w^{(b)}_t}+\mathrm{CE}\big(a^{(b)}\big)+\beta_s\,\mathcal{E}^{(b)}\Bigg],
\qquad
w^{(b)}_t=e^{-\lambda(t-t^\star)_+}\,\kappa_{q^{(b)}_t},
\]
where $\mathrm{CE}(q_t)$ is the cross-entropy of the state logits at position $t$, $\mathrm{CE}(a)$ that of the answer logits at the last position, on TSO only, and $\kappa_q$ the class-balancing weight of state $q$. The coefficient $\beta_s=0.03$ decays by a cosine to $0.003$ over the training steps $s$. The baselines use the same loss without $\mathcal{E}$.

Exploring the most effective learning strategy for the \cellname is beyond the scope of this work. Especially, we think that one could explore specific methods based on the finite-state nature of the model, for example by taking inspiratio from \citet{Bengio1996}.

\textbf{Optimization.} We use AdamW \citep{Loshchilov2019} with learning rate $10^{-3}$, cosine-decayed to $10^{-4}$, weight decay $10^{-4}$ and gradient clipping at norm $1$. The \cellname is trained with the estimator of Appendix~\ref{app:constr-grad} at temperature $\nu=0.5$. Batches contain $256$ sequences of length $64$ on the algebraic tasks, and $64$ sequences of length $256$ on TSO. Every $50$ steps we measure the sequence accuracy on $512$ fresh sequences ($128$ for $M_{11}$). Training stops after $40$ consecutive checks at accuracy $1$ for the \cellname and $0.99$ for the baselines, or after $10^5$ steps, and we keep the checkpoint with the best validation accuracy.

\textbf{Evaluation.} Models are evaluated without dropout and, for the \cellname, without Gumbel noise. The predicted state at a position is the $\arg\max$ of the state logits. The \emph{sequence accuracy} is the fraction of sequences whose predicted states are correct at every position:
\[
  \mathrm{acc}_{\mathrm{seq}}=\frac{1}{B}\sum_{b=1}^{B}\prod_{t=1}^{L}\mathbb{I}\bigl[\hat q_t^{(b)}=q_t^{(b)}\bigr],
\]
where $q_t^{(b)}$ is the target state of sequence $b$ at position $t$, and $\hat q_t^{(b)}$ the prediction. A single wrong position makes the whole sequence count as wrong, so the measure does not dilute a late error over many correct positions. On TSO, the model also predicts, at the last token, which person holds the asked item. The \emph{question accuracy} is the fraction of sequences whose predicted holder is correct:
\[
  \mathrm{acc}_{\mathrm{qst}}=\frac{1}{B}\sum_{b=1}^{B}\mathbb{I}\bigl[\hat a^{(b)}=a^{(b)}\bigr],
\]
with $a^{(b)}$ the true holder and $\hat a^{(b)}$ the $\arg\max$ of the answer logits. It scores the task as a reader would: one answer per sequence, after a stream of arbitrary length. We report the sequence accuracy on the algebraic tasks and the question accuracy on TSO. The failing length $L_{\max}$ is the smallest swept length at which this accuracy falls below $1$ for the \cellname and below $0.9$ for the baselines.

\subsection{Seed-to-seed variability}
\label{app:seeds}
Table~\ref{tab:seeds} complements Table~\ref{tab:results} with the failing length of each of the five seeds. Where the baselines learn, their failing length varies across seeds.

\begin{table}[ht]
\caption{\textbf{Seed-to-seed variability of the failing length.} For each task and model, the failing length $L_{\max}$ of each of the five seeds, and the median and range $[\min;\max]$ over the seeds that learned the task; the median is the value reported in Table~\ref{tab:results}. $\times$ marks a seed that did not learn the task; other notation as in Table~\ref{tab:results}.}
\label{tab:seeds}
\centering
\input{assets/figure_3.tex}
\end{table}

\subsection{Geometry of the RNN states of the affine baselines}
\label{app:pca}

Theorem~\ref{thm:affine} is a statement about the state of a trained model, so we look at it. Figure~\ref{fig:pca} follows the recurrent state on $\mathbb{Z}_5$, a task we train separately for this figure, driven by the constant word $(+1)^T$, on which the five target states must stay apart forever. For each model, the state of the last recurrent layer is projected on its first two principal components, computed on that run alone, one point per step and colored by the target state at that step; the shaded regions are the decision regions of the classifier at the end of the backbone, read in the same plane, so a point is decoded correctly when its color matches its region. The three columns are three windows of the same run, at increasing time.

The \cellname keeps five points, one per state, and the picture does not change from one window to the next: the run visits attractors, and restoration puts the state back on one of them at every step (Proposition~\ref{prop:realizes}). The three affine baselines all lose this. Mamba$^{-}$ contracts: the five clusters collapse onto a single point, which lands on a cell boundary, so after roughly one hundred steps the readout no longer separates the states. This is the $\rho<1$ branch, where forgetting the input also forgets the state. AUSSM shows the $\rho=1$ branch: the clusters do not collapse onto one point, but they stop being separated, and the states that started apart end up interleaved within the same few cells.

\begin{figure}[p]
  \centering
  \includegraphics[width=0.82\textwidth]{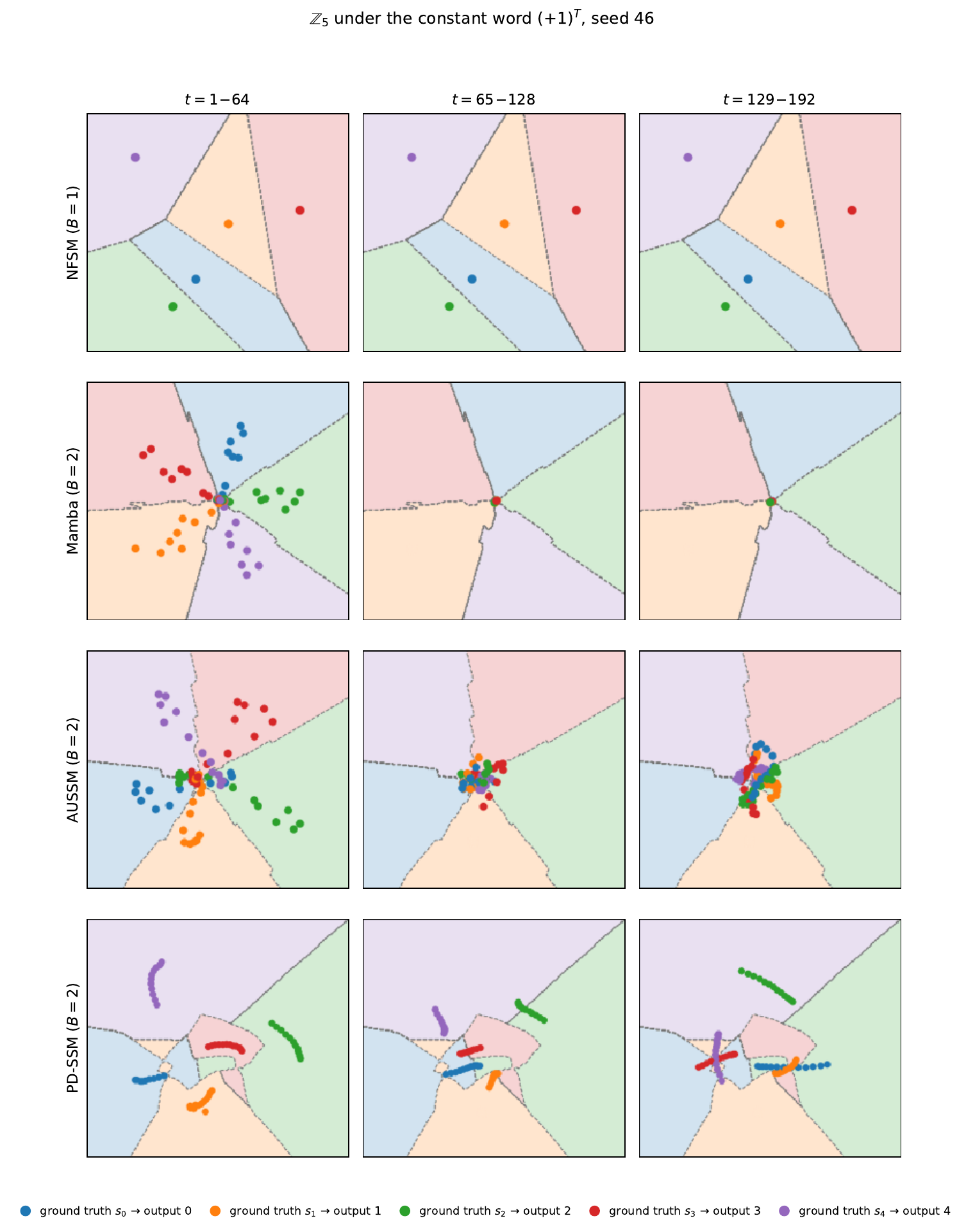}
  \caption{\textbf{The state of a trained model on $\mathbb{Z}_5$ under the constant word $(+1)^T$, seed $46$.} Each row is a model, each column a window of the same run. Points are the recurrent state projected on its first two principal components, one per step, colored by the target state; the shaded regions are the decision regions of the classifier at the end of the backbone. $B$ is the number of layers of the model.}
  \label{fig:pca}
\end{figure}

\subsection{Extraction of the transition tables}
\label{app:table-extraction}

Because every \cellname step commits to a table entry, the tables a trained model executes can be read off its logits, and compared with the target entry by entry. This is a strength of the architecture: it turns the evaluation of length independence into a verification. Testing up to a length $L$ only shows that the model is correct up to $L$, whereas tables that match the target show, by Proposition~\ref{prop:realizes}, that it is correct on every word of every length. An affine recurrence offers nothing comparable, since it has no finite table to read.

\subsubsection{Method}
\label{app:extract-method}

The procedure applies when the input alphabet $\Sig$ is finite, and uses the model at test time, without Gumbel noise.

\textbf{One layer, algebraic tasks.} For each symbol $x\in\Sig$, we pass $x$ alone through the backbone up to the \cellname layer and read the logits $\theta^{(i)}_x$ of every head $i$. The column-wise $\arg\max$ gives the table $\sigma^{(i)}_x\in T_{d_i}$, that is, the state that head $i$ moves to from each of its $d_i$ states when it reads $x$. Repeating this for all $|\Sig|$ symbols gives every transition of the layer.

\textbf{Convolutions and stacks.} With a causal convolution of width $k$ before each layer and several layers, as on TSO, the logits of a layer at step $t$ depend on a \emph{window} rather than on the current symbol: the last $k$ tokens for the first layer, and for a higher layer the last $k$ positions of the residual stream, which are functions of the tokens and of the states of the layers below (Remark~\ref{rem:factor}). Every window takes values in a finite set, so the method extends in principle: the tables of each layer can be extracted for every window that can occur, from the bottom layer upwards. A causal convolution only remembers its last $k-1$ inputs, a finite amount of memory, so the stack remains a finite-state machine and its tables still describe it completely. The windows that can occur grow quickly with the number of layers and of objects, however, and we do not carry out this enumeration on TSO. Appendix~\ref{app:tso-fsm} instead inspects what the trained stack computes.

\textbf{Matching with the target.} The tables are indexed by the attractor indices of the heads, not by the states of the task. To relate the two, we explore the joint indices that the tables reach from the initial joint index $s_0$ of the model, and assign a task state to each of them: $q_0$ to $s_0$, and, whenever $\sigma_x$ sends a joint index assigned $q$ to $s'$, $\delta_x(q)$ to $s'$. If no joint index is assigned two task states, this defines a map $c$ from the reached joint indices to the task states, which satisfies $c\circ\sigma_x=\delta_x\circ c$ by construction; when $c$ is a bijection, its inverse is the map $\chi$ of Appendix~\ref{app:constr-heads}. We check that no joint index is assigned two task states, and that for every reached joint index $s$ and every symbol $x$, the readout of $\sigma_x(s)$, after reading $x$, predicts $\delta_x(c(s))$. The model is then correct at every length: by induction on the length of the word, the joint index after reading $w$ is $\sigma_w(s_0)$, its task state is $\delta_w(q_0)$, and its readout predicts this state. The exploration covers every reachable joint index and every symbol, so the check is exhaustive rather than sampled.

\textbf{Beyond finite alphabets.} When the inputs do not come from a finite alphabet, as with continuous features, the enumeration above is no longer possible. The state space is still finite, however, so a head can execute at most $d_i^{d_i}$ distinct tables, and the input space splits into finitely many regions on each of which the table is constant. Extracting the model then amounts to describing these regions, for example as conditions on the input such as membership in an extracted interval. We leave this to future work.

\subsubsection{Algebraic tasks}
\label{app:extract-algebraic}
We apply the method to the models of Table~\ref{tab:results}, for every algebraic task and all five seeds (Table~\ref{tab:extract}). Every model passes the check: no joint index is assigned two task states, and the readout predicts the assigned state on every reached transition. By Proposition~\ref{prop:realizes}, every model is therefore correct at every length, and the failing lengths $\geq L$ of Table~\ref{tab:results} are certified beyond the tested lengths.

The exploration reaches exactly one joint index per task state, so $c$ is a bijection between the reached joint indices and the reachable states of the task, and its inverse is the map $\chi$ of Appendix~\ref{app:constr-heads}: the heads store the task state without redundancy. Every attractor index of every head is reached, and the joint indices that are never reached are exactly the combinations that the task cannot produce. On $M_{11}$, for instance, the $4$ heads of $11$ attractor indices reach $7920$ of the $11^4=14641$ joint indices, one for each element of the group. Single-head models of $S_3$, $S_4$ and $A_5$, with one head of $6$, $24$ and $60$ attractor indices, pass the same check on all five seeds.

\begin{table}[!ht]
  \centering
  \small
  \caption{\textbf{Extraction on the algebraic tasks.} For each task, the head sizes $(d_i)$, the number of symbols $|\Sig|$, the number $\prod_i d_i$ of joint indices, the number of joint indices reached from the initial joint index, and the number of reachable task states.}
  \label{tab:extract}
  \begin{tabular}{l c c c c c c}
    \toprule
    Task & Heads $(d_i)$ & $|\Sig|$ & Joint indices & Reached & Task states \\
    \midrule
    $\mathbb{Z}_2$ & $(2)$ & $2$ & $2$ & $2$ & $2$  \\
    $\mathbb{Z}_{16}$ & $(16)$ & $16$ & $16$ & $16$ & $16$  \\
    $S_3$ & $(3,2)$ & $2$ & $6$ & $6$ & $6$  \\
    $S_4$ & $(4,3,2)$ & $3$ & $24$ & $24$ & $24$  \\
    $A_5$ & $(5,12)$ & $3$ & $60$ & $60$ & $60$  \\
    $M_{11}$ & $(11,11,11,11)$ & $4$ & $14641$ & $7920$ & $7920$  \\
    $\mathrm{DFF}_5$ & $(3)$ & $3$ & $3$ & $3$ & $3$  \\
    FF & $(3)$ & $3$ & $3$ & $3$ & $3$  \\
    \bottomrule
  \end{tabular}
\end{table}

\textbf{An example on $S_3$.} The extracted machines are small enough to read. Take the model of $S_3$ with seed $42$, whose alphabet has the two transpositions $x_1=(1\,2)$ and $x_2=(2\,3)$ of $\{1,2,3\}$. Its first head has $3$ states and its second head $2$:
\[
\sigma^{(1)}_{x_1}=(2,1,3),\quad \sigma^{(1)}_{x_2}=(1,3,2),\qquad \sigma^{(2)}_{x_1}=\sigma^{(2)}_{x_2}=(2,1),
\]
where each table lists the images of the attractor indices $1,2,\dots$ in order. The first head applies each transposition to its index, and the second head flips at every symbol. The map $c$ sends the joint index $(j,k)$ to the permutation $g$ with $g(1)=j$ and sign $(-1)^{k-1}$, which is unique in $S_3$. The first head therefore tracks the image of the point $1$ and the second head the sign of the permutation, and the two together determine it. This is the factorization announced in Appendix~\ref{app:training}: each head tracks a quotient of the target, and together the heads separate its elements.

\subsubsection{Tracking Shuffled Objects}
\label{app:tso-fsm}
On TSO, we do not certify the model by extraction (Appendix~\ref{app:extract-method}). We instead show what a trained stack of two layers computes on $\mathrm{TSO}_4$. Figure~\ref{fig:tso-trajectory} follows the attractor index of every head, in both layers, as the model reads one stream. To interpret the heads, we map each attractor index of a head, over many streams, to the value it determines of a register of the task: the first name of the sentence in progress, or the holder of an item.

The resulting interpretation is clear for the second layer. Each of its heads tracks the holder of one item exactly. It is unset until the preamble names its item, takes the holder at the item token, and then changes only at the second name of a swap that involves this holder. The question leaves it unchanged, so at the last token the head of the asked item holds the answer. The heads of the first layer have no clear interpretation when taken separately: none of them matches a single register. Together, however, they must build a parser. No head of the second layer tracks the first name of the sentence in progress, yet the second layer updates the holders at the right tokens, and the readout predicts the first name at every position. The first layer must therefore pass this information to the second layer, along with the current token.

\begin{figure}[ht]
  \centering
  \includegraphics[width=\textwidth]{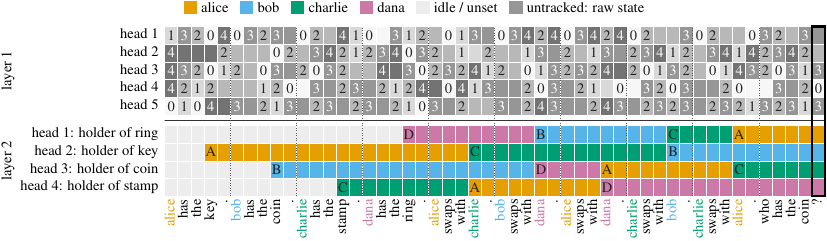}
  \caption{\textbf{Attractor indices of the heads of a two-layer \cellname reading a $\mathrm{TSO}_4$ stream.} Each row is a head, each column a token of the stream. A head whose attractor indices determine a register of the task is colored by its value, a person or unset; the other heads show their raw index in gray. The box marks the question token, where the answer is read.}
  \label{fig:tso-trajectory}
\end{figure}

\subsection{Wall-clock measurements}
\label{app:walltime}

Sections~\ref{sec:budget} and~\ref{sec:architecture} claim that the \cellname runs under a parallel scan whose merge is a gather (Proposition~\ref{prop:maximal}). We measure the cost of one layer of the \cellname and of the baselines of Section~\ref{sec:experiments}, Mamba$^{-}$, AUSSM and PD-SSM, against the sequence length $L$ and the state size $d$ (Figure~\ref{fig:walltime}).

\textbf{Setup.} Each layer is timed alone: the embedding, the causal convolution of the backbone and the MLP are the same for all models and are left out. The \cellname has one head of $d$ attractor indices, Mamba$^{-}$ and AUSSM have $d$ channels with $d_{\mathrm{state}}=16$, and PD-SSM has $N=d$. We vary $L$ from $2^6$ to $2^{17}$ at $d=16$, and $d$ from $2$ to $256$ at $L=2^{12}$. The inputs are batches of $16$ sequences of Gaussian vectors of width $256$, in float32 (with complex64 states for AUSSM and PD-SSM). For each model, we compare the parallel scan with a sequential loop over time steps that computes the same layer. We time a training step, that is, the forward pass and the gradient of the mean squared output with respect to the parameters, including the Gumbel noise and surrogate gradient of the \cellname and the surrogate of PD-SSM. Every configuration runs on one NVIDIA A100 (80\,GB). The step is compiled once, called $5$ times to warm up, and timed over $10$ synchronized calls. We report the median; compilation time is not included.

\textbf{Chunked runs.} By default, a layer processes the whole sequence at once. When a configuration does not fit in memory this way, we run it \emph{chunked}: the sequence is split into consecutive blocks of $C$ steps, the layer processes one block at a time and passes its last state on to the next block, and the backward pass recomputes the intermediate quantities of each block instead of storing them for the whole sequence. The result is unchanged up to round-off (which is, once again, another source of perturbation introduced by an implementation choice), and memory is bounded by the size of a block instead of by $L$, but the blocks are processed one after the other, which costs time. They are circled in red in Figure~\ref{fig:walltime}. We emphasize that the chunked approach can be used to fit any configuration in memory, at the cost of time, and that the parallel scan is still used within each block. It was in practice used interchangeably with the unchunked approach during the training of the models depending on the available memory.

\textbf{Memory.} The peak memory in Figure~\ref{fig:walltime} (bottom) is approximate and should be read with care. The panel shows orders of magnitude and trends, not exact footprints.

\textbf{Results.} \emph{Against the length.} The parallel scan is faster than the sequential loop at every measured length, and the gap widens with $L$. At $d=16$, the \cellname is also the fastest of the four at every length. \emph{Against the state size.} At $L=2^{12}$, the \cellname is the fastest model from $d=4$ to $d=64$. It grows faster with $d$ than the two diagonal models, since it computes $d^2$ logits at every step. It remains faster than PD-SSM, which also computes a $d\times d$ matrix at every step, at every state size we measure.

In practice, some architecture have already been optimized for speed and memory by the field, and custom kernels exist for them. We did not used them in our measurements, and we agree on the fact that this is not a fair comparison on deployment. The message of this experiment is not that some models are faster than others, but that the \cellname is compatible with a parallel scan and that it is fast and memory-efficient enough to be trained on long sequences; with scalings that are comparable to those of the baselines, \emph{it is a practical architecture.}

\begin{figure}[t]
  \centering
  \includegraphics[width=\textwidth]{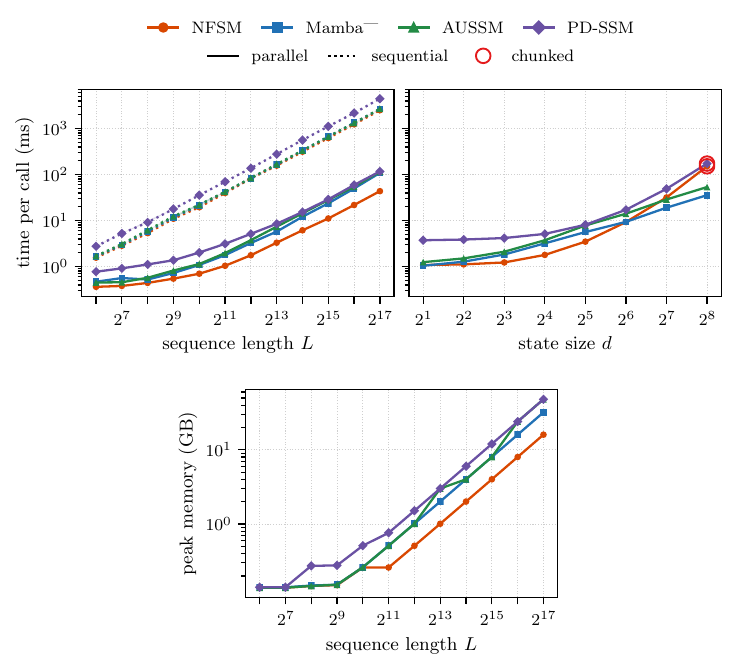}
  \caption{\textbf{Cost of one layer in a training step} (forward and backward passes, batch of $16$, one A100). \textbf{Top left:} time per call against the sequence length $L$, at $d=16$; solid lines are the parallel scan, dotted lines the sequential loop. \textbf{Top right:} time per call of the parallel scan against the state size $d$, at $L=2^{12}$. \textbf{Bottom:} peak memory of the parallel scan against $L$, at $d=16$; this is an approximate measurement (see text). Points circled in red are \emph{chunked}: the sequence is processed in consecutive blocks, which bounds memory at the cost of time.}
  \label{fig:walltime}
\end{figure}

\end{document}

%% file: math_commands.tex
\usepackage{amsmath,amsfonts,bm}

\def\eqref#1{equation~\ref{#1}}

\def\1{\bm{1}}

\def\ry{{\textnormal{y}}}

\DeclareMathAlphabet{\mathsfit}{\encodingdefault}{\sfdefault}{m}{sl}
\SetMathAlphabet{\mathsfit}{bold}{\encodingdefault}{\sfdefault}{bx}{n}

\newcommand{\R}{\mathbb{R}}

\DeclareMathOperator{\Tr}{Tr}

%% file: assets/figure_1.tex
%
%
\providecolor{cellA}{HTML}{CFE2F3}
\providecolor{cellB}{HTML}{D6ECD2}
\providecolor{cellC}{HTML}{FCE4C4}
\providecolor{cellD}{HTML}{E4DAF1}
\providecolor{cellE}{HTML}{F7D6D9}
\providecolor{codecol}{HTML}{C8283C}
\providecolor{deccol}{HTML}{1F5FA8}
\providecolor{failcol}{HTML}{E1700F}
\colorlet{inkcol}{black!82}

\begin{tikzpicture}[
  >={Stealth[length=6.5pt,width=5pt]}, font=\Large,
  declare function={xt(\t)=9.9+2.25*\t;
                    yb(\k)=(\k==0)*0.80+(\k==1)*1.65+(\k==2)*2.60+(\k==3)*3.35+(\k==4)*4.0;},
  autstate/.style={circle,draw=black!55,line width=.7pt,minimum size=11mm,inner sep=0pt,font=\large},
  atrans/.style={->,draw=black!55,line width=.8pt},
  ftrans/.style={->,draw=failcol,line width=1.8pt},
  tok/.style={rounded corners=2.5pt,draw=black!40,line width=.5pt,minimum width=8.5mm,
              minimum height=6mm,inner sep=0pt,font=\large},
  mtok/.style={rounded corners=1.5pt,draw=black!40,line width=.4pt,minimum width=4.5mm,
               minimum height=4.6mm,inner sep=0pt,font=\normalsize},
  cpoint/.style={circle,fill=codecol,draw=white,line width=.8pt,inner sep=0pt,minimum size=7.5pt},
  nominal/.style={draw=inkcol,line width=1.1pt},
  branch/.style={draw=black!60,line width=1.1pt,dash pattern=on 3pt off 2pt},
  bar/.style={line width=4.2pt,line cap=round,draw=black!50},
  fbar/.style={bar,draw=failcol},
  squeeze/.style={-{Stealth[length=4.5pt,width=3.8pt]},draw=black!65,line width=.75pt},
  rowtitle/.style={anchor=base west,inner xsep=0pt,font=\fontsize{15.8}{19}\selectfont\bfseries,text=inkcol},
  axlabel/.style={font=\large,text=black!55},
  legtitle/.style={anchor=base west,font=\Large\bfseries,text=inkcol},
  phase/.style={circle,draw=inkcol,line width=.7pt,fill=white,inner sep=0pt,minimum size=4.2mm,
                font=\normalsize\bfseries,text=inkcol},
  note/.style={font=\Large\bfseries,text=failcol,align=left,inner sep=3pt,rounded corners=3pt,
               fill=white,fill opacity=.8,text opacity=1},
]

\def\Hrect{(8.9,0) rectangle (26.9,4.8)}
\def\bands{0/0.40/1.20/cellA/1,1/1.20/2.10/cellB/2,2/2.10/3.10/cellC/3,3/3.10/3.60/cellD/4,4/3.60/4.40/cellE/5}
\def\statespace#1{
  \begin{scope}[shift={(0,#1)}]
    \fill[black!4,rounded corners=6pt] \Hrect;
    \foreach \k/\lo/\hi/\c/\q in \bands{
      \fill[\c] (9.2,\lo) rectangle (26.6,\hi);
      \node[axlabel,anchor=east,text=black!65] at (8.85,{(\lo+\hi)/2}) {$U_{q_{\q}}$};}
    \node[axlabel,anchor=south,font=\Large,inner sep=1pt,text=black!70] (hlab) at (27.52,4.64) {$\mathcal{H}$};
    \draw[black!55,line width=.7pt] (hlab.west) to[out=195,in=15] (26.15,4.66);
    \draw[decorate,decoration={brace,amplitude=5pt},black!55,line width=.8pt]
         (27.05,4.4) -- node[right=6pt,font=\Large,text=black!70]{$U$} (27.05,0.4);
  \end{scope}}
\def\cellsep#1{
  \begin{scope}[shift={(0,#1)}]
    \foreach \y in {1.20,2.10,3.10,3.60}{\draw[white,line width=1.3pt] (9.2,\y) -- (26.6,\y);}
    \draw[black!55,line width=.8pt] (9.2,0.4) rectangle (26.6,4.4);
  \end{scope}}
\def\seg#1#2#3{
  \pgfmathsetmacro\xa{xt(#1)}\pgfmathsetmacro\xb{xt(#1+1)}%
  \pgfmathsetmacro\ya{yb(#2)}\pgfmathsetmacro\yb{yb(#3)}}
\def\steps{0/0/1,1/1/0,2/0/1,3/1/2,4/2/3,5/3/4,6/4/4}
\def\visits{0/0,1/1,2/0,3/1,4/2,5/3,6/4,7/4}
\def\runtokens{0/1/cellA,1/2/cellB,2/1/cellA,3/2/cellB,4/3/cellC,5/4/cellD,6/5/cellE,7/5/cellE}
\def\nominal#1{%
  \begin{scope}[shift={(0,#1)}]
    \foreach \t/\k/\n in \steps{\seg\t\k\n
      \draw[nominal] (\xa,\ya) to[out=0,in=180] (\xb,\yb);}
    \draw[nominal,dash pattern=on 2pt off 2pt] ({xt(7)},{yb(4)}) -- (26.6,{yb(4)});
  \end{scope}}
\def\phasemk#1{\tikz[baseline=(pm.base)]\node[phase](pm){#1};}
\def\gstep#1#2#3#4#5#6{%
  \pgfmathsetmacro\gxb{#1+1.25}\pgfmathsetmacro\gxc{#1+2.25}%
  \pgfmathsetmacro\gau{#2+#3}\pgfmathsetmacro\gal{#2-#3}%
  \pgfmathsetmacro\gbu{#4+#5}\pgfmathsetmacro\gbl{#4-#5}%
  \fill[#6] (#1,\gau) to[out=0,in=180] (\gxb,\gbu) -- (\gxc,\gbu) -- (\gxc,\gbl) -- (\gxb,\gbl)
            to[out=180,in=0] (#1,\gal) -- cycle;}
\def\gline#1#2#3#4{%
  \pgfmathsetmacro\gxb{#1+1.25}\pgfmathsetmacro\gxc{#1+2.25}%
  \draw[#4] (#1,#2) to[out=0,in=180] (\gxb,#3) -- (\gxc,#3);}
\def\gbar#1#2#3#4{%
  \pgfmathsetmacro\gxb{#1+1.25}\pgfmathsetmacro\gu{#2+#3}\pgfmathsetmacro\gl{#2-#3}%
  \draw[bar,line width=5pt,line cap=butt,draw=#4] (\gxb,\gl) -- (\gxb,\gu);}
\def\stepgrid#1{%
  \begin{scope}[shift={(0,#1)}]
    \foreach \t in {0,...,7}{\draw[white,line width=0.8pt] ({xt(\t)},0.4) -- ({xt(\t)},4.4);}
  \end{scope}}
\def\checkmk{\draw[line width=2.2pt,line cap=round,line join=round] (-0.25,0) -- (-0.05,-0.22) -- (0.3,0.25);}
\def\crossmk{\draw[line width=2.2pt,line cap=round] (-0.2,-0.2) -- (0.2,0.2) (-0.2,0.2) -- (0.2,-0.2);}
\def\readset#1#2#3{%
  \pgfmathsetmacro\rsw{0.5*#2+0.14}%
  \node[draw=black!50,line width=1.1pt,rounded corners=3pt,minimum width=\rsw cm,minimum height=6.6mm] at (#1,0.35) {};
  \foreach \q/\c [count=\i from 0] in {#3}{%
    \node[mtok,fill=\c] at ({#1+0.5*(\i-(#2-1)/2)},0.35) {$q_{\q}$};}}

\useasboundingbox (-0.041,-1.1925) rectangle (28.033,14.885);

\node[anchor=east,font=\Large,text=black!65] at (8.85,14.25) {run on $w$};
\foreach \t/\q/\c in \runtokens{\node[tok,fill=\c] (r\t) at ({xt(\t)},14.25) {$q_{\q}$};}
\foreach \t/\x [evaluate=\t as \s using int(\t+1)] in {0/b,1/b,2/a,3/a,4/a,5/a,6/b}{
  \draw[->,black!55,line width=.8pt] (r\t) -- node[above=1pt,font=\Large,text=inkcol]{$\x$} (r\s);}
\node[font=\Large,text=black!55] at ({xt(7)+1.1},14.25) {$\cdots$};

\def\rowA{7.8}
\node[rowtitle] at (8.9,12.95) {The dynamics execute the transition and absorb the perturbations};
\statespace{\rowA}
\cellsep{\rowA}
\stepgrid{\rowA}   
\begin{scope}[shift={(0,\rowA)}]
  \foreach \t/\k/\n/\o in {0/0/1/.07,1/1/0/-.07,2/0/1/.07,3/1/2/-.07,4/2/3/-.04,5/3/4/.07,6/4/4/-.07}{
    \pgfmathsetmacro\xa{xt(\t)}\pgfmathsetmacro\xc{xt(\t+1)}\pgfmathsetmacro\xb{\xc-1.0}
    \pgfmathsetmacro\ya{yb(\k)}\pgfmathsetmacro\yd{yb(\n)}\pgfmathsetmacro\yc{\yd+\o}
    \pgfmathsetmacro\yu{\yc+0.18}\pgfmathsetmacro\yl{\yc-0.18}
    \pgfmathsetmacro\xm{\xb+0.55}\pgfmathsetmacro\xn{\xc-0.4}
    \fill[black!30,opacity=.9] (\xa,\ya) to[out=0,in=180] (\xb,\yu) -- (\xb,\yl) to[out=180,in=0] (\xa,\ya) -- cycle;
    \shade[left color=black!40,right color=codecol!45,opacity=.9]
          (\xb,\yu) .. controls (\xm,\yu) and (\xn,\yd) .. (\xc,\yd)
          .. controls (\xn,\yd) and (\xm,\yl) .. (\xb,\yl) -- cycle;
    \draw[nominal] (\xa,\ya) to[out=0,in=180] (\xb,\yc);
    \draw[inkcol,line width=.6pt] (\xb,\yc) .. controls (\xm,\yc) and (\xn,\yd) .. (\xc,\yd);
    \draw[bar,line width=5pt,line cap=butt,draw=black!55] (\xb,\yl) -- (\xb,\yu);
  }
  \draw[nominal,dash pattern=on 2pt off 2pt] ({xt(7)},{yb(4)}) -- (26.6,{yb(4)});
  \foreach \t/\k in \visits{\node[cpoint] at ({xt(\t)},{yb(\k)}) {};}
  \node[font=\LARGE,text=inkcol,anchor=north west,inner sep=1pt] at (17.35,1.95) {$h_t$};
  \node[font=\Large,text=black!70,anchor=north] (etalab) at (20.15,3.00) {$\Ball$};
  \draw[black!55,line width=.7pt] (etalab.north) -- (20.15,3.13);
  \node[anchor=west,font=\large,text=black!70] (ph1) at (9.55,2.62)
       {\phasemk{1}\,transition};
  \draw[black!50,line width=.6pt] (ph1.south) to[out=270,in=100] (10.45,1.30);
  \node[anchor=west,font=\large,text=black!70] (ph2) at (12.35,2.62)
       {\phasemk{2}\,restoration};
  \draw[black!50,line width=.6pt] (ph2.south west) to[out=250,in=60] (11.70,1.90);
\end{scope}
\node[anchor=east,font=\Large,text=deccol] at (8.85,7.25) {$\dec(h_t)$};
\foreach \t/\q/\c in \runtokens{\node[tok,fill=\c] at ({xt(\t)},7.3) {$q_{\q}$};}
\begin{scope}[shift={(27.35,7.3)},inkcol]\checkmk\end{scope}

\def\rowB{0.85}
\node[rowtitle] at (8.9,6.0) {The dynamics execute the transition but do not absorb the perturbations};
\statespace{\rowB}
\cellsep{\rowB}
\stepgrid{\rowB}   
\begin{scope}[shift={(0,\rowB)}]
  \begin{scope}
  \clip[rounded corners=6pt] \Hrect;
  \begin{scope}[transparency group,opacity=.55]
    \gstep{9.90}{0.80}{0.00}{1.65}{0.12}{black!40}
    \gstep{12.15}{1.65}{0.12}{0.80}{0.24}{black!40}
    \gstep{14.40}{0.80}{0.24}{1.65}{0.36}{black!40}
    \gstep{16.65}{1.65}{0.36}{2.60}{0.48}{black!40}
    \gstep{18.90}{2.60}{0.48}{3.35}{0.60}{failcol!80}
    \gstep{21.15}{3.350}{0.250}{4.00}{0.370}{failcol!80}
    \gstep{23.40}{4.00}{0.370}{4.00}{0.490}{failcol!80}
    \gstep{21.15}{2.925}{0.175}{3.35}{0.295}{failcol!80}
    \gstep{23.40}{3.35}{0.295}{3.35}{0.415}{failcol!80}
    \gstep{21.15}{3.775}{0.175}{0.80}{0.295}{failcol!80}
    \draw[failcol!80,line width=0.35cm] (21.15,3.775) to[out=0,in=180] (22.40,0.80);   
    \gstep{23.40}{0.80}{0.295}{1.65}{0.415}{failcol!80}
    \fill[failcol!80] (25.65,3.51) rectangle (26.6,4.49);
    \fill[failcol!80] (25.65,2.935) rectangle (26.6,3.765);
    \fill[failcol!80] (25.65,1.235) rectangle (26.6,2.065);
  \end{scope}
  \end{scope}
  \foreach \y in {1.20,2.10,3.10,3.60}{\draw[white,line width=1.3pt,opacity=.9] (9.2,\y) -- (26.6,\y);}
  \gline{9.90}{0.80}{1.65}{nominal}
  \gline{12.15}{1.65}{0.80}{nominal}
  \gline{14.40}{0.80}{1.65}{nominal}
  \gline{16.65}{1.65}{2.60}{nominal}
  \gline{18.90}{2.60}{3.35}{nominal}
  \gline{21.15}{3.35}{4.00}{nominal}
  \gline{23.40}{4.00}{4.00}{nominal}
  \draw[nominal,dash pattern=on 2pt off 2pt] (25.65,4.0) -- (26.6,4.0);
  \gline{21.15}{2.925}{3.35}{branch}
  \gline{23.40}{3.35}{3.35}{branch}
  \gline{21.15}{3.775}{0.80}{branch}
  \gline{23.40}{0.80}{1.65}{branch}
  \draw[branch] (25.65,3.35) -- (26.6,3.35) (25.65,1.65) -- (26.6,1.65);
  \gbar{9.90}{1.65}{0.12}{black!55}
  \gbar{12.15}{0.80}{0.24}{black!55}
  \gbar{14.40}{1.65}{0.36}{black!55}
  \gbar{16.65}{2.60}{0.48}{black!55}
  \gbar{18.90}{3.35}{0.60}{black!55}
  \draw[white,line width=1.3pt] (19.97,3.10) -- (20.33,3.10) (19.97,3.60) -- (20.33,3.60);
  \gbar{21.15}{4.00}{0.370}{black!55}  \gbar{21.15}{3.35}{0.295}{black!55}  \gbar{21.15}{0.80}{0.295}{black!55}
  \gbar{23.40}{4.00}{0.490}{black!55}  \gbar{23.40}{3.35}{0.415}{black!55}  \gbar{23.40}{1.65}{0.415}{black!55}
  \foreach \t/\k in \visits{\node[cpoint] at ({xt(\t)},{yb(\k)}) {};}
  \foreach \p in {(23.40,3.35),(23.40,0.80),(25.65,3.35),(25.65,1.65)}{\node[cpoint] at \p {};}
\end{scope}
\node[anchor=east,font=\Large,text=deccol] at (8.85,0.30) {$\dec(h_t)$};
\foreach \t/\q/\c in {0/1/cellA,1/2/cellB,2/1/cellA,3/2/cellB,4/3/cellC}{
  \node[tok,fill=\c] at ({xt(\t)},0.35) {$q_{\q}$};}
\readset{21.15}{3}{3/cellC,4/cellD,5/cellE}
\readset{23.40}{4}{1/cellA,3/cellC,4/cellD,5/cellE}
\readset{25.65}{4}{2/cellB,3/cellC,4/cellD,5/cellE}
\begin{scope}[shift={(27.35,0.35)},failcol]\crossmk\end{scope}

\foreach \t in {0,...,7}{\node[axlabel] at ({xt(\t)},-0.42) {$\t$};}
\draw[->,black!50,line width=.8pt] ({xt(0)-0.4},-0.9) -- ({xt(7)+1.2},-0.9)
     node[right,axlabel,text=black!65]{step $t$};

\node[font=\LARGE\bfseries,text=inkcol,anchor=base] at (3.75,10.3) {Semiautomaton $(Q,\Sig,\delta)$};
\begin{scope}[shift={(3.75,6.65)},scale=1.2,every node/.append style={transform shape}]
  \node[autstate,fill=cellA] (q1) at ( 90:2.1) {$q_1$};
  \node[autstate,fill=cellB] (q2) at ( 18:2.1) {$q_2$};
  \node[autstate,fill=cellC] (q3) at (-54:2.1) {$q_3$};
  \node[autstate,fill=cellD] (q4) at (-126:2.1) {$q_4$};
  \node[autstate,fill=cellE] (q5) at (162:2.1) {$q_5$};
  \path (q1) edge[atrans,bend left=22] node[above right=1pt and 0pt,font=\large]{$a,b$} (q2)
        (q2) edge[atrans,bend left=22] node[below left=1pt and 0pt,font=\large]{$b$} (q1)
        (q2) edge[atrans] node[right,font=\large]{$a$} (q3);
  \path (q3) edge[atrans] node[below,font=\large]{$a$} (q4)
        (q4) edge[atrans] node[left,font=\large]{$a$} (q5)
        (q5) edge[atrans] node[above left=-1pt,font=\large]{$a$} (q1);
  \path (q3) edge[atrans,loop right] node[right,font=\large]{$b$} ()
        (q4) edge[atrans,loop left] node[left,font=\large]{$b$} ()
        (q5) edge[atrans,loop left] node[left,font=\large]{$b$} ();
\end{scope}
\end{tikzpicture}

%% file: assets/figure_2.tex
\providecommand{\RTtot}{22.4667}
\pgfmathsetmacro{\rtS}{\textwidth/28.4527/\RTtot}

\begin{tikzpicture}[font=\large,scale=\rtS,
                    every node/.append style={scale=\rtS}]

\providecommand{\RTlab}{0.05}
\providecommand{\RTfirst}{3.60}
\providecommand{\RTpitch}{1.73}   
\providecommand{\RTl}{0.00}
\providecommand{\RTr}{16.60}

\providecommand{\RTRl}{16.90}
\providecommand{\RTRr}{22.4667}
\providecommand{\RTpitchR}{1.855} 
\providecommand{\RTfirstR}{17.83}

\providecommand{\RTcell}[5]{%
  \pgfmathsetmacro{\rtcx}{\RTfirst+#1*\RTpitch}%
  \node[font=\Large]               at (\rtcx,{#2})      {#3};
  \node[font=\small,text=black!68] at (\rtcx,{#2-0.40}) {#5};}

\providecommand{\RTcellO}[4]{%
  \pgfmathsetmacro{\rtcx}{\RTfirst+#1*\RTpitch}%
  \node[font=\Large,text=codecol] at (\rtcx,{#2})      {#3};
  \node[font=\small,text=codecol] at (\rtcx,{#2-0.40}) {#4};}

\providecommand{\RTcellR}[5]{%
  \pgfmathsetmacro{\rtcx}{\RTfirstR+#1*\RTpitchR}%
  \node[font=\Large]               at (\rtcx,{#2})      {#3};
  \node[font=\small,text=black!68] at (\rtcx,{#2-0.40}) {#5};}

\providecommand{\RTcellOR}[4]{%
  \pgfmathsetmacro{\rtcx}{\RTfirstR+#1*\RTpitchR}%
  \node[font=\Large,text=codecol] at (\rtcx,{#2})      {#3};
  \node[font=\small,text=codecol] at (\rtcx,{#2-0.40}) {#4};}

\providecommand{\RTa}{4.22}
\providecommand{\RTb}{3.06}
\providecommand{\RTc}{1.90}
\providecommand{\RTd}{0.74}

\fill[codecol!7] (\RTl,3.46)  rectangle (\RTr,4.62);
\fill[codecol!7] (\RTRl,3.46) rectangle (\RTRr,4.62);


\node[anchor=west,font=\large] at (\RTlab,6.05)
     {Failing length $L_{\max}$
      \textcolor{black!60}{\normalsize(NFSM at $100\%$, baselines at $90\%$;
      trained at $L{=}64$)}};
\draw[black!75,line width=.9pt] (\RTl,5.78) -- (\RTr,5.78);

\foreach \ci/\task/\idx in {%
    0/{$\mathbb{Z}_{2}$}/2,
    1/{$\mathbb{Z}_{16}$}/16,
    2/{$S_{3}$}/6,
    3/{$S_{4}$}/24,
    4/{$A_{5}$}/60,
    5/{$M_{11}$}/7920,
    6/{$\mathrm{DFF}_{5}$}/3,
    7/{$\mathrm{FF}$}/3}{%
  \pgfmathsetmacro{\rthx}{\RTfirst+\ci*\RTpitch}%
  \node[font=\large]               at (\rthx,5.30) {\task};
  \node[font=\small,text=black!68] at (\rthx,4.86) {[\idx]};}

\draw[black!75,line width=.6pt] (\RTl,4.62) -- (\RTr,4.62);

\node[anchor=west,font=\large,text=codecol] at (\RTlab,\RTa)
     {NFSM};
\node[anchor=west,font=\small,text=codecol] at (\RTlab,{\RTa-0.40})
     {1 layer, 2 on TSO};

\node[anchor=west,font=\large] at (\RTlab,\RTb) {Mamba$^{\text{---}}$};
\node[anchor=west,font=\large] at (\RTlab,\RTc) {AUSSM};
\node[anchor=west,font=\large] at (\RTlab,\RTd) {PD-SSM};

\foreach \ry in {\RTb,\RTc,\RTd}{%
  \node[anchor=west,font=\small,text=black!68] at (\RTlab,{\ry-0.40})
       {$4\times256$};}

\RTcellO{0}{\RTa}{$\geq\!1\mathrm{M}^{\dagger}$}{5/5}
\RTcellO{1}{\RTa}{$\geq\!1\mathrm{M}^{\dagger}$}{5/5}
\RTcellO{2}{\RTa}{$\geq\!1\mathrm{M}^{\dagger}$}{5/5}
\RTcellO{3}{\RTa}{$\geq\!1\mathrm{M}^{\dagger}$}{5/5}
\RTcellO{4}{\RTa}{$\geq\!1\mathrm{M}^{\dagger}$}{5/5}
\RTcellO{5}{\RTa}{$\geq\!32\mathrm{k}^{\dagger}$}{5/5}
\RTcellO{6}{\RTa}{$\geq\!1\mathrm{M}^{\dagger}$}{5/5}
\RTcellO{7}{\RTa}{$\geq\!1\mathrm{M}^{\dagger}$}{5/5}

\RTcell{0}{\RTb}{$4352$}{[$326$;$\geq\!512\mathrm{k}$]}{5/5}
\RTcell{1}{\RTb}{$85$}{[$78$;$91$]}{4/5}
\RTcell{2}{\RTb}{$186$}{[$155$;$200$]}{5/5}
\RTcell{3}{\RTb}{$105$}{[$92$;$113$]}{4/5}
\RTcell{4}{\RTb}{$\times$}{}{0/5}
\RTcell{5}{\RTb}{$\times$}{}{0/5}
\RTcell{6}{\RTb}{$\geq\!1\mathrm{M}$}{[$\geq\!512\mathrm{k}$;$\geq\!1\mathrm{M}$]}{5/5}
\RTcell{7}{\RTb}{$178$}{[$143$;$1738$]}{5/5}

\RTcell{0}{\RTc}{$144$}{[$101$;$274$]}{5/5}
\RTcell{1}{\RTc}{$135$}{[$108$;$145$]}{5/5}
\RTcell{2}{\RTc}{$128$}{[$100$;$216$]}{5/5}
\RTcell{3}{\RTc}{$89$}{[$83$;$103$]}{5/5}
\RTcell{4}{\RTc}{$82$}{[$81$;$86$]}{5/5}
\RTcell{5}{\RTc}{$\times$}{}{0/5}
\RTcell{6}{\RTc}{$96$}{[$87$;$120$]}{5/5}
\RTcell{7}{\RTc}{$119$}{[$70$;$148$]}{5/5}

\RTcell{0}{\RTd}{$110$}{[$102$;$119$]}{5/5}
\RTcell{1}{\RTd}{$98$}{[$94$;$117$]}{5/5}
\RTcell{2}{\RTd}{$133$}{[$89$;$149$]}{5/5}
\RTcell{3}{\RTd}{$131$}{[$101$;$171$]}{5/5}
\RTcell{4}{\RTd}{$187$}{[$141$;$288$]}{5/5}
\RTcell{5}{\RTd}{$\times$}{}{0/5}
\RTcell{6}{\RTd}{$\geq\!256\mathrm{k}$}{[$\geq\!256\mathrm{k}$;$\geq\!256\mathrm{k}$]}{5/5}
\RTcell{7}{\RTd}{$133$}{[$105$;$192$]}{5/5}


\node[anchor=west,font=\normalsize,text=black!60] at (\RTRl,6.05)
     {trained at $L{=}256$};
\draw[black!75,line width=.9pt] (\RTRl,5.78) -- (\RTRr,5.78);

\foreach \ci/\task/\idx in {%
    0/{$\mathrm{TSO}_{N=3}$}/44,
    1/{$\mathrm{TSO}_{N=4}$}/202,
    2/{$\mathrm{TSO}_{N=5}$}/1132}{%
  \pgfmathsetmacro{\rthx}{\RTfirstR+\ci*\RTpitchR}%
  \node[font=\large]               at (\rthx,5.30) {\task};
  \node[font=\small,text=black!68] at (\rthx,4.86) {[\idx]};}

\draw[black!75,line width=.6pt] (\RTRl,4.62) -- (\RTRr,4.62);

\RTcellOR{0}{\RTa}{$\geq\!1\mathrm{M}$}{5/5}
\RTcellOR{1}{\RTa}{$\geq\!1\mathrm{M}$}{5/5}
\RTcellOR{2}{\RTa}{$\geq\!128\mathrm{k}$}{5/5}

\RTcellR{0}{\RTb}{$1501$}{[$1420$;$1583$]}{2/5}
\RTcellR{1}{\RTb}{$\times$}{}{0/5}
\RTcellR{2}{\RTb}{$\times$}{}{0/5}

\RTcellR{0}{\RTc}{$\times$}{}{0/5}
\RTcellR{1}{\RTc}{$\times$}{}{0/5}
\RTcellR{2}{\RTc}{$\times$}{}{0/5}

\RTcellR{0}{\RTd}{$363$}{[$286$;$575$]}{5/5}
\RTcellR{1}{\RTd}{$1922$}{[$1922$;$1922$]}{1/5}
\RTcellR{2}{\RTd}{$321$}{[$321$;$321$]}{1/5}

\draw[black!75,line width=.9pt] (\RTRl,-0.10) -- (\RTRr,-0.10);


\draw[black!75,line width=.9pt] (\RTl,-0.10) -- (\RTr,-0.10);
\end{tikzpicture}

%% file: assets/figure_3.tex
%
%
%

\definecolor{mcnfsm}{HTML}{D94801}
\providecommand{\tbd}{\textcolor{black!40}{--}}
\small
\setlength{\tabcolsep}{5pt}
\begin{tabular}{ll ccccc cc}
\toprule
& & \multicolumn{5}{c}{Failing length $L_{\max}$ per seed} & & \\
\cmidrule(lr){3-7}
Task & Model & $42$ & $43$ & $44$ & $45$ & $46$ & Median & Range \\
\midrule
\rowcolor{mcnfsm!10}$\mathbb{Z}_{2}$      & \textcolor{mcnfsm}{NFSM}     & $\geq 1M$ & $\geq 1M$ & $\geq 1M$ & $\geq 1M$ & $\geq 1M$ & $\geq 1M$ & [$\geq 1M$;\,$\geq 1M$] \\ 
\rowcolor{black!2.5}                      & Mamba$^{\text{---}}$         & $4352$ & $5176$ & $326$ & $655$ & $\geq 512k$ & $4352$ & [$326$;\,$\geq 512k$] \\ 
                      & AUSSM                        & $274$ & $115$ & $144$ & $101$ & $151$ & $144$ & [$101$;\,$274$] \\ 
\rowcolor{black!2.5}                      & PD-SSM                       & $102$ & $117$ & $105$ & $119$ & $110$ & $110$ & [$102$;\,$119$] \\ 
\addlinespace
\rowcolor{mcnfsm!10}$\mathbb{Z}_{16}$     & \textcolor{mcnfsm}{NFSM}     & $\geq 1M$ & $\geq 1M$ & $\geq 1M$ & $\geq 1M$ & $\geq 1M$ & $\geq 1M$ & [$\geq 1M$;\,$\geq 1M$] \\ 
\rowcolor{black!2.5}                      & Mamba$^{\text{---}}$         & $78$ & $91$ & $83$ & $87$ & $\times$ & $85$ & [$78$;\,$91$] \\ 
                      & AUSSM                        & $110$ & $108$ & $145$ & $135$ & $142$ & $135$ & [$108$;\,$145$] \\ 
\rowcolor{black!2.5}                      & PD-SSM                       & $98$ & $117$ & $94$ & $98$ & $97$ & $98$ & [$94$;\,$117$] \\ 
\addlinespace
\rowcolor{mcnfsm!10}$S_{3}$               & \textcolor{mcnfsm}{NFSM}     & $\geq 1M$ & $\geq 1M$ & $\geq 1M$ & $\geq 1M$ & $\geq 1M$ & $\geq 1M$ & [$\geq 1M$;\,$\geq 1M$] \\ 
\rowcolor{black!2.5}                      & Mamba$^{\text{---}}$         & $155$ & $163$ & $186$ & $197$ & $200$ & $186$ & [$155$;\,$200$] \\ 
                      & AUSSM                        & $102$ & $168$ & $128$ & $216$ & $100$ & $128$ & [$100$;\,$216$] \\ 
\rowcolor{black!2.5}                      & PD-SSM                       & $130$ & $89$ & $133$ & $139$ & $149$ & $133$ & [$89$;\,$149$] \\ 
\addlinespace
\rowcolor{mcnfsm!10}$S_{4}$               & \textcolor{mcnfsm}{NFSM}     & $\geq 1M$ & $\geq 1M$ & $\geq 1M$ & $\geq 1M$ & $\geq 1M$ & $\geq 1M$ & [$\geq 1M$;\,$\geq 1M$] \\ 
\rowcolor{black!2.5}                      & Mamba$^{\text{---}}$         & $103$ & $\times$ & $113$ & $92$ & $107$ & $105$ & [$92$;\,$113$] \\ 
                      & AUSSM                        & $103$ & $83$ & $83$ & $89$ & $98$ & $89$ & [$83$;\,$103$] \\ 
\rowcolor{black!2.5}                      & PD-SSM                       & $113$ & $144$ & $101$ & $131$ & $171$ & $131$ & [$101$;\,$171$] \\ 
\addlinespace
\rowcolor{mcnfsm!10}$A_{5}$               & \textcolor{mcnfsm}{NFSM}     & $\geq 1M$ & $\geq 1M$ & $\geq 1M$ & $\geq 1M$ & $\geq 1M$ & $\geq 1M$ & [$\geq 1M$;\,$\geq 1M$] \\ 
\rowcolor{black!2.5}                      & Mamba$^{\text{---}}$         & $\times$ & $\times$ & $\times$ & $\times$ & $\times$ & $\times$ &  \\ 
                      & AUSSM                        & $81$ & $82$ & $81$ & $82$ & $86$ & $82$ & [$81$;\,$86$] \\ 
\rowcolor{black!2.5}                      & PD-SSM                       & $187$ & $179$ & $288$ & $189$ & $141$ & $187$ & [$141$;\,$288$] \\ 
\addlinespace
\rowcolor{mcnfsm!10}$M_{11}$              & \textcolor{mcnfsm}{NFSM}     & $\geq 32k$ & $\geq 32k$ & $\geq 32k$ & $\geq 32k$ & $\geq 32k$ & $\geq 32k$ & [$\geq 32k$;\,$\geq 32k$] \\ 
\rowcolor{black!2.5}                      & Mamba$^{\text{---}}$         & $\times$ & $\times$ & $\times$ & $\times$ & $\times$ & $\times$ &  \\ 
                      & AUSSM                        & $\times$ & $\times$ & $\times$ & $\times$ & $\times$ & $\times$ &  \\ 
\rowcolor{black!2.5}                      & PD-SSM                       & $\times$ & $\times$ & $\times$ & $\times$ & $\times$ & $\times$ &  \\ 
\addlinespace
\rowcolor{mcnfsm!10}$\mathrm{DFF}_{5}$    & \textcolor{mcnfsm}{NFSM}     & $\geq 1M$ & $\geq 1M$ & $\geq 1M$ & $\geq 1M$ & $\geq 1M$ & $\geq 1M$ & [$\geq 1M$;\,$\geq 1M$] \\ 
\rowcolor{black!2.5}                      & Mamba$^{\text{---}}$         & $\geq 1M$ & $\geq 1M$ & $\geq 1M$ & $\geq 512k$ & $\geq 512k$ & $\geq 1M$ & [$\geq 512k$;\,$\geq 1M$] \\ 
                      & AUSSM                        & $102$ & $120$ & $93$ & $96$ & $87$ & $96$ & [$87$;\,$120$] \\ 
\rowcolor{black!2.5}                      & PD-SSM                       & $\geq 256k$ & $\geq 256k$ & $\geq 256k$ & $\geq 256k$ & $\geq 256k$ & $\geq 256k$ & [$\geq 256k$;\,$\geq 256k$] \\ 
\addlinespace
\rowcolor{mcnfsm!10}FF                    & \textcolor{mcnfsm}{NFSM}     & $\geq 1M$ & $\geq 1M$ & $\geq 1M$ & $\geq 1M$ & $\geq 1M$ & $\geq 1M$ & [$\geq 1M$;\,$\geq 1M$] \\ 
\rowcolor{black!2.5}                      & Mamba$^{\text{---}}$         & $178$ & $173$ & $143$ & $181$ & $1738$ & $178$ & [$143$;\,$1738$] \\ 
                      & AUSSM                        & $70$ & $119$ & $148$ & $116$ & $146$ & $119$ & [$70$;\,$148$] \\ 
\rowcolor{black!2.5}                      & PD-SSM                       & $192$ & $105$ & $133$ & $115$ & $169$ & $133$ & [$105$;\,$192$] \\ 
\midrule
\rowcolor{mcnfsm!10}$\mathrm{TSO}_{3}$    & \textcolor{mcnfsm}{NFSM}     & $\geq 1M$ & $\geq 1M$ & $\geq 1M$ & $\geq 1M$ & $\geq 1M$ & $\geq 1M$ & [$\geq 1M$;\,$\geq 1M$] \\ 
\rowcolor{black!2.5}                      & Mamba$^{\text{---}}$         & $1583$ & $\times$ & $\times$ & $1420$ & $\times$ & $1501$ & [$1420$;\,$1583$] \\ 
                      & AUSSM                        & $\times$ & $\times$ & $\times$ & $\times$ & $\times$ & $\times$ &  \\ 
\rowcolor{black!2.5}                      & PD-SSM                       & $358$ & $575$ & $363$ & $286$ & $503$ & $363$ & [$286$;\,$575$] \\ 
\addlinespace
\rowcolor{mcnfsm!10}$\mathrm{TSO}_{4}$    & \textcolor{mcnfsm}{NFSM}     & $\geq 1M$ & $\geq 1M$ & $\geq 1M$ & $\geq 1M$ & $\geq 1M$ & $\geq 1M$ & [$\geq 1M$;\,$\geq 1M$] \\ 
\rowcolor{black!2.5}                      & Mamba$^{\text{---}}$         & $\times$ & $\times$ & $\times$ & $\times$ & $\times$ & $\times$ &  \\ 
                      & AUSSM                        & $\times$ & $\times$ & $\times$ & $\times$ & $\times$ & $\times$ &  \\ 
\rowcolor{black!2.5}                      & PD-SSM                       & $\times$ & $1922$ & $\times$ & $\times$ & $\times$ & $1922$ & [$1922$;\,$1922$] \\ 
\addlinespace
\rowcolor{mcnfsm!10}$\mathrm{TSO}_{5}$    & \textcolor{mcnfsm}{NFSM}     & $\geq 128k$ & $\geq 128k$ & $\geq 128k$ & $\geq 128k$ & $\geq 128k$ & $\geq 128k$ & [$\geq 128k$;\,$\geq 128k$] \\ 
\rowcolor{black!2.5}                      & Mamba$^{\text{---}}$         & $\times$ & $\times$ & $\times$ & $\times$ & $\times$ & $\times$ &  \\ 
                      & AUSSM                        & $\times$ & $\times$ & $\times$ & $\times$ & $\times$ & $\times$ &  \\ 
\rowcolor{black!2.5}                      & PD-SSM                       & $\times$ & $\times$ & $\times$ & $321$ & $\times$ & $321$ & [$321$;\,$321$] \\ 
\bottomrule
\end{tabular}